\documentclass[10pt]{article}%
\usepackage[preprint]{tmlr}
\usepackage{microtype}
\usepackage{graphicx}
\usepackage{tikz}
\usepackage{threeparttable}%
\usepackage{tabularx}
\usepackage{subcaption}
\usepackage{enumitem}
\usepackage{booktabs}%
\usepackage{algorithm}
\usepackage{algpseudocode}
\usepackage{amsmath,amsthm}
\usepackage{siunitx}

\usepackage{amsmath,amsfonts,bm}

\def\eqref#1{equation~\ref{#1}}
\def\Eqref#1{Equation~\ref{#1}}

\def\1{\bm{1}}

\DeclareMathAlphabet{\mathsfit}{\encodingdefault}{\sfdefault}{m}{sl}
\SetMathAlphabet{\mathsfit}{bold}{\encodingdefault}{\sfdefault}{bx}{n}

\usepackage{hyperref}
\usepackage{url}
\theoremstyle{plain}
\newtheorem{theo}{Theorem}[section]
\newtheorem{proposition}[theo]{Proposition}
\newtheorem{lemma}[theo]{Lemma}
\newtheorem{corollary}[theo]{Corollary}
\theoremstyle{definition}
\newtheorem{definition}[theo]{Definition}

\theoremstyle{remark}
\newtheorem{remark}[theo]{Remark}

\title{VBO-MI: A Differentiable \textcolor{black}{Black-Box} Optimization Framework via Variational Mutual Information Estimation}

\author{\name Farhad Mirkarimi \email farhad.mirkarimi@mcgill.mail.ca \\
      \addr McGill University}

\def\month{10}%
\def\year{2026}%
\def\openreview{\url{https://openreview.net/forum?id=nkj5zYqJ2p&nesting=2&sort=date-desc}}%

\begin{document}

\maketitle
\begin{abstract}
    Many real-world tasks require optimizing expensive black-box functions accessible only through noisy evaluations, a setting commonly addressed with Bayesian optimization. While Gaussian processes~(GPs) and Bayesian neural networks have driven significant progress, GP-based methods scale cubically with observations, and BNN-based approaches remain burdened by expensive posterior sampling such as Hamiltonian Monte Carlo, and inner-loop acquisition optimization in high-dimensional setting. In this work, we propose \textbf{VBO-MI} (Variational \textcolor{black}{Black-Box} Optimization with Mutual Information), a fully gradient-based BO framework that requires no  assumptions on a Gaussian prior or a variational posterior family over $f$. While our framework uses gradient-based updates to train the internal action and information-estimation networks, it treats the objective function $f$ as a strict black box, requiring only zeroth-order evaluations (input-output pairs). To enable end-to-end gradient flow, we employ an actor-critic architecture consisting of an action-net to navigate the input space and a variational critic to estimate a \textcolor{black}{surrogate information signal} for exploration. This formulation effectively eliminates the traditional inner-loop acquisition optimization bottleneck, achieving up to a $10^{2}\times$ reduction in FLOPs compared to BNN-BO baselines. Additionally, we design a lightweight surrogate that approximates the objective from observed data using a replay buffer, which reduces the total required real function queries to one batch per iteration. We provide theoretical analysis for our framework, establishing
convergence of its training objective and surrogate
estimator, and showing, that near-optimal
solutions exist within our hypothesis class. We evaluate our method on a diverse suite of benchmarks, including high-dimensional synthetic functions (Ackley, Levy, Griewank) and complex real-world tasks such as Rover Trajectory planning optimization, the Lunar Lander, and Pest Control problems. Our experiments demonstrate that VBO-MI provides the same or superior optimization performance over the various baselines.
\end{abstract}

\section{Introduction}
\label{submission}
Bayesian optimization has matured into a de-facto tool for expensive global black-box optimization with applications in hyperparameter tuning, materials design, and experimental sciences~\citep{frazier2018tutorialbayesianoptimization,garnett_bayesoptbook_2023}\nocite{tanha2026structured}.
In Bayesian optimization, the algorithm treats $f$ as a black-box function and aims to optimize it sequentially under a finite budget $N$ that is, with access to only $N$ evaluations (input--output pairs) of the function~\citep{271617}. At each evaluation, the algorithm builds a probabilistic model of $f$, using the observed input--output pairs and uses this model to guide the selection of the next evaluation points, by maximizing an acquisition function that trades off exploration and exploitation with the goal of finding the global optimum as efficiently as possible.

Traditionally, Gaussian process (GP) assumptions for the surrogate coupled with various acquisition functions, such as expected improvement~\citep{Jones1998EfficientGO}, entropy search ~\citep{JMLR:v13:hennig12a}~\citep{Hennig_Osborne_Kersting_2022}, and Upper confidence bound (UCB) ~\citep{srinivas2010gaussian} have been effective in developing various efficient Bayesian optimization algorithms ~\citep{pmlr-v162-rudner22a}. However, because standard GPs require inverting a covariance matrix (which scales cubically in the number of observations), they become computationally prohibitive for large datasets. Although approximate GPs such as sparse inducing-point models have been proposed~\citep{ermonbo}, and more recent techniques based on GP approximation scale to high-dimensional BO problems~\citep{pmlr-v206-jimenez23a}, these methods still face trade-offs between the accuracy of finding optimal points, local fidelity, and uncertainty quantification.

Recently, several works have been focused on the replacing GPs with Bayesian neural networks (BNNs) in Bayesian optimization algorithms ~\citep{li2024a,NIPS2016_a96d3afe}. Various methods have been used in the literature for computing posteriors in the BNNs, these include variational inference with a Sparse Gaussian assumption ~\citep{NEURIPS2024_da30215e}, Gamma variational distribution ~\citep{cheng2025a}, Hamiltonian Monte Carlo (HMC) among others ~\citep{li2024a}. However, all of the methods rely on the specific form of distribution for posterior sampling, which limits their capability to search entire space, specially when the underlying function is complex and high dimensional. Another limitation is computational efficiency~\citep{cheng2025a} as training BNNs ~\citep{li2024a} is prohibitive.

\textcolor{black}{These limitations are not merely theoretical, but manifest in
specific, increasingly common situations. Hyperparameter tuning is one
of the classical applications of BO \citep{7352306}, yet
many modern machine learning models from deep networks to large
language models, expose dozens to hundreds of hyperparameters
governing architecture, optimization, and fine-tuning. Such
high-dimensional search spaces require substantially more observations
to adequately explore, a difficulty long recognized as a core
limitation of BO \citep{frazier2018tutorialbayesianoptimization}, which in turn compounds
the $\mathcal{O}(T^3)$ per-iteration cost of exact GP inference
discussed above. Large language model fine-tuning is a particularly
prominent recent instance of this challenge, where the scalability of
BO's inherently sequential search has been explicitly identified as a
bottleneck \citep{appllm}. Separately, in robotic control,
an agent must learn an optimal operating configuration through direct
interaction with its physical environment
\citep{7352306}, where the resulting reward landscape is
often highly non-stationary and shaped by complex, hard-to-specify
dynamics -- conditions under which GP surrogates, which impose a fixed
covariance structure over the entire input domain, are known to
underfit around the true optimum \citep{li2024a}. Such settings
need frameworks that commit to no specific parametric or
distributional form for either the objective or its noise, while
remaining tractable at scale.}

  To address these limitations, this work take a different approach and develops a framework that models the $f$ without assuming a specific parametric family. By leveraging a flexible variational representation, the proposed approach enables more expressive uncertainty modeling and improved exploration in complex, high-dimensional search spaces. Our method enables full gradient flow through the internal optimization process, resulting in faster training. Critically, while the framework utilizes gradient-based updates for its internal neural components, it maintains the standard Bayesian optimization requirement of working strictly with zeroth-order evaluations (input-output pairs) of the black-box function $f$. Specifically, we leverage recent advances in variational mutual information estimation~\citep{song2020understanding,pmlr-v80-belghazi18a} to design a novel acquisition function that is directly approximated from data samples. This acquisition function is jointly trained with a separate neural network, termed the action-net, which explores the input space and guides the optimization process. Our main contributions are as follows:%
  \begin{enumerate}
\item We propose a general Bayesian optimization framework that does not assume any specific variational form for the surrogate model, allowing greater flexibility in modeling complex objective functions.
\item We introduce a novel variational acquisition function (which is used as a loss function) derived from a variational mutual information formulation, which provides an explicit and principled trade-off between exploration and exploitation.
 \item \textcolor{black} { We provide theoretical study of our framework, examining the
   estimation properties of its training objective and surrogate
   network under adaptively-sampled (non-i.i.d.)\ data, and show
   that, under stated conditions, networks achieving near-optimal
   performance exist within our hypothesis class.}
\item We benchmark VBO-MI across a diverse suite of baselines, including BNN based Bayesian optimization methods, and test it on high-dimensional synthetic functions and complex real-world control tasks such as Rover trajectory optimization, Pest control, and the Lunar lander problem. Our results show that VBO-MI provides competetive optimization performance  over various set of baselines with reduced complexity overhead.
\end{enumerate}
\paragraph{Notation.}
Random vectors are denoted by {bold lowercase} symbols (e.g., $\mathbf{x}_t \in \mathbb{R}^d$). 
The sequence of random vectors observed up to step $t$ is written as
$\mathbf{X}_{t} = [\mathbf{x}_1, \dots, \mathbf{x}_t] \in \mathbb{R}^{t \times d}.
$ A realization (sample) of a random vector uses a superscript index $(i)$. 
Thus
$\mathbf{x}_t^{(i)} \in \mathbb{R}^d,\,
    \mathbf{X}_t^{(i)} = [\mathbf{x}_1^{(i)}, \dots, \mathbf{x}_t^{(i)}] \in \mathbb{R}^{t \times d}.
$
Here, the index $i$ distinguishes samples in a mini-batch or dataset.
Moreover, scalar components of a realization are written as
$\mathbf{x}^{(i)}_t = [\, x^{(i)}_{1,t}, \dots, x^{(i)}_{d,t} \,].
$
Calligraphic letters (e.g. $\mathcal{D} \subset \mathbb{R}^d$) denote domains or sets.
Expectations are written as $\mathbb{E}_{P_X}[\cdot]$ under distribution $P_X$. 
The joint and product-of-marginals distributions are denoted by $P_{X,Y}$ and $P_X \times P_Y$, respectively. 
Probability density functions use lowercase notation, e.g. $p(\mathbf{x})$.
Mutual information between random vectors $\mathbf{X}$ and random variable ${y}$ is denoted
$
    I(\mathbf{x};y),
$
and $D_{\mathrm{KL}}(P\|Q)$ denotes the Kullback--Leibler divergence between two probability measures $P$ and $Q$. 
We use $\rightarrow$ to denote a Markov chain.

\section{Related Work}
 The literature on BO has increasingly focused on overcoming the computational limitations of traditional Gaussian Processes for high dimensional and large datasets and application and the sampling complexity of Bayesian neural networks. We categorize recent advancements into three distinct research frontiers.
 \paragraph{Scalable Bayesian optimization}
 
Gaussian Processes (GPs) remain the canonical choice for BO due to their well-calibrated uncertainty ~\citep{srinivas2010gaussian}. However, their $\mathcal{O}(T^3)$ complexity has led to a surge in scalable alternatives. Recent advancements include Vecchia approximations ~\citep{pmlr-v206-jimenez23a}, which achieve linear scaling by leveraging local conditioning sets, and Trust-Region Bayesian Optimization (TuRBO) ~\citep{NEURIPS2019_6c990b7a}, which decomposes the search space into local regions to handle high-dimensional tasks. More recently, Focalized Sparse GPs ~\citep{NEURIPS2024_da30215e} have been proposed to focus inducing points on promising regions of the input space. \textcolor{black} {In another line of work, Quadrature Fourier Features (QFF) ~\citep{mutny2018efficient} construct a deterministic Fourier feature map via Gauss--Hermite quadrature whose approximation error to the squared-exponential kernel decays exponentially with the number of features, reducing the cost of kernel inversion for (generalized) additive GPs from $\mathcal{O}(T^3)$ to $\mathcal{O}(T(\log T)^2)$ while preserving no-regret guarantees for Thompson sampling and GP-UCB, though the required feature count still scales exponentially with the effective dimension of the additive components. Budgeted Kernelized Bandit (BKB) ~\citep{pmlr-v99-calandriello19a} instead selects an adaptive dictionary of inducing points via ridge leverage score sampling for arbitrary (not necessarily stationary) kernels, achieving near-optimal regret with near-constant per-iteration complexity -- governed by the GP's effective dimension $d_{\text{eff}}$ rather than the number of past observations -- and without suffering from variance starvation. Batch Budgeted Kernel Bandits (BBKB) ~\citep{pmlr-v119-calandriello20a} extends this dictionary-based approach to batched, parallel candidate selection by adaptively delaying and enlarging updates to the inducing set, reaching near-linear total runtime while matching the same regret guarantees. While effective, these methods often rely on stationary kernel assumptions that struggle to capture the complex, non-linear landscape of neural network landscapes.}
 \paragraph{Variational and Sampling-Based BNNs.}
 
 Bayesian Neural Networks (BNNs) offer superior scalability and representation power compared to GPs. Recent efforts integrated BNNs using Hamiltonian Monte Carlo (HMC)~\citep{NIPS2016_a96d3afe, li2024a} for surrogate modeling of BO. To reduce the prohibitive cost of sampling, Variational Inference (VI) has emerged as a dominant paradigm. This includes the use of Variational Last Layers (VBLL) ~\citep{harrison2024vbll}, which treats only the final layer stochastically, and Prior-Fitted Networks (PFNs) ~\citep{muller2023pfns}, which use Transformers to amortize the cost of Bayesian inference. Despite these gains, many VI-based BO methods still rely on specific parametric families (e.g., Gaussian or Gamma) for the posterior, which can lead to model misspecification in complex search spaces.

 \paragraph{Information-Theoretic acquisition function desgin}
 
  Information-theoretic acquisition functions, such as \textit{Predictive Entropy Search} (PES) ~\citep{pmlr-v37-hernandez-lobatob15} and \textit{Max-Value Entropy Search} (MES) ~\citep{pmlr-v70-wang17e}, are prized for their principled approach to non-myopic exploration. Recent works like \textit{Joint Entropy Search} (JES) ~\citep{hvarfner2022joint} attempt to enhance robustness by modeling the joint density over optimal inputs and values; however, due to the inherent complexity of this joint density, these methods typically revert to Gaussian Process (GP) assumptions for simplification.
  
  Furthermore, in a recent work, ~\citet{cheng2025a} proposes a variational entropy search framework, and shows its connection expected improvement acquistion functin. However, the introduced acquisition functions lacks interpretability regarding the relative contributions of exploration and exploitation due to its complexity, while being constrained by specific variational families such as the Gamma distribution~\citep{cheng2025a}.
 In another recent work, ~\citet{pmlr-v267-ishikura25a} propose a variational lower bound maximization for entropy search. However, their formulation remains rooted in Gaussian Process framework and relies on combining over- and under-truncation approximation to maximize the bound. As it is mentioned in~\citet{pmlr-v267-ishikura25a} finding a full gradient-based approach remains a open future work.
 
 In contrast, our work reformulates the search for optimal queries as a \textcolor{black}{two level} optimization problem, which we solve through a novel \textit{actor-critic} inspired framework. Unlike previous work, our approach makes no Gaussian Process assumptions on $f$, allowing for greater expressiveness. Essentially our framework, jointly addresses the problem of uncertainty modeling and data-driven acquisition function variational mutual information estimation.  We utilize two neural networks trained in an alternating, gradient-based fashion: an the action-net to identify optimal inputs, and a \textit{neural information gain estimator} (the helper) which provides the feedback to compute the acquisition function in the form of loss of the action-net. By optimizing the Donsker-Varadhan variational bound through direct gradient flow, our framework eliminates the need for both the rigid kernel dependencies of GPs and the expensive posterior sampling of traditional BNNs.

\section{Problem Setting and Background}
\label{sec:background}
\subsection{Problem Setting}
We consider the task of maximizing an unknown, scalar-valued objective function 
$f:\mathcal{D}\!\rightarrow\!\mathbb{R}$ defined over a compact domain 
$\mathcal{D}\!\subseteq\!\mathbb{R}^d$.  
The function is a black box: its analytical form and gradients are unavailable, and each evaluation is expensive or noisy.  
Formally, the optimization problem is expressed as
\begin{equation}
    \mathbf{x}^{*} = \arg\max_{\mathbf{x}\in\mathcal{D}} f(\mathbf{x}),
    \label{eq:bo_objective}
\end{equation}
where $\mathbf{x}^{*}$ denotes the global maximizer of~$f$.

At iteration~$t$, the optimizer selects the input~$\mathbf{x}_t$ and receives a noisy observation
\begin{equation}
    y_t = f(\mathbf{x}_t) + \varepsilon_t,
    \label{eq:noisy_eval}
\end{equation}
where $\varepsilon_t$ denotes independent, zero-mean observation noise.
The history up to step~$t$ is denoted by
$\mathbf{X}_t = [\mathbf{x}_1,\ldots,\mathbf{x}_t]$.
Given a fixed evaluation budget, the goal here is to design a sequential policy that
efficiently identifies $\mathbf{x}^{*}$ using ~$T$ steps. 
At each of the T main iterations, $f$ is evaluated at a batch of $B$ points all drawn from the same query distribution, $\mathbf{x}_{t}$ at step t; these $B$ samples serve as Monte Carlo estimates of expectations (as discussed in section 4) at that step. Hence, we denote by $N = T \cdot B$ the total number of real function evaluations in $T$ steps, the history $\mathbf{X}_t = [\mathbf{x}_1,...,\mathbf{x}_t]$ records one query distribution per iteration, and each $\mathbf{x}_t$ representing the query at step $t$. We discuss further about this query distribution in the Section 4 of the paper.\footnote{As discussed in Sec.\ref{secvbo}, we also perform $W$ warm-up iterations which query a batch 
of $B$ points.}. Querying the function in batches is practical in many real world tasks~\citep{pmlr-v162-jain22a}, and has been pursued in many recent works in high-dimensional Bayesian Optimization~\citep{yun2025posterior,li2024a}.

\subsection{Bayesian Optimization.}
Bayesian optimization (BO) traditionally addresses the problem by maintaining a
probabilistic surrogate model over~$f$ and using an
\emph{acquisition function}~$\alpha(\mathbf{x}\mid \mathbf{X}_t,\mathbf{Y}_t)$ to determine the next input point in $T$ steps. Specifically, at each step, next input is chosen as :
\begin{flalign}\label{BO2}
   \mathbf{x}_t= \arg\max_{\mathbf{x}}\alpha(\mathbf{x}\mid \mathbf{X}_{t-1},\mathbf{Y}_{t-1}), \quad\quad t=1,...,T
\end{flalign}
At each step, BO updates the probabilistic model and chooses the next input based on previous observations up to step $T$. Various acquisition functions have been proposed in the literature, including Expected improvement~\citep{Jones1998EfficientGO}, UCB~\citep{JIA2021728}, Entropy search~\citep{Hennig_Osborne_Kersting_2022}, and \textcolor{black}{VES-Gamma}~\citep{cheng2025a}, among others.

\subsubsection{Exploration-Exploitation trade-off}
Since our goal is to maximize \( f(\mathbf{x}) \) with as few queries as possible, BO seeks to simultaneously approximate \( f \) globally while identifying its maximum. This corresponds to an exploration-exploitation strategy, where the posterior mean of the surrogate function drives the exploitation, and the posterior variance or uncertainty of the surrogate guides the exploration. As discussed, this is done by optimizing an acquisition function, which guides the sampling of the best next query.
Consider a set of sampled input points $\mathbf{X}_{T}=[\mathbf{x}_{1},...,\mathbf{x}_{T}]$, and $\mathbf{Y}_{T}=[y_{1},...,y_{T}]$ their corresponding noisy observations, where $y_{t}=f(\mathbf{x}_{t})+\epsilon_{t}$, and $\epsilon_{t}\sim \mathcal{N}(0,\sigma_{t}^{2})$, and $\mathbf{f}_{\mathbf{X}_T}=[f(\mathbf{x_{1}}),...,f(\mathbf{x}_{T})]$ are unknown function values. We aim to choose a set of points that provides the maximal information to approximate $f$, and approximate it with the fewest possible queries. To do that, the concept of information gain~\citep{srinivas2010gaussian} has been used in the literature, which quantifies how much information is gained about $f$ through observations $\mathbf{Y}_{T}$ as follows:
\begin{equation}\label{eq:0}
    I(\mathbf{Y}_{T};f)=H(\mathbf{Y}_{{T}})-H(\mathbf{Y}_{{T}}\mid f),
\end{equation}
$H(\cdot)$ and $H(\cdot\mid \cdot)$ denote the entropy and conditional entropy.%
Considering the GP assumption, it can be seen that we have: \begin{equation}F(\mathbf{X}_{T})=I(\mathbf{Y}_{{T}};\mathbf{f}_{\mathbf{X}_T})=\frac{1}{2}\log (\mid \mathbf{I}+\sigma^{-2}\mathbf{K}_{\mathbf{X}_{T}} \mid),\end{equation} where $\mathbf{K}_{\mathbf{X}_T}=[k(\mathbf{x},\mathbf{x}^{'})]_{\mathbf{x},\mathbf{x}^{'}\in {\mathbf{X}_{T}}}$, where $k(\mathbf{x},\mathbf{x}^{'})$ represent covariance function between two input points $\mathbf{x}$, and $\mathbf{x}^{'}.$\footnote{See Appendix \ref{GBO}.} Here the goal is to select points that maximize information gain. Although finding these points among different $\mathbf{X}\subset \mathcal{D}$, $\mid \mathbf{X}\mid \leq T$ is NP hard, it can be done by sequentially maximizing $F$ as follows~\citep{NEURIPS2023_3feb8ed3,srinivas2010gaussian}:

Imagine that we have observed up to $t-1$ input points ( set $\mathbf{X}_{t-1}$). The optimal selection strategy for choosing the next input is such that \begin{equation}\mathbf{x}_{t}=\arg \max_{\mathbf{x}\in \mathcal{D}}F(\mathbf{X}_{t-1}\cup \{\mathbf{x}\}),\end{equation}  In case of Gaussian assumption on $f$ this reduces to~\citep{srinivas2010gaussian} \begin{equation} \label{eq:1}\mathbf{x}_{t}=\arg\max_{\mathbf{x} \in \mathcal{D}} \sigma_{t-1}(\mathbf{x}),\end{equation} where $\sigma_{t-1}$ is the posterior standard deviation at point $\mathbf{x}$ after observing previous points. In this case posterior variance has a closed-form (See Appendix \ref{GBO}). In next section we discuss how we will solve for exploration term to a general case other than Gaussian.
Although using the strategy of using only~\Eqref{eq:1} is a good way to
explore the function globally and approximate it, it is not particularly well suited to find the maximum of the function $f$. In this case, the objective is to~\citep{srinivas2010gaussian} quickly identify locations
$\mathbf{x}$ where
$f(\mathbf{x})$ attains high values, allowing the algorithm to focus sampling efforts there and effectively exploit the current knowledge. The idea is to select points that maximize posterior mean at sampled point. In case of Gaussian assumption~\citep{srinivas2010gaussian} we have:
\begin{equation}\label{eq:2}
    \mathbf{x}_{t}=\arg\max_{\mathbf{x}\in \textcolor{black}{\mathcal{D}}} \mu_{t-1}(\mathbf{x}),
\end{equation}
to maximize the posterior GP mean, $\mu_{t}(\mathbf{x})$, of the function $f$ so far, which has closed-form expression (Appendix \ref{GBO}). To take into account both exploring and exploiting strategies in~\Eqref{eq:1}, and~\Eqref{eq:2} GP-UCB rule has been applied \textcolor{black}{as~\citep{srinivas2010gaussian,iwazaki2025improved,whitehouse2023sublinearregretgpucb,pmlr-v70-chowdhury17a}:}%
\begin{equation}\label{eq:3}
    \mathbf{x}_{t}=\arg\max_{\mathbf{x}\in \textcolor{black}{\mathcal{D}}} \sqrt{\beta}\,\sigma_{t-1}(\mathbf{x})+ \mu_{t-1}(\mathbf{x}),
\end{equation}
where $\beta$ is a constant. This sampling rule sequentially selects points $\mathbf{x}$
such that the selected input points maximize an upper quantile of the marginal of the posterior on $f$ up to time $t-1$. In other words, with high probability:
$f(\mathbf{x})\leq \sqrt{\beta}\sigma_{t-1}(\mathbf{x})+ \mu_{t-1}(\mathbf{x})$
which accounts for choosing the name of GP-UCB bound.

Despite their success, GP-based BO methods scale poorly with the number of observations and rely on restrictive kernel assumptions, while BNN-based approaches often incur significant computational overhead due to posterior sampling and inner-loop acquisition optimization. These limitations motivate the need for BO frameworks that are both scalable and fully differentiable. In the next section, we introduce VBO-MI, a gradient-based BO framework that addresses these challenges via variational mutual information estimation.

\section{Proposed Black-Box Optimization Framework with Variational MI Estimation}\label{secvbo}
To develop our framework, \textcolor{black}{our} first observation is that the policy described in ~\Eqref{eq:3} implicitly depends on the MI between the unknown function \(f\) and the observed rewards \(\mathbf{y}_t\), conditioned on the selected actions \(\mathbf{x}_t\), as established in~\Eqref{eq:1} and~\Eqref{eq:3}. In the following sections, we leverage the results of ~\citet{pmlr-v80-belghazi18a} to derive a variational bound on the exploration term in~\Eqref{eq:3}. We then introduce a data-driven approach for estimating the exploitation term. Building on these two components, we design a loss function that enables the joint training of two neural networks aimed at solving our \textcolor{black}{black-box optimization} problem: one network responsible for identifying promising actions, and a second auxiliary (helper) network that estimates information gain and provides guidance for training both models that are trained in alternative fashion. 

\subsection{Variational Mutual Information} 
To incorporate Mutual Information (MI) into our optimization strategy without relying on restrictive assumptions (e.g., Gaussianity), we employ a variational lower bound based on the Donsker--Varadhan (DV) representation~\citep{pmlr-v80-belghazi18a,https://doi.org/10.1002/cpa.3160280102} of the Kullback--Leibler (KL) divergence.

Let $P$ and $Q$ be two probability measures over the same measurable space. The DV representation expresses the KL divergence as
\begin{flalign}\label{Fdiv}
D_{\mathrm{KL}}(P \,\|\, Q) = \sup_{D} \, \mathbb{E}_{P}[D] - \log \mathbb{E}_{Q}[e^{D}], 
\end{flalign}
where the supremum is taken over all measurable functions \(D\) such that the expectations are finite~\citep{pmlr-v80-belghazi18a,https://doi.org/10.1002/cpa.3160280102}. Now, consider random vector $\mathbf{x}$ and random variable $y$ with joint distribution $P_{\mathbf{x}, y}$, and let $Q = P_{\mathbf{x}}\times P_{{y}}$ denote the product of marginals. Then, the MI between $\mathbf{x}$ and ${y}$ can be written as
\begin{flalign}
&I(\mathbf{x}; {y}) = D_{\mathrm{KL}}\!\big(P_{\mathbf{x},{y}} \,\|\, P_{\mathbf{x} \times {y}}\big) 
=\\& \sup_{D} \, \mathbb{E}_{P_{\mathbf{x},{y}}}[D(\mathbf{x},{y})] - \log \mathbb{E}_{P_{\mathbf{x}}\times P_{{y}}}[e^{D(\mathbf{x},{y})}]. \label{MIbound}
\end{flalign}
In practice, we parameterize $D(\mathbf{x}, {y})$ using a Neural Network $D_\theta(\mathbf{x}, {y})$ with parameters $\theta$, and estimate the expectations via sampling. This yields the following variational lower bound:
\begin{equation}\label{IGbound}
I(\mathbf{x}; {y}) \geq \max_{\theta \in \Theta} \, \mathbb{E}_{P_{\mathbf{x},{y}}}[D_\theta(\mathbf{x},{y})] - \log \mathbb{E}_{P_{\mathbf{x}} \times P_{{y}}}[e^{D_\theta(\mathbf{x},{y})}].
\end{equation}
The goal here would be to maximize this variational lower (by finding the optimal weights of a neural network through gradient descent) to approximate the mutual information.

Given samples of distributions $\{(\mathbf{x}^{(i)}, {y}^{(i)})\}_{n=1}^N$, the Neural Network $D_{\theta}$ is trained by minimizing the following loss function:
\begin{flalign}\label{miestimate}
L_{D_{\theta}} = - \frac{1}{B} \sum_{i=1}^B D_\theta(\mathbf{x}^{(i)}, y^{(i)}) + \log\left( \frac{1}{B^2} \sum_{i=1}^B \sum_{j=1}^B e^{D_\theta(\mathbf{x}^{(i)}, y^{(j)})} \right),
\end{flalign}
\normalsize
where \(B\) denotes the minibatch size. This formulation provides a scalable, sample-based estimator of MI, and integrate into our \textcolor{black}{black-box} optimization framework to guide exploration.%

\subsection{Exploration and Mutual Information Gain Estimation}
Using results of previous section, our goal here is to find a variational bound for exploring the function $f$.
Based on the defined setting, at time step $t$ we have access to $\mathbf{X}_{t}=[\mathbf{x}_{1},...,\mathbf{x}_{t}],$ of current chosen actions. We also, assume that these actions are chosen according to $P_{\mathbf{X}_{t}}$. Moreover, corresponding observations are $\mathbf{Y}_{t}=[{y}_{1},..,{y}_{t}]$, where ${y}_{t}=f(\mathbf{x}_{t})+{\epsilon_{t}}$, and
$\epsilon_{t}$ is zero mean, independent, identically distributed (i.i.d.) noise.  Our goal in this section is to find a lower bound on the information gain quantity which corresponds to the exploration term in~\Eqref{eq:3}, which is:
\begin{equation}\label{eq:MI-exp}
    I_{G}=I(f(\mathbf{X}_t);\mathbf{Y}_{t}),
\end{equation}
where we have\footnote{Depending on its argument, $f(\mathbf{x}_{t})$ represents a scalar value, or a vector, $f(\mathbf{X}_{t})$.} $\mathbf{Y}_{t}=f(\mathbf{X}_t)+{\mathcal{E}}_{t}$, $\mathcal{E}_{t}=[\epsilon_{i}]_{i=1}^{t}$, and $f(\mathbf{X}_t)=[f(\mathbf{x}_{i})]_{i=1}^{t}$.\footnote{ Note that here we don't have a Gaussian assumption (for $f$ or noise), and therefore we cannot use the closed form for Gaussian distribution entropy)}
To find a simple bound on this quantity using the bound of~\citet{pmlr-v80-belghazi18a}, and~\Eqref{MIbound}, we first note that we have the following Markov chain\textcolor{black}{~\citep{cover2012elements}}:
\begin{equation}\label{markov0}
    \mathbf{X}_{t}\longleftrightarrow f(\mathbf{X}_{t})\longleftrightarrow\mathbf{Y}_{t}\longleftrightarrow y_t.
\end{equation}
Based on the data-processing inequality~\citep{cover2012elements} we have 
\begin{equation}\label{IG-f-1}
    I(\mathbf{X}_{t};{y}_{t})\leq I_{G}=I(f(\mathbf{X}_t),\mathbf{Y}_{t}).
\end{equation}
By expanding the mutual information via chain rule we have:
\begin{equation}
    I(\mathbf{X}_{t};y_t)=I(\mathbf{x}_{t};y_t)+I(\mathbf{x}_{t-1};y_{t}\mid \mathbf{x}_{t})+I(\mathbf{x}_{t-2};y_t\mid \mathbf{x}_{t},\mathbf{x}_{t-1})+\dots+I(\mathbf{x}_{1};y_{t}\mid \mathbf{x_{t}},\dots \mathbf{x}_{2}),
\end{equation}
except for the first term, all the conditional terms are zero due to presence of condition on $\mathbf{x}_{t}$, and conditional mutual information\footnote{Given $\mathbf{x}_t$, the observation $y_t = f(\mathbf{x}_t) + \varepsilon_t$
is determined directly through the fixed objective $f$ and on the
i.i.d.\ noise $\varepsilon_t$. Assuming sequential selection of $\mathbf{x}_{t}$, neither of the previous $\mathbf{x}_{t-k}$
 $k \geq 1$ affect observation if $\mathbf{x}_{t}$ is known.
Therefore $\mathbf{x}_{t-k} \perp y_t \mid \mathbf{x}_t$ for all $k \geq 1$.} definition~\citep{cover2012elements}. On the other hand, by using the variational representation in~\Eqref{Fdiv}:
\begin{equation}\label{eq:IG-f}
   I_{\theta}(\mathbf{x}_{t};{y}_{t})= \sup_{\theta\in\Theta}\mathbb{E}_{P_{\mathbf{x}_{t},{y}_{t}}}[D_{\theta}(\mathbf{x}_{t},{y}_{t})]-\log (\mathbb{E}_{ P_{\mathbf{x}_{t}}\times P_{{y}_{t}}}[e^{D_{\theta}(\mathbf{x}_{t},{y}_{t})}]).
\end{equation}
By noting that $I_{\theta}(\mathbf{x}_{t},{y}_{t})\leq I(\mathbf{x}_{t},{y}_{t}),$ and combining~\Eqref{IG-f-1} and~\Eqref{eq:IG-f} we have the following bound on the exploration term, or \textcolor{black}{$I_{G}$}bound in~\Eqref{eq:3}:
\begin{equation}\label{eq:IG-f+1}
    \mathbb{E}_{P_{\mathbf{x}_{t},{y}_{t}}}[D_{\theta}(\mathbf{x}_{t},{y}_{t})]-\log (\mathbb{E}_{ P_{\mathbf{x}_{t}}\times P_{{y}_{t}}}[e^{D_{\theta}(\mathbf{x}_{t},{y}_{t})}]) \leq I_{G}.
\end{equation}
We choose $y_{t}$, the current observation, rather than the full history $\mathbf{Y}_{t}$: by the Markov chain in~\Eqref{markov0}, $y_{t}$ captures the \emph{marginal} information gained at step $t$ given the already-observed history, avoiding double-counting of information already incorporated in previous iterations. This choice allows the 
exploration signal to remain focused on the incremental uncertainty 
reduction provided by the newest query, rather than considering redundant historical observations which would also increase the complexity of evaluating the mutual information bound.
\subsection{A Data-Driven Exploitation Term}
To take into account the exploitation capability of our algorithm, we also estimate exploitation term using a data-driven approach. The goal is use these results in our acquisition function (loss of a neural network) to choose the best actions to approximate and maximize $f$. In other word, our loss function should simultaneously consider the balance of exploration and exploitation in searching the input space. In the previous subsection, we have seen how the information gain quantity can be leveraged to consider the exploration of the input space. To account for the exploitation capability of our algorithm, we need to consider finding the maximum of $f$ by utilizing our current knowledge. A straightforward approach to find the maximum is to consider the expectation of the function $f$ conditioned on previous points (posterior of $f$), e.g:
\begin{equation}
    \mathbb{E}[f(\mathbf{x}_{t})\mid y_{t-1},..,y_{1},\mathbf{x}_{t-1},..,\mathbf{x}_{1}].
\end{equation}
Furthermore, we note that by assuming that $f$ and noise are independent, and noise has a zero mean, we have:
\begin{flalign}
   &\mathbb{E}[f(\mathbf{x}_{t})\mid y_{t-1},..,y_{1},\mathbf{x}_{t-1},..,\mathbf{x}_{1}]=\nonumber \\&\mathbb{E}[f(\mathbf{x_{t}})+\epsilon_{t}\mid y_{t-1},..,y_{1},\mathbf{x}_{t-1},..,\mathbf{x}_{1}]= \\&\mathbb{E}[y_{t}\mid y_{t-1},..,y_{1},\mathbf{x}_{t-1},..,\mathbf{x}_{1}].\label{explt1}
\end{flalign}
In practice, computing $\mathbb{E}[y_t \mid y_{t-1},\dots,y_1,
\mathbf{x}_{t-1},\dots,\mathbf{x}_1]$ analytically requires a 
parametric surrogate such as a Gaussian Process, which reintroduces 
the assumptions we seek to avoid.  Instead, we approximate this 
conditional expectation directly from data using a neural network 
surrogate $\hat{f}_\psi : \mathcal{D} \to \mathbb{R}$, trained by 
minimizing the mean squared error on all past observations in replay buffer 
$\mathcal{H} = \{(\mathbf{x}^{(i)}_{t'}, y^{(i)}_{t'})\},\, t^{'}\leq t$:
\begin{equation}
    \mathcal{L}_{\mathrm{MSE}}(\psi) 
    = \frac{1}{B}\sum_{i=1}^{B}
      \bigl(\hat{f}_\psi(\mathbf{x}^{(i)}) - y^{(i)}\bigr)^2,
      \quad 
      \{(\mathbf{x}^{(i)}, y^{(i)})\}_{i=1}^{B} 
      \overset{\mathrm{i.i.d.}}{\sim} \mathrm{Uniform}(\mathcal{H}).
\end{equation}
 The trained surrogate $\hat{f}_\psi(\mathbf{x})$ then serves as a 
point estimate of $\mathbb{E}[y_t \mid \mathcal{H}]$, which we call it $\hat{y}_{t}$ at any proposed 
action $\mathbf{x}$, and its gradient $\nabla_\mathbf{x}\hat{f}_\psi$ 
provides a differentiable exploitation signal for updating $E_\phi$ 
without any parametric assumption on $f$ and without querying the 
black-box function.

\begin{figure*}[t]
\centering
\resizebox{1.0\textwidth}{!}{%
\tikzset{every picture/.style={line width=0.75pt}}
\begin{tikzpicture}[x=0.75pt,y=0.75pt,yscale=-1,xscale=1]

\draw [color={rgb,255:red,52;green,152;blue,219},draw opacity=0.30,
       fill={rgb,255:red,52;green,152;blue,219},fill opacity=0.13]
  (0,90) rectangle (310,455);

\draw [color={rgb,255:red,26;green,188;blue,156},draw opacity=0.28,
       fill={rgb,255:red,26;green,188;blue,156},fill opacity=0.10]
  (308,90) rectangle (608,455);

\draw [color={rgb,255:red,192;green,57;blue,43},draw opacity=0.22,
       fill={rgb,255:red,192;green,57;blue,43},fill opacity=0.07]
  (606,90) rectangle (1100,455);

\draw (150,108) node {\Large\textbf{Action-Net} $E_\phi$};
\draw (458,108) node {\Large\textbf{Surrogate} $\hat{f}_\psi$};
\draw (878,108) node {\Large\textbf{Helper} $D_\theta$};

\draw [fill={rgb,255:red,189;green,195;blue,199},fill opacity=0.65]
  (8,196) rectangle (82,240);
\draw (45,210) node {\large Seed};
\draw (45,228) node {$\mathbf{s}_t \!\sim\! \mathcal{N}(\mathbf{0},\mathbf{I})$};

\draw [fill={rgb,255:red,52;green,152;blue,219},fill opacity=0.40]
  (94,158) -- (248,178) -- (248,262) -- (94,282) -- cycle;
\draw (165,206) node {\large Action-net};
\draw (165,225) node {$E_{\phi}(\mathbf{s}_t)$};

\draw [line width=1.5,->] (82,218) -- (93,218);

\draw [fill={rgb,255:red,184;green,233;blue,134},fill opacity=0.55]
  (258,200) -- (300,200) -- (300,236) -- (258,236) -- cycle;
\draw (279,210) node {\small Input};
\draw (279,225) node {$\mathbf{x}_t^{(i)}$};

\draw [line width=1.5,->] (248,218) -- (257,218);

\draw [fill={rgb,255:red,26;green,188;blue,156},fill opacity=0.35,
       rounded corners=8pt]
  (320,185) -- (498,185) -- (498,253) -- (320,253) -- cycle;
\draw (409,207) node {\large Surrogate network};
\draw (409,228) node {$\hat{f}_{\psi}(\mathbf{x}_t^{(i)})\!=\!\hat{y}_t^{(i)}$};

\draw [line width=1.5,->] (300,218) -- (319,218);

\draw [fill={rgb,255:red,184;green,233;blue,134},fill opacity=0.55]
  (508,200) -- (560,200) -- (560,236) -- (508,236) -- cycle;
\draw (534,210) node {\small Pred.};
\draw (534,225) node {$\hat{y}_t^{(i)}$};

\draw [line width=1.5,->] (498,218) -- (507,218);

\draw [fill={rgb,255:red,192;green,57;blue,43},fill opacity=0.28,
       rounded corners=8pt]
  (625,175) -- (910,175) -- (910,265) -- (625,265) -- cycle;
\draw (737,207) node {\large LSTM + MLP};
\draw (737,230) node {$D_{\theta}\!\left(\mathbf{x}_t^{(i)},\,\hat{y}_t^{(i)}\right)$};

\draw [line width=1.5,->] (560,218) -- (624,218);

\draw [fill={rgb,255:red,192;green,57;blue,43},fill opacity=0.22,
       rounded corners=5pt]
  (925,192) -- (1070,192) -- (1070,245) -- (925,245) -- cycle;
\draw (997,210) node {$L_{D_\theta}$};
\draw (997,230) node {\small (DV variational bound)};

\draw [line width=1.5,->] (910,218) -- (924,218);

\draw [line width=1.3,->,dashed,
       color={rgb,255:red,192;green,57;blue,43},draw opacity=0.90]
  (947,192) -- (947,143) -- (717,143) -- (717,175);
\draw (840,133) node [color={rgb,255:red,192;green,57;blue,43}]
  {$\nabla_\theta L_{D_\theta}$};

\draw [fill={rgb,255:red,26;green,188;blue,156},fill opacity=0.25]
  (458,128) -- (518,148) -- (458,168) -- (398,148) -- cycle;
\draw (458,148) node {\large $L_{E_\phi}$};

\draw [line width=1.2,->,dashed,draw opacity=0.70]
  (459,185) -- (459,168);%
\draw (385,138) node {\small exploit};

\draw [line width=1.2,->,dashed,draw opacity=0.70]
  (695,175) -- (695,148) -- (520,148);
\draw (613,138) node {\small explore ($I_\theta$)};

\draw [line width=1.4,->,dashed,
       color={rgb,255:red,41;green,128;blue,185},draw opacity=0.92]
  (398,148) -- (165,148) -- (165,158);
\draw (278,136) node [color={rgb,255:red,41;green,128;blue,185}]
  {$\nabla_\phi L_{E_\phi}$};

\draw [line width=1.0,dashed,
       color={rgb,255:red,127;green,140;blue,141},rounded corners=4pt]
  (252,285) -- (1090,285) -- (1090,440) -- (252,440) -- cycle;
\draw (511,300) node [color={rgb,255:red,127;green,140;blue,141}]
  {\small\textit{Single real oracle query per BO iteration}};

\draw [line width=1.5,->] (279,236) -- (279,325);
\draw [line width=1.5,->] (279,325) -- (318,325);

\draw [fill={rgb,255:red,230;green,126;blue,34},fill opacity=0.52,
       rounded corners=7pt]
  (320,310) -- (498,310) -- (498,360) -- (320,360) -- cycle;
\draw (409,326) node {\large\textbf{black-box oracle} $f$};
\draw (409,344) node {$y_t^{(i)} \!=\! f(\mathbf{x}_t^{(i)}) \!+\! \epsilon^{(i)}$};

\draw [line width=1.5,->] (498,334) -- (548,334);
\draw (560,325) node {\small $y_t^{(i)}$};

\draw [line width=1.5,->] (576,334) -- (618,334);

\draw [fill={rgb,255:red,149;green,165;blue,166},fill opacity=0.55,
       rounded corners=7pt]
  (620,310) -- (870,310) -- (870,360) -- (620,360) -- cycle;
\draw (745,326) node {\large Replay buffer};
\draw (745,345) node {$\mathcal{H} \leftarrow \mathcal{H} \cup \{(\mathbf{x}_t^{(i)}, y_t^{(i)})\}$};

\draw [line width=1.4,->,
       color={rgb,255:red,22;green,160;blue,133}]
  (745,310) -- (745,278) -- (409,278) -- (409,253);
 \draw (588,270) node [align=center,
        color={rgb,255:red,22;green,160;blue,133}]
   {\small train $\hat{f}_\psi$ on $\mathcal{L}_{\mathrm{MSE}}$ ($K_s$ gradient steps)};

\end{tikzpicture}%
}
\caption{%
VBO-MI framework. 
An \emph{action-net} $E_\phi$ maps random seeds $\mathbf{s}_t$ to candidate inputs $\mathbf{x}_t^{(i)}$.
A \emph{surrogate} $\hat{f}_\psi$ predicts rewards $\hat{y}_t^{(i)}$, enabling
oracle-free updates of the helper $D_\theta$ ($K_a$ steps) and action-net ($K_b$ steps)
via the acquisition loss $L_{E_\phi}$, which balances exploitation (surrogate mean)
and exploration (variational MI estimate $I_\theta$).
One real oracle query per iteration yields $y_t^{(i)}$, which replenishes the replay
buffer $\mathcal{H}$ used to retrain $\hat{f}_\psi$ via MSE ($K_s$ steps)\textcolor{black}{(real queries happen after internal model updates)}.
Dashed arrows indicate gradient flows.%
}
\label{fig:vertical_rotated}
\end{figure*}
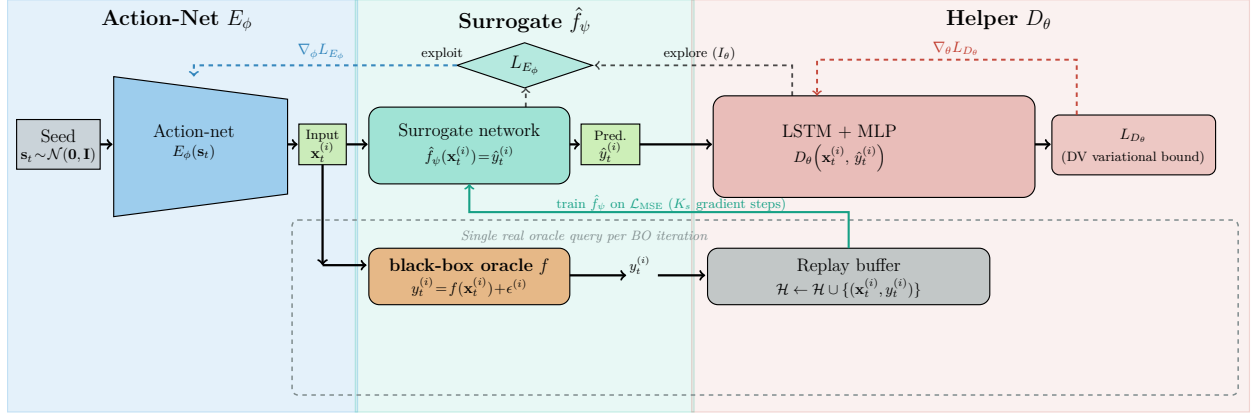
In the next part, we define a loss function and state our main framework, and algorithm for the proposed variational \textcolor{black}{black-box optimization}.

\subsection{Proposed Variational \textcolor{black}{Black-Box Optimization} with MI and Algorithm}

In this section, we develop our proposed variational acquisition
function and algorithm to find the best possible inputs of $f$ at
each step.  More specifically, we aim to solve the \textcolor{black}{black-box}
optimization problem described in~\Eqref{eq:bo_objective} in $T$
steps with batch queries of function $f$. To do so, we find the best
possible actions for $f$ using a neural network (called action-net,
$E_{\phi}$), based on maximizing an acquisition function (minimizing
as a loss function as described later), which we call $g$, through
the following optimization problem:
\textcolor{black}{
\begin{equation}\label{mnet}
\phi^* = \arg\max_\phi\, g\big(\phi, \theta^*(\phi)\big), \qquad
\theta^*(\phi) = \arg\max_\theta\, g(\phi,\theta).
\end{equation}}
where $\phi$, $\theta$ represent parameters of two neural networks,
$E_{\phi}$ and $D_{\theta}$, which we call the action-net and Helper
network respectively.  These networks are trained in an alternating
manner (similar to an actor-critic architecture) using gradient
descent. \textcolor{black}{Both levels of~\Eqref{mnet} maximize the same objective the inner level tightens the Donsker–Varadhan lower bound on mutual information for the current action-net (Phase I, Section 4.4.1), and the outer level then maximizes exploitation together with this tightened exploration term (Phase II, Section 4.4.2).} Here $g(\phi,\theta)$ denotes the variational acquisition objective combining exploitation and exploration, defined explicitly in $\Eqref{loss-f}$ after the two components are developed in Sections \ref{sec4.4.1} and \ref{sec4.4.2}.  To the best of our knowledge, this is the first work in
this area that applies variational formulations to design a fully
gradient-based \textcolor{black}{black-box} \textcolor{black}{optimization} framework using deep learning,
in contrast to earlier works that either rely on Hamiltonian Monte
Carlo sampling with BNNs~\citep{pmlr-v267-ishikura25a} or on specific
variational distribution families~\citep{cheng2025a}.

\subsubsection{Phase I: Surrogate Training and Helper Network Update}\label{sec4.4.1}
At each BO iteration $t$, the algorithm maintains a replay buffer
$\mathcal{H} = \{(\mathbf{x}^{(i)}_{t'}, y^{(i)}_{t'})\}$ of all
past input and output observations collected from $f$.  Before updating
the helper network $D_{\theta}$, we first use $\mathcal{H}$ to train
a lightweight surrogate network $\hat{f}_{\psi} : \mathcal{D} \to \mathbb{R}$, a
feed-forward neural network that approximates the conditional
expectation $\mathbb{E}[y_t \mid \mathbf{x}, \mathcal{H}]$. The optimal point estimator is:
\begin{equation}\label{eq:surr-pop}
    \psi^* = \operatorname*{arg\,min}_{\psi}\;
    \mathbb{E}_{(\mathbf{x},y)\sim\mathcal{H}}
    \bigl[\bigl(\hat{f}_{\psi}(\mathbf{x}) - y\bigr)^{2}\bigr],
\end{equation}
since the minimizer of the expected squared error is the conditional
mean $\hat{f}_{\psi^*}(\mathbf{x}) = \mathbb{E}[y \mid \mathbf{x},
\mathcal{H}]$.
Because the joint distribution under $\mathcal{H}$
is unknown, we approximate the objective
in~\Eqref{eq:surr-pop} by its sample average over a mini-batch
$\{(\mathbf{x}^{(i)}, y^{(i)})\}_{i=1}^{B}
\overset{\mathrm{i.i.d.}}{\sim} \mathrm{Uniform}(\mathcal{H})$:

\begin{equation}\label{eq:surr-loss}
    \mathcal{L}_{\mathrm{MSE}}(\psi)
    = \frac{1}{B}\sum_{i=1}^{B}
      \bigl(\hat{f}_{\psi}(\mathbf{x}^{(i)}) - y^{(i)}\bigr)^{2},
    \quad
    \{(\mathbf{x}^{(i)},y^{(i)})\}_{i=1}^{B}
    \overset{\mathrm{i.i.d.}}{\sim}
    \mathrm{Uniform}(\mathcal{H}).
\end{equation}
Because $\hat{f}_\psi$ is trained exclusively on observed
input--output pairs and its gradient $\nabla_{\mathbf{x}}\hat{f}_\psi$
is computed through the surrogate network rather than through $f$,
the black-box condition on $f$ is never violated.

Once the surrogate is updated, its predictions serve as a
query-free substitute for real function evaluations inside the helper
update.  Specifically, the action-net $E_{\phi}$ proposes a batch of
actions $\{\mathbf{x}^{(i)}_t\}_{i=1}^B$, and the surrogate
produces corresponding predictions:
\begin{equation}
    \hat{y}^{(i)}_t = \hat{f}_{\psi}(\mathbf{x}^{(i)}_t), \quad i = 1, \dots, B.
\end{equation}
These surrogate predictions $\hat{y}_t$ replace real observations in
the helper's loss, allowing $D_\theta$ to be updated without
additional queries to $f$.  The helper network $D_{\theta}$ is then
trained to estimate the \textcolor{black}{surrogate information signal}, between the proposed actions
and these \textcolor{black}{predicted observations via surrogate} using the Donsker-Varadhan
variational bound~\Eqref{MIbound}:
\begin{flalign}\label{eq:-mi}
    L_{D_{\theta}}(\mathbf{x}_{t}, \hat{y}_{t})
    =\;-I_{\theta}(\mathbf{x}_{t}, \hat{y}_{t})
    =\;-\mathbb{E}_{P_{\mathbf{x}_{t}, \hat{y}_{t}}}
             [D_{\theta}(\mathbf{x}_{t}, \hat{y}_{t})]
     +\log\mathbb{E}_{P_{\mathbf{x}_{t}}\times P_{\hat{y}_{t}}}
             [e^{D_{\theta}(\mathbf{x}_{t}, \hat{y}_{t})}],
             \end{flalign}
where the weights $\phi$ of the action-net are frozen during this
update, and $\hat{f}_{\psi}$ is kept fixed.  This phase runs for
$K_a$ gradient steps and produces an updated estimate
$I_\theta(\mathbf{x}_t, \hat{y}_t)$ of the information gain, which
is then passed to Phase~II.

\begin{remark}
   \textcolor{black}{ An approximation for the information gain using \Eqref{eq:IG-f+1} for the helper network is :\begin{equation*}
          L_{D_{\theta}}(\mathbf{x}_{t}, {y}_{t})
    =\;-I_{\theta}(\mathbf{x}_{t}, {y}_{t})
    =\;-\mathbb{E}_{P_{\mathbf{x}_{t}, {y}_{t}}}
             [D_{\theta}(\mathbf{x}_{t}, {y}_{t})]
     +\log\mathbb{E}_{P_{\mathbf{x}_{t}}\times P_{{y}_{t}}}
             [e^{D_{\theta}(\mathbf{x}_{t}, {y}_{t})}],
   \end{equation*} which is we need to use the function outputs. However, to reduce the number of required queries from the main function, in training the helper network, we instead replace $y_t$ by the surrogate output $\hat{y}_{t}$.
   As discussed later (also see Alg.\ref{alg:vbo}) see allows us to avoid the additional oracle queries that direct evaluation inside the inner loops would require, and effectively only require one batch query per iteration. We call $I_{\theta}(\mathbf{x}_{t};\hat{y}_{t})$, as a surrogate information signal in our framework. We perform ablation studies to in numerical section to consider the effect of this surrogate instead of only using direct queries in our framework. }
\end{remark}
\subsubsection{Phase II: Action-Net Update via Proposed Loss Function}\label{sec4.4.2}

With the helper network $D_\theta$ updated and frozen, the action-net
$E_{\phi}$ is trained to propose inputs that jointly maximise
exploitation of the surrogate's predicted reward and exploration
measured by the \textcolor{black}{surrogate information signal} from Phase~I.  The action-net maps a
random seed $\mathbf{s}_t \sim \mathcal{N}(0, I)$ to a candidate
action:
\begin{equation}
    E_{\phi}(\mathbf{s}_{t}) = \mathbf{x}_{t},
\end{equation}
and its loss combines the two objectives.  The \emph{exploitation}
term is the surrogate's mean predicted reward at the proposed
actions, providing a differentiable signal that pushes $E_\phi$
toward high-reward regions without querying $f$.  The
\emph{exploration} term is the negative variational MI estimate
$-I_\theta(\mathbf{x}_t, \hat{y}_t)$ from Phase~I, which encourages
$E_\phi$ to propose inputs that are informative relative to the
surrogate's current predictions.  Together, by combining the
exploitation term in~\Eqref{explt1} with the exploration bound
in~\Eqref{eq:IG-f+1}, the loss (or acquisition function) takes the form:
\begin{equation}\label{loss-f}
    L_{E_{\phi}}(\mathbf{x}_{t},\hat{y}_{t})=-g({\phi,\theta})= -\mathbb{E}[y_t\mid y_{t-1},\dots y_1,\mathbf{x}_{t-1},\dots \mathbf{x}_1]
    -\sqrt{\beta}\,\mathbb{E}_{P_{\mathbf{x}_{t},\hat{y}_{t}}}
        [D_{\theta}(\mathbf{x}_{t},\hat{y}_{t})] 
    +\sqrt{\beta}\log\mathbb{E}_{P_{\mathbf{x}_{t}}\times P_{\hat{y}_{t}}}
        [e^{D_{\theta}(\mathbf{x}_{t},\hat{y}_{t})}]
\end{equation}
\normalsize
where $\beta$ controls the trade-off between exploration and
exploitation, and $L_{E_{\phi}}$ is a \emph{loss to be minimized} by $E_\phi$,
equivalent to maximizing $\hat{f}_\psi(\mathbf{x}_t) +
\sqrt{\beta}\,I_\theta(\mathbf{x}_t, \hat{y}_t)$.  

Here the weights of $\theta$ and $\psi$ are frozen throughout.
Since both losses contain only expectation terms, they are
implemented as sample averages over the batch.  Specifically we approximate it as:
\begin{equation}
\begin{aligned}
    L_{E_{\phi}}
    =\textcolor{black}{-\frac{1}{B}\sum_{i=1}^{B}\hat{y}^{(i)}_{t}}
    -\sqrt{\beta}\bigg[\frac{1}{B}\sum_{i=1}^{B}
        D_{\theta}(\mathbf{x}^{(i)}_{t},\hat{y}^{(i)}_{t})
    -\log\frac{1}{B^{2}}
        \sum_{i=1}^{B}\sum_{j=1}^{B}
        e^{D_{\theta}(\mathbf{x}^{(i)}_{t},\hat{y}^{(j)}_{t})}\bigg]
\end{aligned}
\label{vbomi-loss}
\end{equation}
where $\hat{y}^{(i)}_{t}=\hat{f}_{\psi}(\mathbf{x}^{(i)}_{t})$.
This update runs for $K_b$ gradient steps with $D_\theta$ and
$\hat{f}_\psi$ fixed.
\textcolor{black}{\begin{remark}
    The log-mean-exp term in~\Eqref{vbomi-loss} is computed via the standard log-sum-exp stabilization: the batch maximum of $D_{\theta}$ is subtracted before exponentiating and added back after taking the log, avoiding overflow when $D_{\theta}$ takes large values.
\end{remark}}

Our framework has three neural network components.  $E_{\phi}$
searches the input space $\mathcal{D}$ and selects the next action
$\mathbf{x}_{t}$; \textcolor{black}{the helper $D_{\theta}$ estimates the surrogate information
signal and send the signal for $E_{\phi}$}; and the
surrogate $\hat{f}_{\psi}$ approximates the conditional expectation
of $f$ from past observations and supplies the exploitation gradient
to $E_{\phi}$.  All three networks are trained in an alternating
fashion toward convergence (see Fig.~\ref{fig:vertical_rotated}). In order to give the algorithm smooth, and efficient performance we have included warm-up step where only helper network, and surrogate are updated. Our detailed algorithm is summarized in Algorithm~\ref{alg:vbo}.

\begin{algorithm}[t]
    \caption{VBO-MI}
    \label{alg:vbo}
    \textbf{Input:} Initialized weights of networks $E_\phi$ (action-network),
    $D_\theta$ (MI helper network),
    and $\hat{f}_\psi$ (surrogate network);
    number of warm-up steps $W$, number of BO iterations $T$,
    iteration counts $K_a$, $K_b$, $K_p$ and $K_s$, batch size $B$. \\
    \textbf{Output:} Optimized weights of $E_\phi$, $D_\theta$, and $\hat{f}_\psi$;
    estimated optimum $\mathbf{x}^*$.
    \begin{algorithmic}[1]

    \State \textbf{Warm-up Phase:}
    \State Initialize replay buffer $\mathcal{H} \leftarrow \emptyset$.
    \For{$k = 1, \dots, W$}
        \State \textcolor{black}{Sample warm-up seeds $\mathbf{s}_{k}^{(i)} \sim \mathcal{N}(0, I)$}.
        \State  Using the seeds $\mathbf{s}_{\textcolor{black}{k}}^{(i)}$
     compute (a batch of) actions $\mathbf{x}_{\textcolor{black}{k}}^{(i)} = E_{\phi}(\mathbf{s}_{\textcolor{black}{k}}^{(i)})$,\; 
    $i = 1, \dots, B$.
        \State Evaluate observations $y_{\textcolor{black}{k}}^{(i)} = f(\mathbf{x}_{\textcolor{black}{k}}^{(i)}) + \epsilon^{(i)}$.%
        \State Store $\mathcal{H} \leftarrow \mathcal{H} \cup \{(\mathbf{x}_{\textcolor{black}{k}}^{(i)}, y_{\textcolor{black}{k}}^{(i)})\}_{i=1}^{B}$.
        \State Compute $-\mathrm{MI}$ using $\{\mathbf{x}_{\textcolor{black}{k}}^{(i)}, y_{\textcolor{black}{k}}^{(i)}\}_{i=1}^B$ and helper network.
        \State Update $D_{\theta}$ via gradient descent on $L_{D_\theta}$, keeping $E_{\phi}$ fixed and vice versa.
    \EndFor
    \State \textbf{Pre-train} surrogate: for $k=1,\dots,K_p$, sample a mini-batch
           $\{(\mathbf{x}_{\textcolor{black}{k}}^{(i)},y_{\textcolor{black}{k}}^{(i)})\} \subset \mathcal{H}$ and update $\hat{f}_\psi$ by
           gradient descent on $\mathcal{L}_{\mathrm{MSE}} =
           \frac{1}{B}\sum_{i}(\hat{f}_\psi(\mathbf{x}_{\textcolor{black}{k}}^{(i)}) - y_{\textcolor{black}{k}}^{(i)})^2$.

    \State \textbf{Main VBO-MI Phase:}
    \For{$t = 1, \dots, T$}
\State  Compute (a batch of) actions $\mathbf{x}^{(i)}_{t} = E_{\phi}(\mathbf{s}^{(i)}_{t})$,\; 
    $i = 1, \dots, B$.
        \State \textbf{Update $D_{\theta}$ (Helper network):}
        \For{$k = 1, \dots, K_a$}
            \State Compute surrogate predictions
                   $\hat{y}^{(i)}_{t} = \hat{f}_{\psi}(\mathbf{x}^{(i)}_{t})$,\;
                   $i=1,\dots,B$.  \hfill
            \State Update $D_{\theta}$ by gradient descent on $L_{D_{\theta}}$
                   with $E_{\phi}$ fixed, using $\{\mathbf{x}^{(i)}_{t},\,\hat{y}^{(i)}_{t}\}_{i=1}^{B}$.
        \EndFor

        \For{$k = 1, \dots, K_s$}
            \State Sample a mini-batch
                   $\{(\mathbf{x}^{(i)},y^{(i)})\}_{i=1}^{B} \subset \mathcal{H}$.
            \State Update $\hat{f}_\psi$ by gradient descent on
                   $\mathcal{L}_{\mathrm{MSE}} =
                   \frac{1}{B}\sum_{i}(\hat{f}_\psi(\mathbf{x}^{(i)}) - y^{(i)})^2$.
        \EndFor

        \State \textbf{Update $E_{\phi}$ (Action-Net):}
        \For{$k = 1, \dots, K_b$}
            \State  Compute (a batch of) actions $\mathbf{x}^{(i)}_{t} = E_{\phi}(\mathbf{s}^{(i)}_{t})$,\; 
    $i = 1, \dots, B$. \hfill
            \State Compute surrogate predictions
                   $\hat{y}^{(i)} = \hat{f}_{\psi}(\mathbf{x}^{(i)})$.
                   \hfill
            \State Compute MI estimate $I_\theta(\mathbf{x}^{(i)}, \hat{y}^{(i)})$
                   via helper $D_\theta$ (weights frozen).
            \State Update $E_{\phi}$ by SGD on
                     $L_{E_{\phi}}$ in~\Eqref{vbomi-loss}
                   with $D_\theta$ and $\hat{f}_\psi$ fixed.
                   \EndFor

        \State Compute actions $E_{\phi}(\mathbf{s}^{(i)}_{t})=\mathbf{x}^{(i)}_{t} $,\; $i=1,\dots,B$.
        \State Evaluate $y^{(i)}_{t} = f(\mathbf{x}^{(i)}_{t}) + \epsilon^{(i)}$.
               \hfill
        \State Store $\mathcal{H} \leftarrow \mathcal{H} \cup
               \{(\mathbf{x}^{(i)}_{t}, y^{(i)}_{t})\}_{i=1}^{B}$.

        \State Save estimated optimum
               $\mathbf{x}_{t}^{*} =
               \operatorname*{arg\,max}_{(\mathbf{x},y)\,\in\,\mathcal{H}}\; y$.

    \EndFor
    \State \textbf{Return:} final network weights $(\phi^{*}, \theta^{*}, \psi^{*})$
           and estimated optimum $\mathbf{x}_{t}^{*}$.
    \end{algorithmic}
\end{algorithm}

\begin{figure*}[t]
\centering
\begin{subfigure}{\textwidth}
    \centering
    \includegraphics[width=0.7\linewidth]{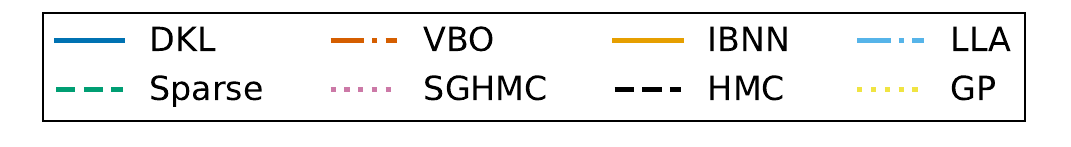}
\end{subfigure}
\vspace{0.15cm}

\begin{subfigure}{0.32\textwidth}
    \centering
    \includegraphics[width=\linewidth]{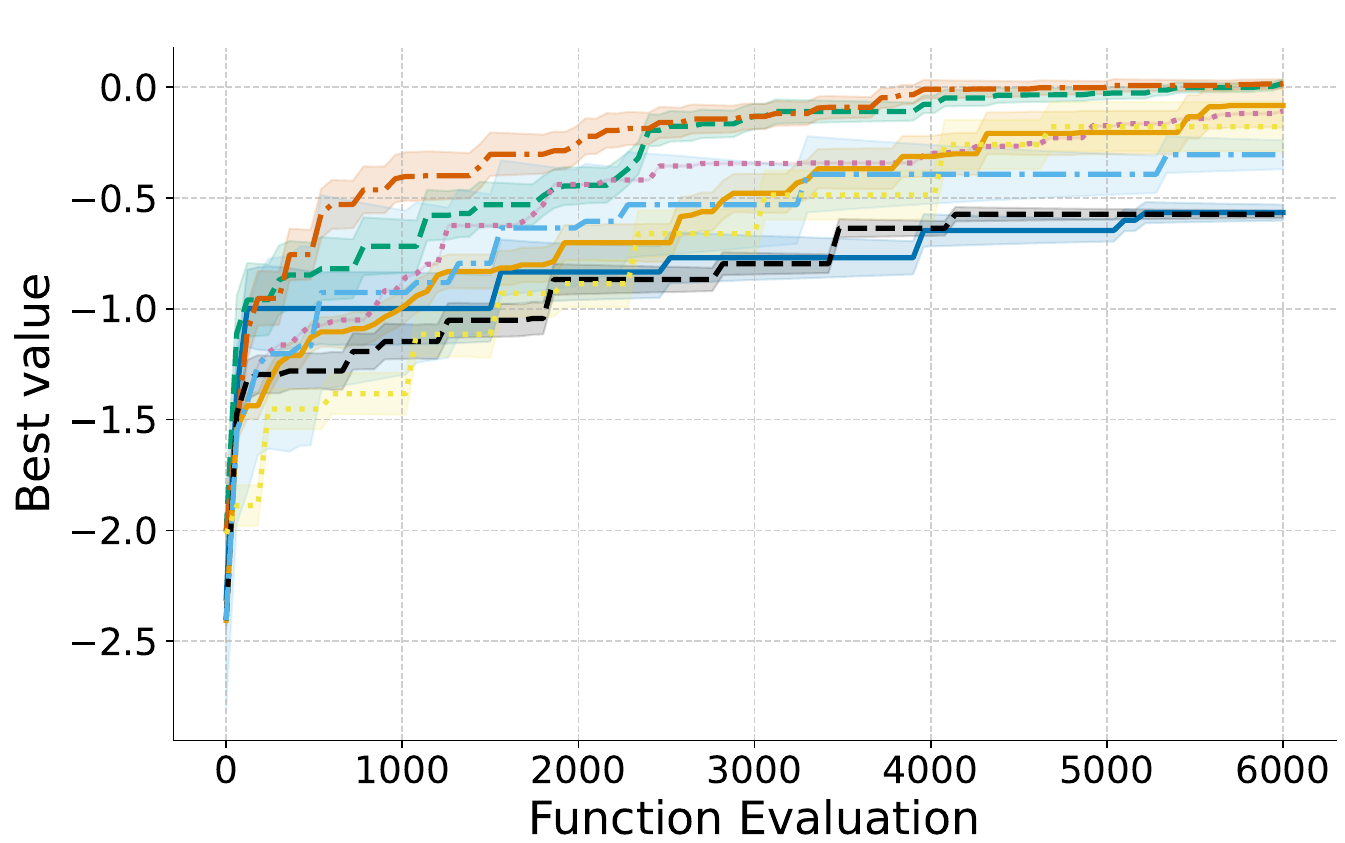}
    \caption{Ackley ($d=10$)}
    \label{fig:ackley}
\end{subfigure}
\hfill
\begin{subfigure}{0.32\textwidth}
    \centering
    \includegraphics[width=\linewidth]{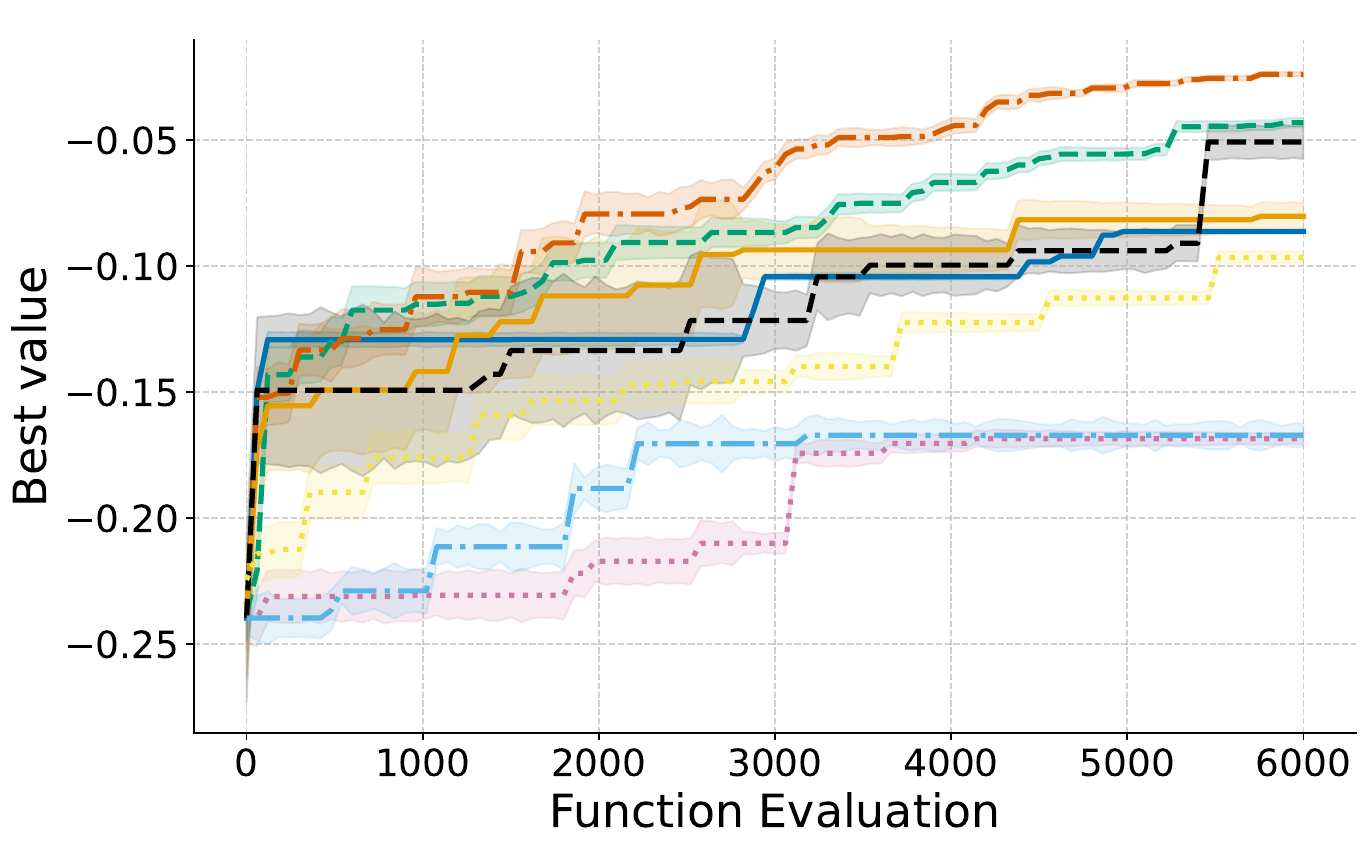}
    \caption{Levy ($d=20$)}
    \label{fig:levy}
\end{subfigure}
\hfill
\begin{subfigure}{0.32\textwidth}
    \centering
    \includegraphics[width=\linewidth]{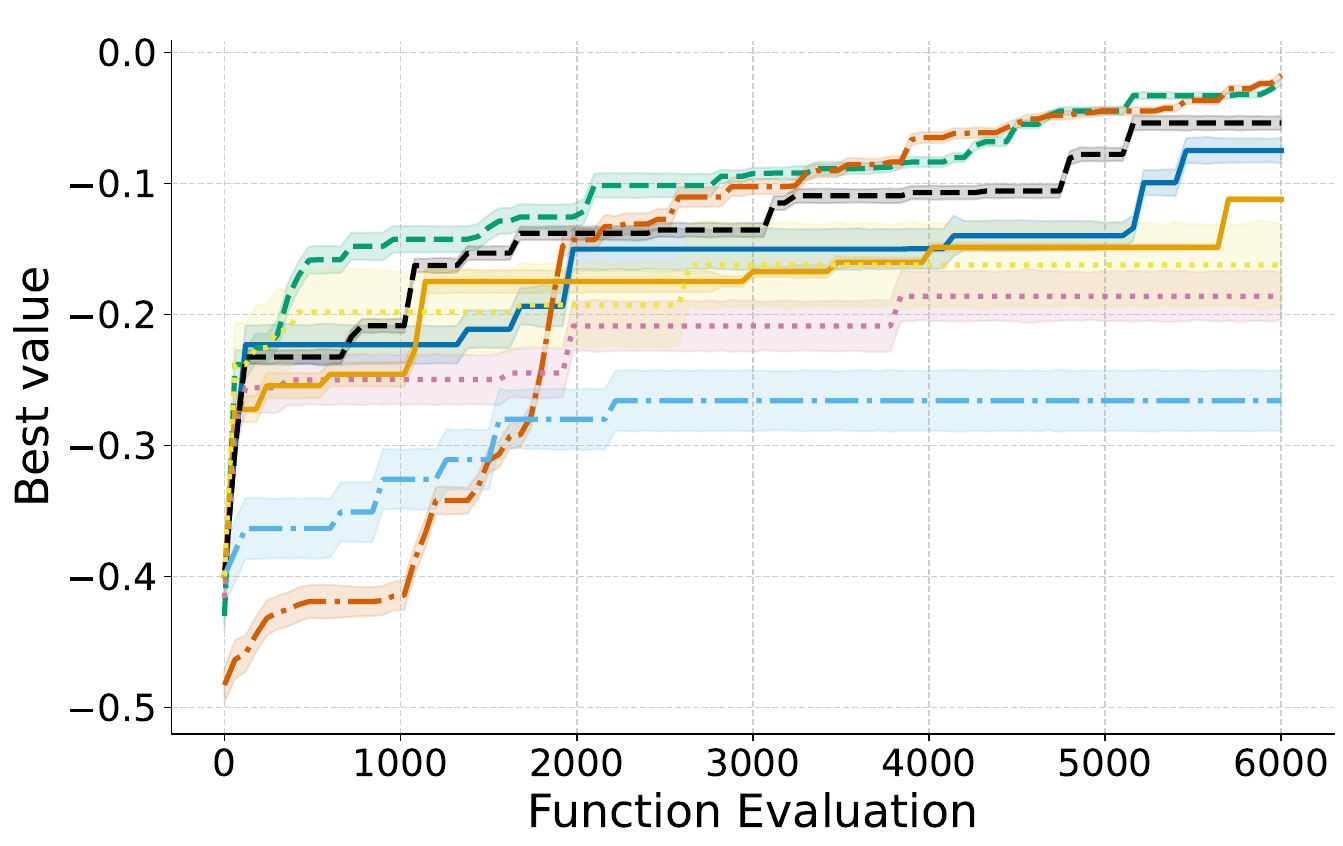}
    \caption{Griewank ($d=20$)}
    \label{fig:griewank}
\end{subfigure}

\caption{Performance comparison on synthetic benchmark tasks (a--c).
         Global optimum is 0 for all functions.
         All curves show mean $\pm$ one standard deviation over 5 trials.}
\label{fig:synthetic}
\end{figure*}

\begin{figure*}[t]
\centering
\begin{subfigure}{\textwidth}
    \centering
    \includegraphics[width=0.7\linewidth]{figs/labelscropped.pdf}
\end{subfigure}
\vspace{0.15cm}

\begin{subfigure}{0.32\textwidth}
    \centering
    \includegraphics[width=\linewidth]{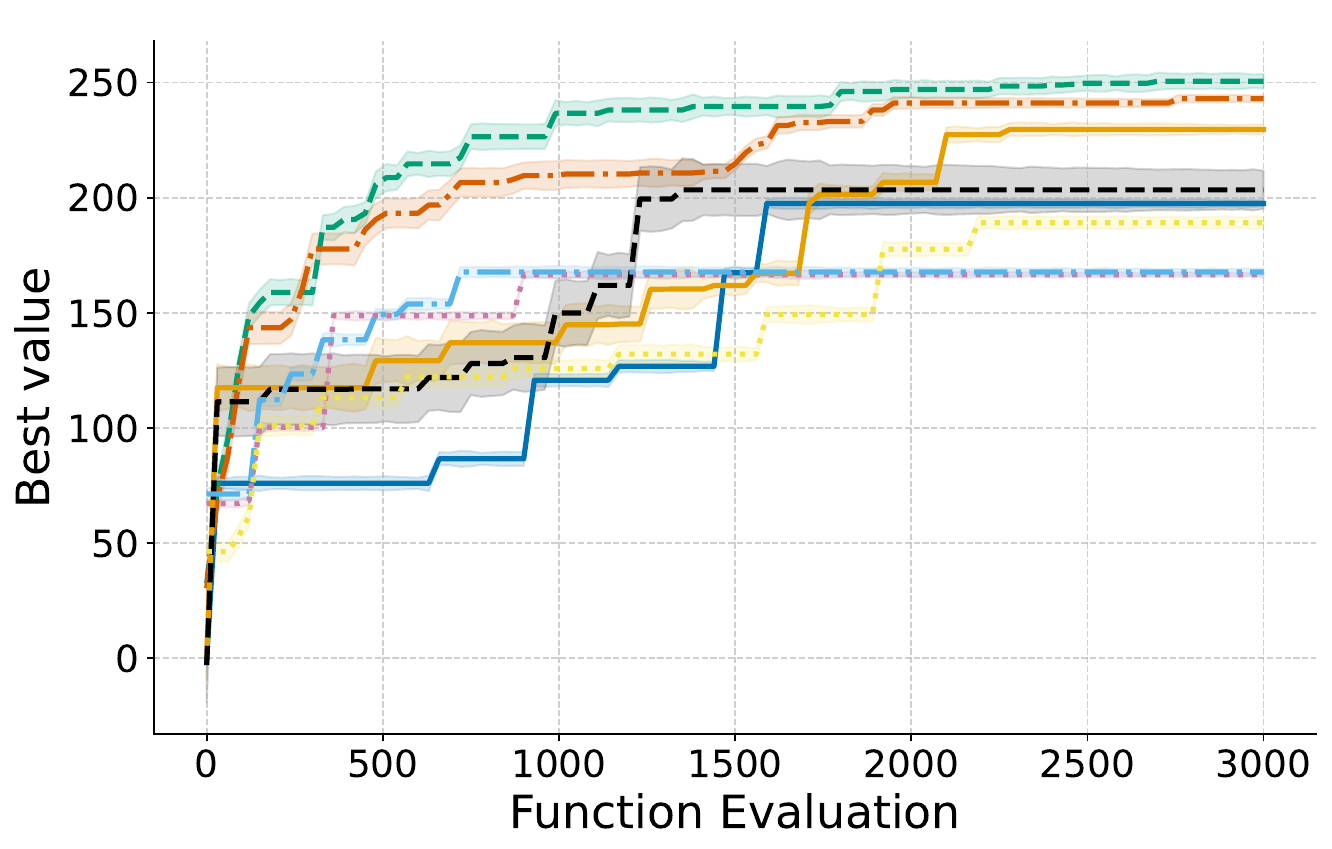}
    \caption{Lunar Lander ($d=12$)}
    \label{fig:interferometer}
\end{subfigure}
\hfill
\begin{subfigure}{0.32\textwidth}
    \centering
    \includegraphics[width=\linewidth]{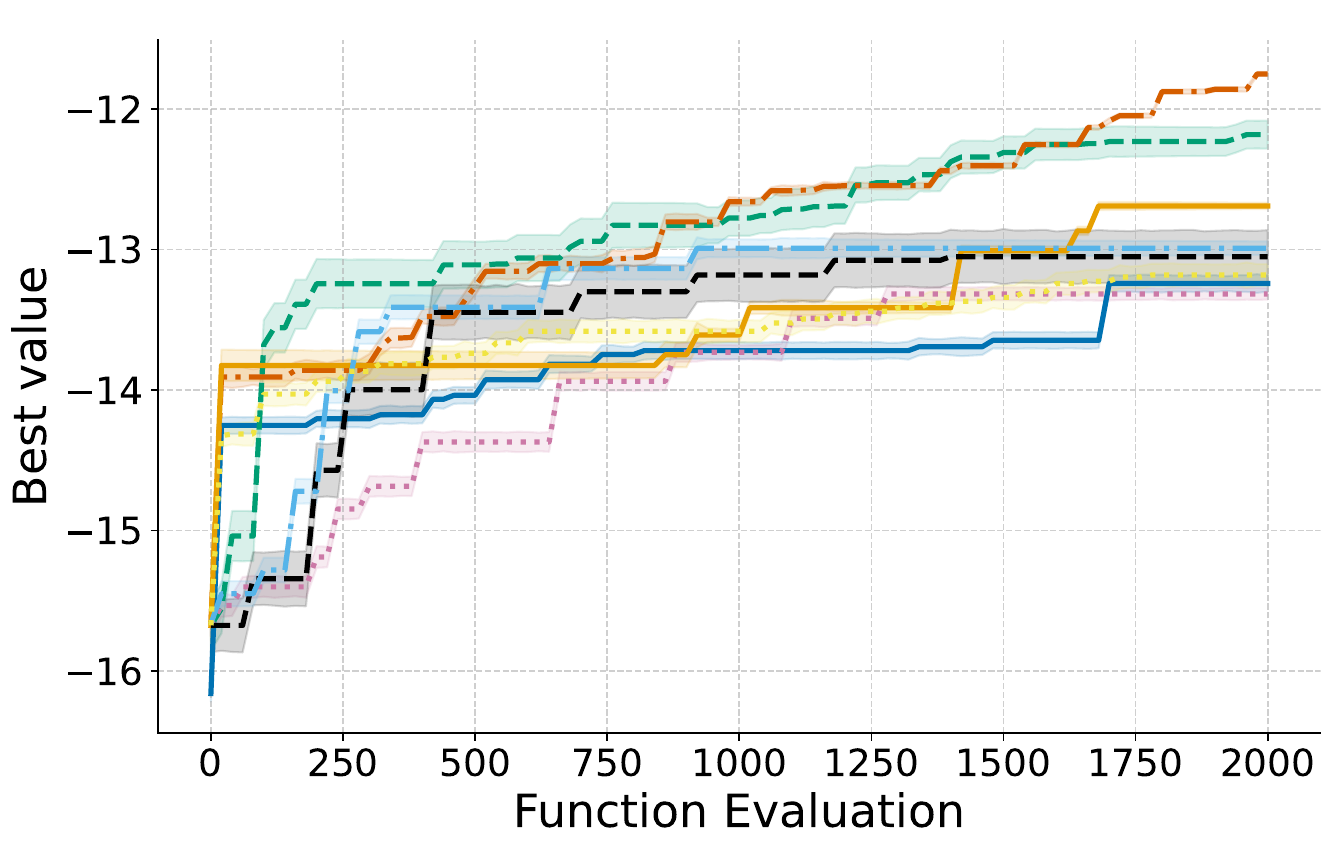}
    \caption{Pest Control ($d=25$)}
    \label{fig:pde}
\end{subfigure}
\hfill
\begin{subfigure}{0.32\textwidth}
    \centering
    \includegraphics[width=\linewidth]{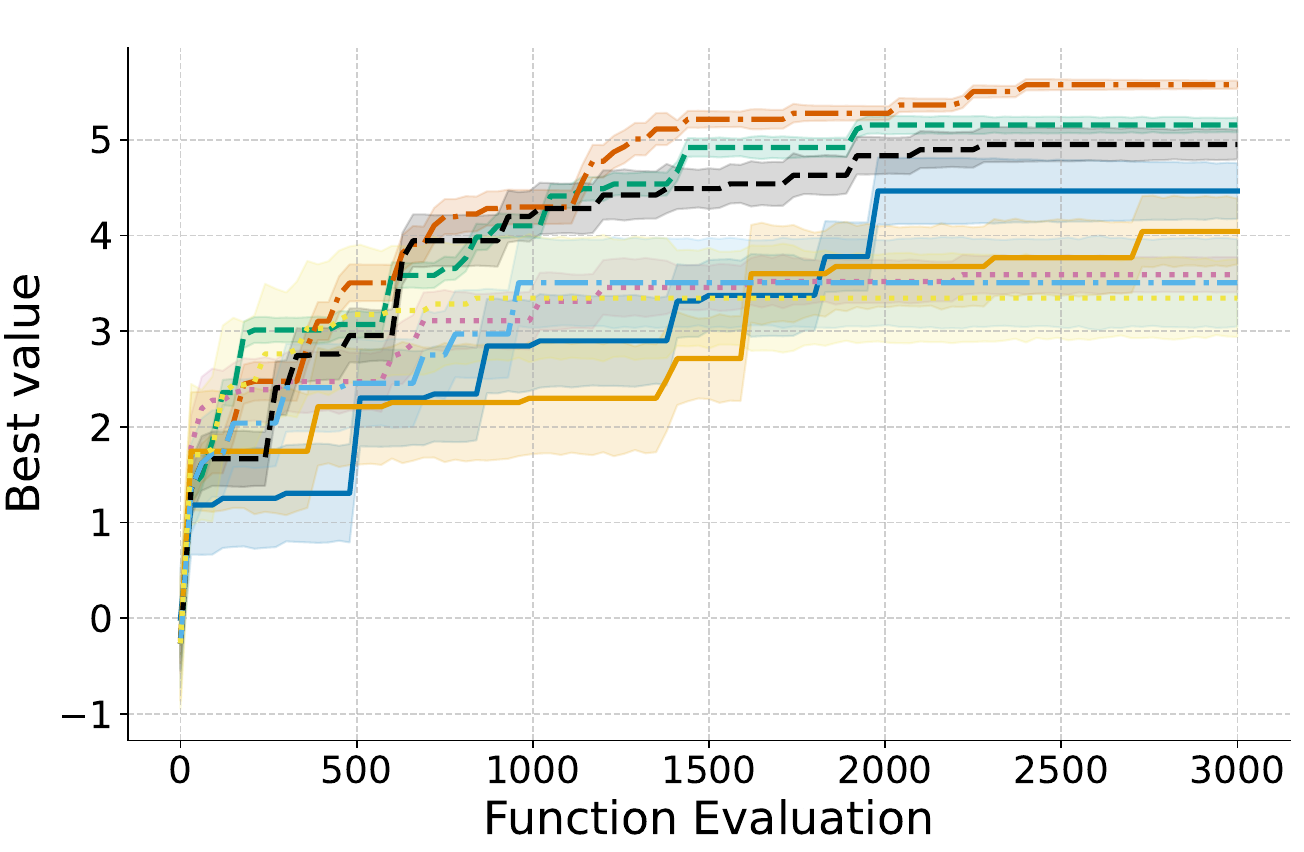}
    \caption{Rover Trajectory ($d=60$)}
    \label{fig:lunar}
\end{subfigure}

\caption{Performance comparison on real-world benchmark tasks (a--c).
         All curves show mean $\pm$ one standard deviation over 5 trials.}
\label{fig:realworld}
\end{figure*}

 \begin{remark}
It is important to distinguish between the optimization of our neural networks and the optimization of $f$. VBO-MI uses gradients to update network parameters $\phi$ and $\theta$ with respect to the variational loss. However, these gradients are computed through the neural architecture; the objective function $f$ itself remains a black box, and no gradients of $f$ are required or used.
\end{remark}
 Now we resort to the loss functions $L_{E_{\phi}}$ and $L_{D_{\theta}}$, and the outputs of $E_{\phi}$ and $D_{\theta}$, to propose our algorithm for the BO defined in Section \ref{sec:background}. 
 In the proposed method the action-net $E_{\phi}$, 
 searches the input space $\mathcal{D}$ and selects the next query $\mathbf{x}_{t}$; 
the helper $D_{\theta}$, \textcolor{black}{which approximates surrogate information signal and helps to guide exploration 
signal} for $E_{\phi}$; and the surrogate $\hat{f}_{\psi}$, which approximates the 
conditional expectation of $f$ from accumulated observations and supplies the exploitation 
gradient. The core optimization follows an alternating actor-critic scheme between 
$E_{\phi}$ and $D_{\theta}$: at each iteration, $D_{\theta}$ is updated first to sharpen 
the MI estimate, then $E_{\phi}$ is updated using the frozen $D_{\theta}$. The surrogate 
$\hat{f}_{\psi}$ is updated separately via MSE minimization on the replay buffer 
$\mathcal{H}$, serving as an oracle-free substitute for $f$ throughout both updates. 
The complete procedure is summarized in Algorithm~\ref{alg:vbo} and illustrated in 
Fig.~\ref{fig:vertical_rotated}.

\subsection{Theoretical Results}
\label{theoreticalresults}

{\color{black}In this section, we establish two types of theoretical guarantees for
VBO-MI. First, we show  the sample-average
training objective converges to its population counterpart as the
batch size $B$ grows (Theorem~\ref{thm:consistency}a), and the surrogate
network converges to $f$ as the replay buffer $\mathcal{H}$ grows
(Theorem~\ref{thm:consistency}b). Second, under idealized conditions
detailed in Proposition~\ref{prop:existence}, we show that networks
realizing an $\epsilon$-optimal solution \emph{exist} within our
hypothesis class. We emphasize that this second result concerns
what the hypothesis class contains, not what the alternating training
procedure of Algorithm~\ref{alg:vbo} is achieved in finite $T$ iteration. Establishing such guarantee requires a regret bound for VBO-MI as a function of the number of
BO iterations $T$, which would require showing that the acquisition
objective in~\Eqref{vbomi-loss} forms a valid confidence bound on
$f$ and that its exploration term decays at a suitable rate as $T$
grows, remains an open
problem, and we discuss it as future work in Section~\ref{sec:conclusion}.
\begin{theo}
\label{thm:consistency}
Let $f: \mathcal{D} \to \mathbb{R}$ be a continuous black-box function,
and let $\widehat{g}(\phi, \theta)$ be the sample-average
approximation of the variational loss $g(\phi, \theta)$ as
defined in~\Eqref{mnet} and~\Eqref{vbomi-loss}.
Assume the following conditions hold:
\begin{enumerate}[label=(\roman*)]

   \item The observation noise $\epsilon_t$
satisfies $\mathbb{E}[\epsilon_t \mid \mathcal{F}_{t-1}] = 0$ and
$\mathbb{E}[\epsilon_t^2 \mid \mathcal{F}_{t-1}] < \infty$, where
$\mathcal{F}_t = \sigma(\mathbf{x}_1,y_1,\ldots,\mathbf{x}_t,y_t)$ 
denotes the $\sigma$-algebra generated by the history observed up to
step $t$; that
is, $\{y_t - f(\mathbf{x}_t)\}$ forms a martingale difference sequence
(MDS) with respect to $\{\mathcal{F}_t\}$.

     \item
    The networks $E_\phi$, $D_\theta$, and $\hat{f}_\psi$ belong to 
    universal approximation classes --- feed-forward networks for 
    $E_\phi$ and $\hat{f}_\psi$, and LSTMs for 
    $D_\theta$~\citep{Hornik1989MultilayerFN,rnn} with 
    sufficient capacity. The empirical occupation measure of the replay buffer's inputs, $\mu_n := \frac{1}{n}\sum_{i=1}^n \delta_{x_i}$, converges weakly, $\mu_n \to \mu_\infty$, $P$-almost surely, to some limiting measure $\mu_\infty$ on $\mathcal{D}$. \textsuperscript{\ref{rem:lmmeasure}}

    \item%
         The parameter spaces $\Phi$,
$\Theta$, $\Psi$ are compact, and $\hat{f}_\psi$ is realized by the
feedforward architecture of Section~C.1 with layerwise Lipschitz
activations (e.g, Relu).%

\end{enumerate}
Then we will prove the following conclusions in Appendix~\ref{app:consistency}:
\begin{enumerate}[label=(\alph*)]
    \item \textbf{(Within-batch convergence)} For any fixed
    $(\phi,\theta)\in\Phi\times\Theta$, the sample average approximation
    $\widehat{g}(\phi,\theta)$ converges to $g(\phi,\theta)$
    $\mathbb{P}$-almost surely as the batch size $B\to\infty$.

    \item \textbf{(Surrogate convergence)} The MDS structure of
    $\{y_t - f(\mathbf{x}_t)\}$, and
    the martingale-adapted uniform convergence result of
    \citep{lauer2023uniform} applied to the surrogate's induced
    function class, gives:
    \[
        \sup_{x\in \mathrm{supp}(\mu_{\infty})}\,|\hat{f}_{\psi^*}(x) - f(x)|
        \;\xrightarrow{|\mathcal{H}|\to\infty}\; 0
        \quad \mathbb{P}\text{-almost surely}.
    \]

 \end{enumerate}
 \end{theo}
 \begin{remark}
 \label{rem:lmmeasure}
     Since $\mathcal{D}$ is compact, $(\mu_n)$ is automatically tight, so every subsequence admits a further weakly convergent sub-subsequence (Prokhorov's theorem~\cite{prokhorov1956convergence}); what Assumption (ii) additionally requires is that the whole sequence converges to a single limit, i.e., that the query distribution induced by the alternating training dynamics of Algorithm 1 stabilizes asymptotically. We do not establish this stabilization for the general nonconvex alternating stochastic optimization underlying Algorithm 1, and state it explicitly as a working assumption. 
     
     Intuitively, it is supported by the following: if the action-net's parameters $\phi$ approach a neighborhood of a stationary point as training proceeds, as suggested empirically by the loss curves in our experiments, then $E_\phi$ approaches a fixed map, and the induced distribution of proposed actions $\mathbf{x}_t = E_\phi(\mathbf{s}_t)$ correspondingly stabilizes; the buffer's empirical occupation measure, accumulating draws from this progressively stabilizing distribution, would then inherit weak convergence to the limit.

 \end{remark}
 \begin{proof}
     See Appendix \ref{app:consistency}.
 \end{proof}
 \begin{proposition}[Existence of Near-Optimal Networks within the Hypothesis Class]
\label{prop:existence}
Under Assumptions~(i)--(iii) of Theorem~\ref{thm:consistency}, for any
$\epsilon>0$ there exist networks $E_{\hat\phi}$, $D_{\hat\theta}$,
$\hat f_{\hat\psi}$ — realizable within the hypothesis classes
$\Phi,\Theta,\Psi$ — such that output of $E_{\hat\phi}$ is $\epsilon$-optimal for $f$.
\footnote{Note that this is a capacity-type result: it concerns only what the
hypothesis class contains, not whether the alternating stochastic
optimization of Algorithm~\ref{alg:vbo} is guaranteed to find such a
network within finite $T$ steps, and what regret bound holds on the
resulting solution, which remain open problems.}
\end{proposition}
\begin{proof}
See Appendix \ref{lem:concentrated-policy}.

\end{proof}}

\section{Numerical Evaluation}\label{sec:eval}

In this section, we present numerical evaluations of our proposed method,
Variational Black-Box optimization with Mutual Information (VBO-MI), and compare
its performance with several baselines~\citep{li2024a, NEURIPS2024_da30215e}.
These baselines include the following methods for Bayesian neural networks:
Laplace Linear Approximation (LLA)~\citep{pmlr-v130-immer21a, li2024a},
Hamiltonian Monte Carlo (HMC)~\citep{pmlr-v139-izmailov21a, li2024a},
Stochastic Gradient Hamiltonian Monte Carlo
(SGHMC)~\citep{pmlr-v32-cheni14, li2024a},
Deep Kernel Learning (DKL)~\citep{pmlr-v51-wilson16, li2024a},
Infinite-width Bayesian Neural Networks
(IBNN)~\citep{lee2018deep, li2024a},
and Gaussian Processes (GPs)~\citep{1snoek, li2024a}.
In addition, we include the Focal Sparse Gaussian Process
method~\citep{NEURIPS2024_da30215e} as a recent baseline.
We present experiments on two settings: first, three synthetic benchmark
functions; then, three real-world tasks of increasing dimensionality.
\paragraph{Implementation.}
For VBO-MI, we use a batch size of $B=30$ and $B=60$ for Real world task and synthetic functions respectively,
a learning rate of $2\times10^{-3}$ with Adam optimizer for all three networks,
$K_a = 5$ helper update steps, $K_b = 10$ action-net update steps,
$K_s = 10$ surrogate update steps per BO iteration, and $W = 30$
warm-up iterations. The total number of evaluations considers both evaluations needed during the warm-up, and main phase of algorithm (see  Algorithm \ref{alg:vbo}).
Baselines use the hyperparameter configurations reported
in~\citet{li2024a}.
All results are averaged over 5 random seeds; shaded regions denote
$\pm 1$ standard deviation.
Full architecture details and per-task configurations are provided
in Appendix~\ref{app:setup}.
\subsection{Synthetic Functions}

We evaluate our algorithm on three widely used benchmark functions: Ackley,
Levy, and Griewank~\citep{li2024a, NEURIPS2019_6c990b7a}.
All three functions are highly non-convex, contain many local optima, and have
their global minimum at zero---making them a diverse and challenging suite for
evaluating BO methods.
Because these functions are conventionally used as minimization problems, we
optimize $-f$ instead of $f$; maximizing $-f$ is equivalent to minimizing $f$,
and the global maximum of $-f$ is therefore 0 in all cases.

In Fig.~\ref{fig:ackley}, we compare all methods on the 10-dimensional Ackley
function.
Ackley is particularly challenging due to its near-flat outer region and sharp
global basin, which makes exploration critical.
We observe that VBO-MI achieves competitive performance, consistently reaching
higher final average rewards than the BNN baselines.
The Focal Sparse GP baseline performs strongly in the early iterations but VBO-MI
closes the gap and surpasses it as the number of iterations increases.

Fig.~\ref{fig:levy} presents results on the 20-dimensional Levy function.
Levy is a scalable multimodal function whose difficulty grows with dimension,
making the 20-dimensional setting a meaningful test of high-dimensional
exploration.
In this setting, VBO-MI demonstrates faster initial convergence and achieves
a higher final reward compared to all baselines, including GP and IBNN.

Finally, Fig.~\ref{fig:griewank} evaluates all methods on the 20-dimensional
Griewank function. Unlike Ackley and Levy, Griewank is non-separable due to its
product term, meaning that its dimensions interact, which can pose challenges
for surrogate-based methods relying on stationary kernels.
\textcolor{black}{
VBO-MI achieves a final reward comparable to that of Focal Sparse GP while
outperforming the other baseline methods. We hypothesize that its variational
mutual-information-based exploration signal, which does not impose a specific
covariance structure over the input space, may be beneficial in this setting.
Additional experiments on two synthetic functions, Hartmann (6D) and
Rastrigin (20D), are provided in Appendix~E. These results further evaluate
the performance of VBO-MI across synthetic functions with different
dimensionalities and characteristics.
}

\subsection{Real-World Tasks}

We further evaluate VBO-MI on three real-world benchmark tasks of increasing
dimensionality, following the experimental protocol of prior
work~\citep{li2024a, NEURIPS2019_6c990b7a}.
These tasks test whether the method scales to higher-dimensional,
non-stationary objectives where standard GP surrogates are known to
struggle~\citep{li2024a}.
For all methods, we report the best value so far across different evaluations and
report results averaged over multiple trials with standard deviation.
Further experimental details are provided in the Appendix.

\paragraph{Lunar Lander ($d=12$).}
In Fig.~\ref{fig:interferometer}, we evaluate all methods on the
OpenAI Gym LunarLander-v2 control task~\citep{NEURIPS2019_6c990b7a}.
The goal is to optimise the 12 parameters of a controller such that the average
episode reward over 50 randomly generated environments is maximised.
\textcolor{black}{VBO-MI is competitive with the strongest baseline on this task: the Focal Sparse GP baseline attains a marginally higher final mean reward, while VBO-MI retains a clear margin over the BNN-based baselines (DKL, SGHMC, IBNN, HMC, LLA).}
The non-stationarity of the reward landscape as a function of controller
parameters makes this task well-suited to our method, which does not rely on
stationary kernel assumptions.

\paragraph{Pest Control ($d=25$, categorical).}
Fig.~\ref{fig:pde} shows results on the Pest Control
problem~\citep{oh}, a combinatorial optimization task with 25
categorical input variables each taking 5 values.
We apply one-hot encoding following~\citet{pmlr-v37-hernandez-lobatob15}, 
and use their wrapper transformation $T(\cdot)$ to project the continuous 
output of $E_{\phi}$ back to a categorical {input} at each oracle 
query.
This is the highest-cardinality benchmark in our evaluation
($5^{25} \approx 3\times10^{17}$ configurations), making exhaustive search
infeasible.
VBO-MI achieves superior performance compared to all baselines, and we observe
that the advantage over the Focal Sparse GP grows as more iterations are
completed, consistent with the observation that our method benefits from
accumulated replay data.

\paragraph{Rover Trajectory ($d=60$).}
In Fig.~\ref{fig:lunar}, we evaluate on the Rover trajectory planning
benchmark~\citep{NEURIPS2019_6c990b7a}, which optimises the positions of 30
waypoints in a 2D environment (60 total dimensions).
This is the highest-dimensional task in our evaluation and directly tests
scalability.
We observe that VBO-MI maintains competitive performance relative to the
strongest baselines at this dimensionality.
The Focal Sparse GP and IBNN also perform well at this scale, consistent with
findings in the high-dimensional BO literature~\citep{li2024a}, but VBO-MI
achieves comparable or superior final reward values while requiring
significantly fewer real function evaluations per iteration due to its
surrogate-assisted inner loops.
\subsection{Ablation}
\label{sec:ablation}
We present ablation studies to isolate the contribution of
the surrogate network $\hat{f}_\psi$ and each term of
the variational acquisition function $L_{E_\phi}$.
All ablations are evaluated on Ackley ($d=10$) and Rover trajectory
planning ($d=60$), and all other hyperparameters are held fixed.
\paragraph {Surrogate network vs Direct queries}
To validate the role of the surrogate network $\hat{f}_\psi$ in
VBO-MI, we compare the complete implementation of the surrogate-assisted implementation against
an ablated baseline in which the surrogate is removed and the helper
and action-net updates return to querying the true objective $f$
directly inside the loops $K_a$ and $K_b$.
All other components network architectures, learning rates, batch
size, and number of iterations are held fixed, so the surrogate
is the sole variable.

Figs.~\ref{fig:ab1} and~\ref{fig:ab2} report best-so-far reward
averaged over multiple random seeds (shaded regions: $\pm 1$ standard
deviation) on Ackley ($d=10$) and Rover trajectory planning
($d=60$), respectively.

On the lower-dimensional Ackley task (Fig.~\ref{fig:ab1}), both
variants converge toward the global optimum, and the surrogate-assisted
method achieves a higher final reward.
The improvement is already visible in the early iterations, suggesting
that the surrogate provides a useful exploitation gradient from the
outset once it has been pre-trained on the warm-up observations.

The advantage of the surrogate becomes substantially more pronounced
on the higher-dimensional Rover task (Fig.~\ref{fig:ab2}).
Here, the original variant without surrogate converges more slowly
and plateaus at a noticeably lower reward value, whereas the
surrogate-assisted method continues to improve throughout the budget.

This behaviour is consistent with the role of the surrogate as a
query-efficient approximation to $\mathbb{E}[y_t \mid \mathcal{H}]$ as
in higher-dimensional spaces the action-net requires more gradient
steps to identify promising regions, and the surrogate makes those
steps possible without incurring additional real function evaluations.

Our ablation provides us with the conclusion that the surrogate provides a
meaningful benefit as dimensionality increases, justifying its
inclusion as a core component of VBO-MI rather than an optional
efficiency add-on.

From an oracle-efficiency perspective, the surrogate-assisted
variant reduces real $f$ queries per iteration from
$(K_a + K_b + 1) \times B$ to $B$ which is a $16\times$ reduction with
our default hyperparameters while achieving superior
optimization performance across both tasks when compared on the same budget.
\begin{figure*}[t]
\centering
\begin{subfigure}{0.48\textwidth}
    \centering
    \includegraphics[width=\linewidth]{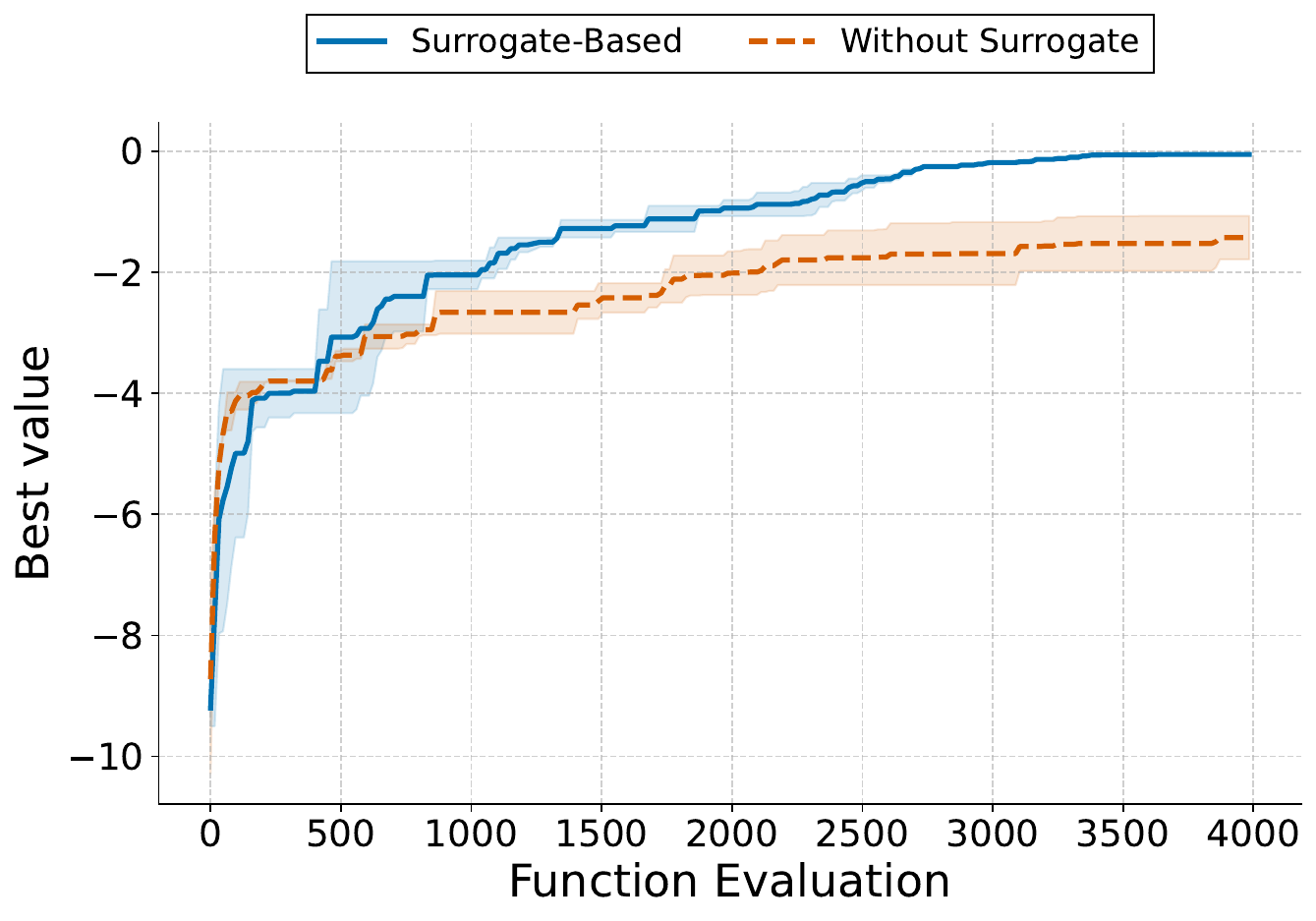}
    \caption{Ackley ($d=10$)}
    \label{fig:ab1}
\end{subfigure}
\hfill
\begin{subfigure}{0.48\textwidth}
    \centering
    \includegraphics[width=\linewidth]{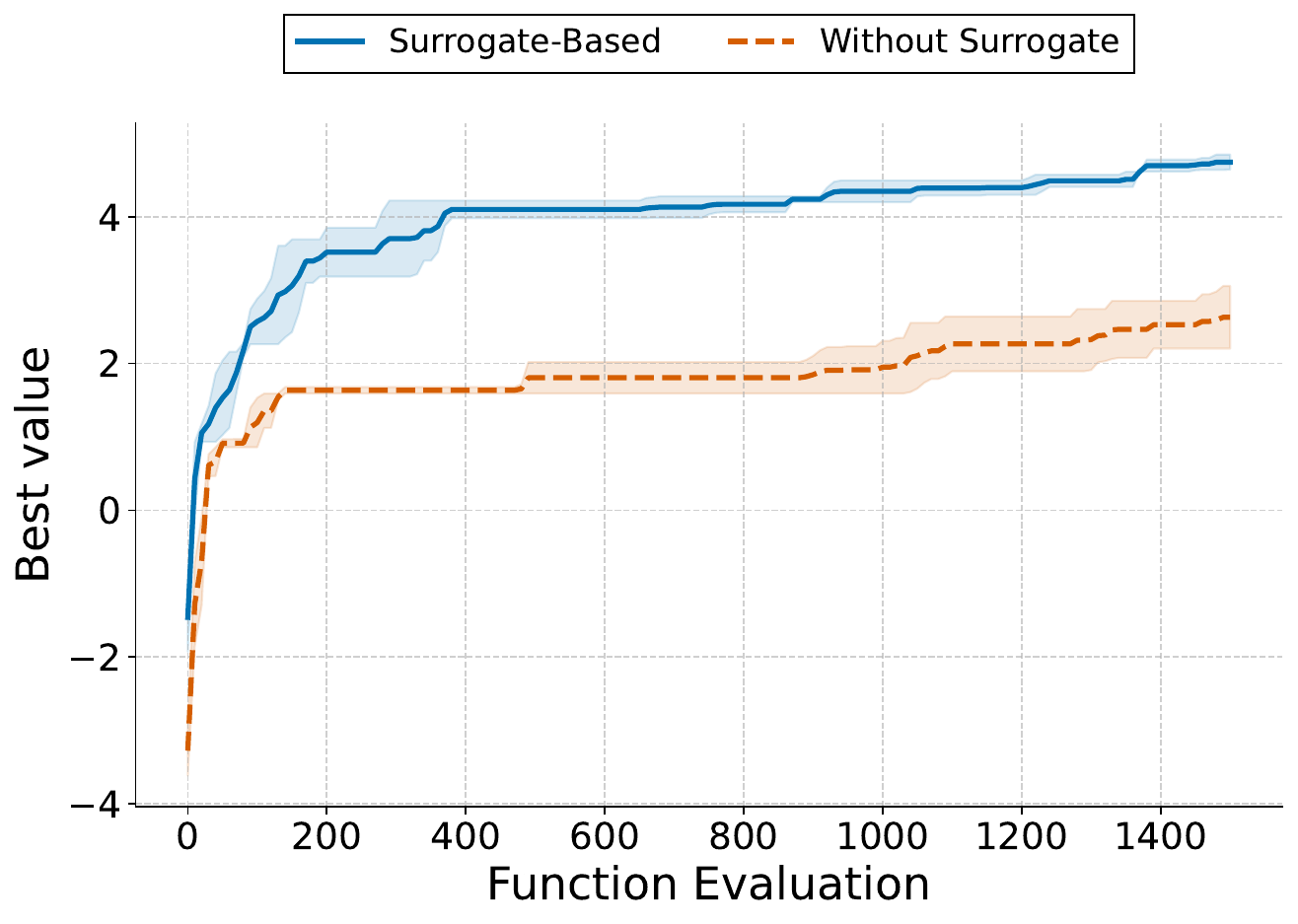}
    \caption{Rover trajectory ($d=60$)}
    \label{fig:ab2}
\end{subfigure}
\caption{Ablation I: surrogate network contribution.
         Surrogate-assisted VBO-MI vs.\ the original variant that
         queries $f$ directly inside the $K_a$ and $K_b$ loops.
         Mean $\pm$ 1 std over three trials.}
\label{fig:ablation-surrogate}
\end{figure*}
 
\begin{figure*}[t]
\centering
\begin{subfigure}{0.48\textwidth}
    \centering
    \includegraphics[width=\linewidth]{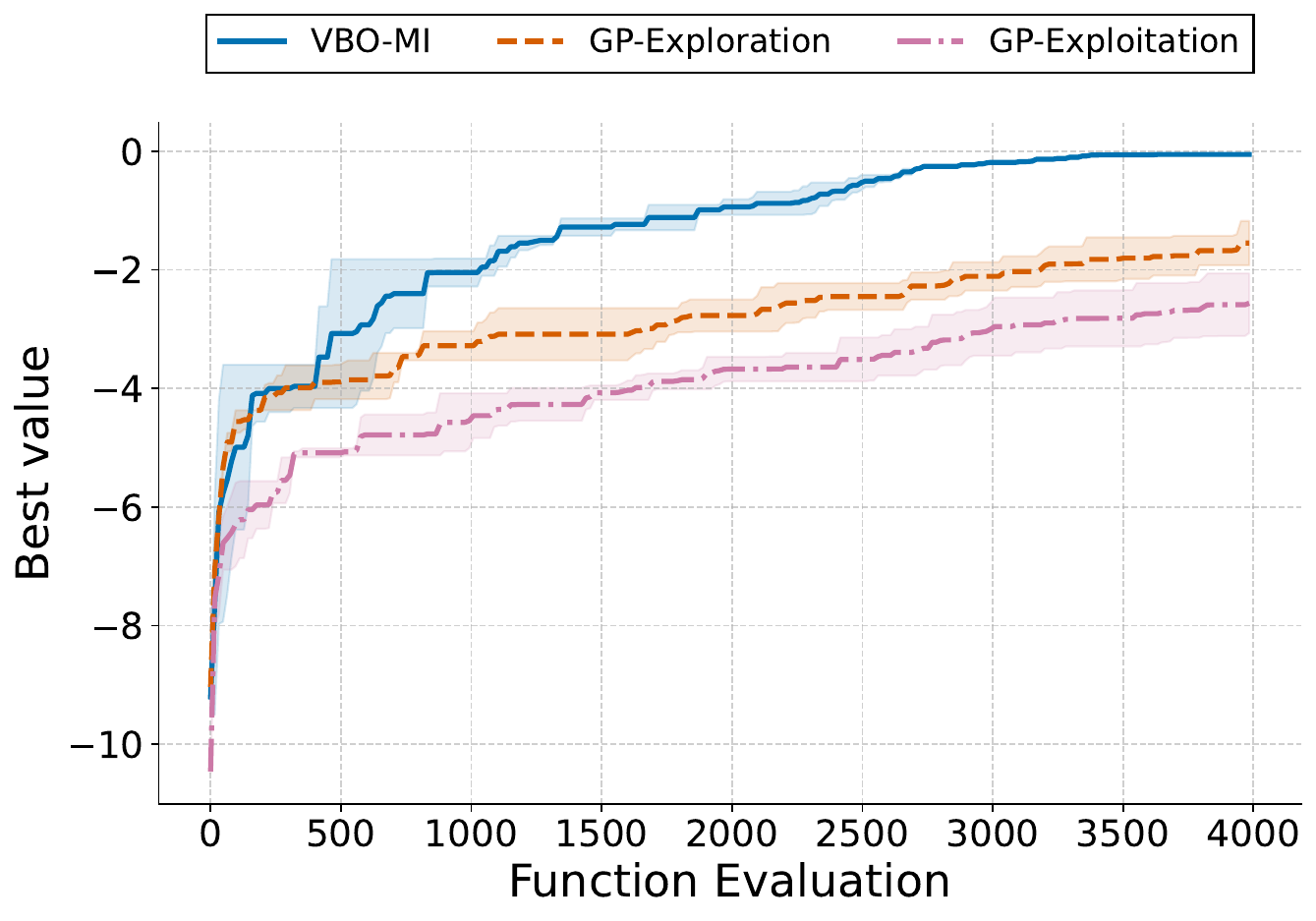}
    \caption{Ackley ($d=10$)}
    \label{fig:ab3}
\end{subfigure}
\hfill
\begin{subfigure}{0.48\textwidth}
    \centering
    \includegraphics[width=\linewidth]{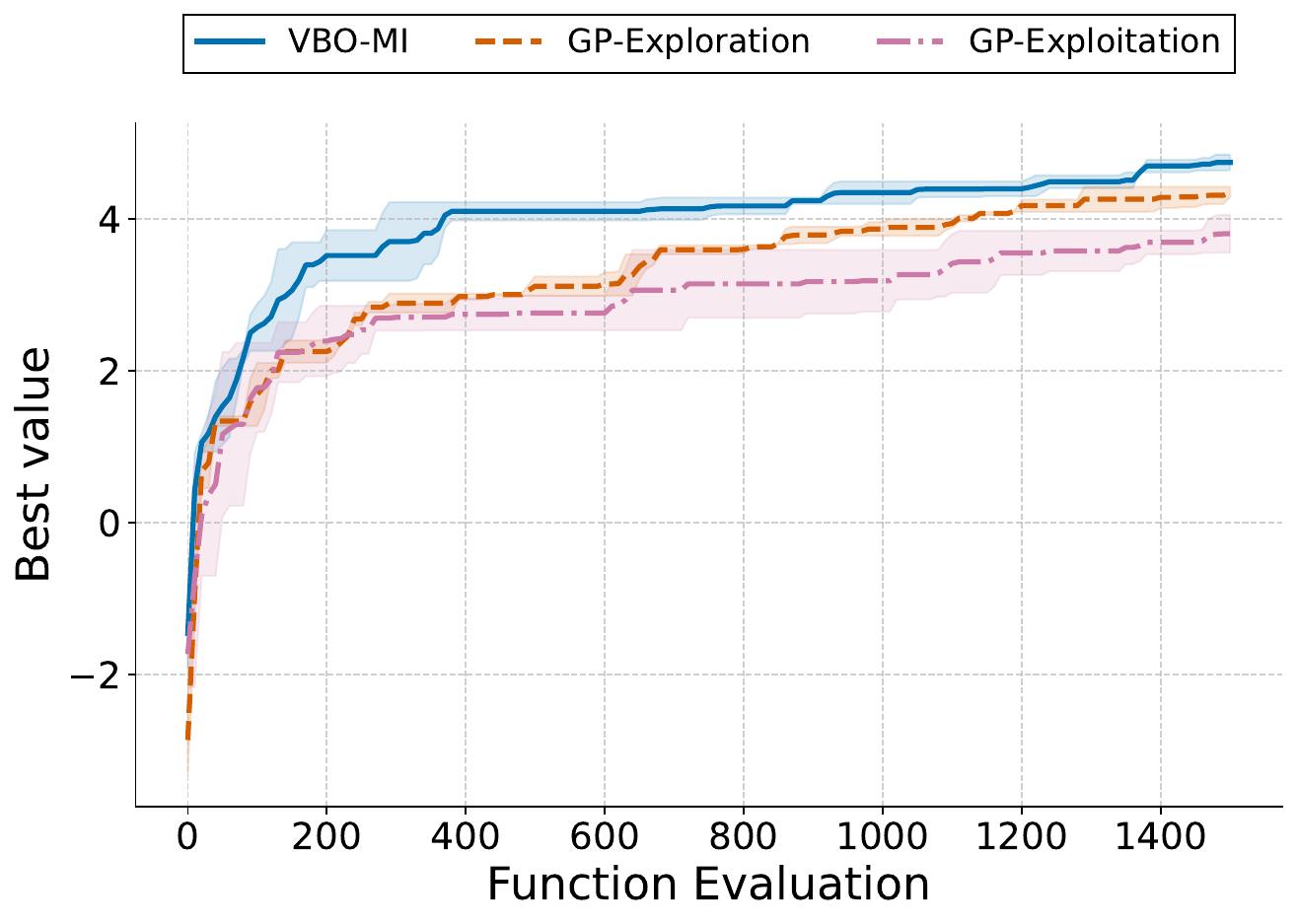}
    \caption{Rover trajectory ($d=60$)}
    \label{fig:ab4}
\end{subfigure}
\caption{Ablation II: acquisition function term contribution.
         Full VBO-MI vs.\ two hybrid variants: one replaces the
         exploration term with the GP posterior variance
         ($\sigma^2_{\mathrm{GP}}$), and one replaces the
         exploitation term with the GP posterior mean
         ($\mu_{\mathrm{GP}}$).
         Mean $\pm$ 1 std over three trials.}
\label{fig:ablation-acquisition}
\end{figure*}
\paragraph{Acquisition function terms.}
We assess the individual contributions of the exploration and
exploitation terms in $L_{E_\phi}$~\Eqref{vbomi-loss} by
constructing two hybrid variants.
In the first, the exploration term
$\sqrt{\beta}\,I_\theta(\mathbf{x}, \hat{y})$ is replaced by the
posterior variance of a GP surrogate,
$\sqrt{\beta}\,\sigma^2_{\mathrm{GP}}(\mathbf{x})$, while the
exploitation term $\hat{f}_\psi(\mathbf{x})$ is kept intact.
In the second, the exploitation term is replaced by the GP posterior
mean $\mu_{\mathrm{GP}}(\mathbf{x})$, while the exploration term
$I_\theta$ is kept intact.
Results are shown in Fig.~\ref{fig:ablation-acquisition}.
 
On both tasks (Figs.~\ref{fig:ab3}--\ref{fig:ab4}), the full VBO-MI
outperforms both hybrid variants, confirming that both terms
contribute to the performance of the algorithm.

Replacing the exploitation term with the GP posterior mean produces a
larger degradation than replacing the exploration term with the GP
variance, particularly on the higher-dimensional Rover task.

This indicates that the exploitation gradient provided by
$\hat{f}_\psi$ plays a more critical role in steering the action-net
toward high-reward regions than the exploration signal alone.
Nonetheless, removing the variational MI exploration term also
degrades performance, demonstrating that the information-theoretic
exploration in VBO-MI is not redundant and cannot be replaced by
standard GP-based uncertainty without loss.

Taken together, both terms are necessary for the full performance of
VBO-MI, and neither can be substituted by its GP-based counterpart
without incurring a measurable cost.
\paragraph{Additional ablations.}
Appendix~\ref{app:additional} reports three further sensitivity
studies.
The acquisition term ablation of Fig.~\ref{fig:benchmarks-ex-comp}
is extended to four additional real-world tasks (PDE optimization,
Interferometer, Lunar Lander, Pest Control), confirming that the
finding --- exploitation removal hurts more than exploration removal
--- holds consistently across task types and dimensionalities.
Fig.~\ref{fig:benchmarks-batch} shows that VBO-MI's performance
improves with batch size up to a threshold, beyond which gains
diminish, a pattern shared by all baselines.
Fig.~\ref{fig:benchmarks-beta} shows that VBO-MI is less sensitive
to the exploration parameter $\beta$ than GP-based baselines,
maintaining competitive final reward across $\beta \in [10, 200]$.

\section{Conclusion}\label{sec:conclusion}
We presented VBO-MI, a fully gradient-based black-box optimization
framework that eliminates GP priors and specific variational family
assumptions.
By combining a variational MI estimator for exploration with a
data-driven surrogate for exploitation, all three networks in VBO-MI
are trained via backpropagation while treating $f$ as a strict black
box.
The surrogate reduces oracle queries to one batch per iteration
(a $16\times$ reduction), and the amortised action-net replaces the
inner-loop acquisition optimizer (a $10^2\times$ FLOP reduction).
\textcolor{black}{We established consistency guarantees and demonstrated competitive or
superior performance on diverse benchmarks spanning $d=10$ to $d=60$.
Future work includes extending VBO-MI to multi-objective settings,
investigating tighter MI estimators for high-dimensional reward landscapes, and
establishing finite-time regret bounds for VBO-MI in place of the
asymptotic and existence results presented here.}

\section*{Impact Statement}

This work proposes a general-purpose framework for \textcolor{black}{black-box }optimization based on variational mutual information estimation\nocite{fmnf0,fmnf1}, and a neural actor-critic structure. As a methodological contribution, the proposed approach is intended to improve optimization performance across a wide range of tasks and does not target any specific application domain. \textcolor{black}{Black-box }optimization techniques are commonly used in areas such as engineering design, scientific experimentation, and machine learning model tuning. Such methods may be applied in diverse fields including biology, chemistry, or finance, and as a result there are many potential societal consequences of our work, none
which we feel must be specifically highlighted here.

\bibliography{main2026}
\bibliographystyle{tmlr}
\appendix
\section{Proofs of Theorem~\ref{thm:consistency} and Proposition~\ref{prop:existence}}

\label{app:consistency}

{\color{black} We prove the results of Theorem~\ref{thm:consistency} and proposition 
in \textcolor{black}{Section} \ref{theoreticalresults}. The proof proceeds via following lemmas and setup.

\paragraph{Setup.}
At each BO iteration $t$, the action-net $E_\phi$ maps $B$ i.i.d.\ 
seeds $\{s^{(i)}_t\}_{i=1}^B \sim \mathcal{N}(0,I)$ to candidate 
actions $\{\mathbf{x}^{(i)}_t\}_{i=1}^B$. Since the seeds are i.i.d.\ and 
$E_\phi$ is a deterministic map, the actions $\{\mathbf{x}^{(i)}_t\}_{i=1}^B$ 
are i.i.d.\ within the batch. The surrogate predictions 
$\hat{y}^{(i)}_t = \hat{f}_\psi(\mathbf{x}^{(i)}_t)$ are deterministic 
functions of these i.i.d.\ inputs and therefore also i.i.d.\ within 
the batch.
\begin{remark}
   Under Assumption (iii), $\Theta$ and $\Psi$ are compact, and $\hat f_\psi, D_\theta$ are jointly continuous in their parameters and inputs. Hence $\mathcal{Y} := \{\hat f_\psi(x):\psi\in\Psi, x\in\mathcal{D}\}$ is compact, and $D_\theta$ is bounded on the compact set $\mathcal{D}\times\mathcal{Y}$: there exists $C<\infty$ with $|D_\theta(x,\hat y)|\le C$ for all $\theta\in\Theta$, $x\in\mathcal{D}$, $\hat y\in\mathcal{Y}$. Consequently $\mathbb{E}[D_\theta(x_t,\hat y_t)]$ and $\mathbb{E}[e^{D_\theta(x_t,\hat y_t)}]$ are (uniformly) finite and the latter is bounded away from zero, for every $\theta\in\Theta$. This is what we invoke in the proof of Lemma A.1 below; it is a consequence of Assumption (iii).
\end{remark}
\begin{lemma}[Within-batch convergence — Conclusion (a)]
\label{lem:within-batch}
By the boundedness remark above (as a consequence of Assumption (iii)) and the i.i.d.\ 
within-batch structure above, the sample-average approximation 
$\widehat{g}(\phi,\theta)$ converges to $g(\phi,\theta)$ 
$\mathbb{P}$-almost surely as $B\to\infty$.
\end{lemma}

\begin{proof}
The loss $L_{E_\phi}$ in~\Eqref{vbomi-loss} contains 
three terms. We handle each separately.

\textbf{Term 1 (exploitation):}
Since $\{\hat{y}^{(i)}_t\}_{i=1}^B$ are i.i.d.\ with 
$\mathbb{E}[\hat{y}_t]<\infty$, the strong law of large numbers 
(SLLN) gives:
\begin{equation}
    \frac{1}{B}\sum_{i=1}^B \hat{y}^{(i)}_t 
    \;\xrightarrow{B\to\infty}\; 
    \mathbb{E}_{P_{\mathbf{x}_{t}}}[\hat{f}_\psi(\mathbf{x}_{t})]
    \quad \mathbb{P}\text{-a.s.}
\end{equation}

\textbf{Term 2 (DV joint expectation):}
Since $\{(\mathbf{x}^{(i)}_t, \hat{y}^{(i)}_t)\}_{i=1}^B$ are i.i.d.\ and 
$\mathbb{E}[D_\theta(\mathbf{x}_{t},\hat{y}_t)]<\infty$ by Assumption~(i), 
SLLN gives:
\begin{equation}
    \frac{1}{B}\sum_{i=1}^B D_\theta(\mathbf{x}^{(i)}_t, \hat{y}^{(i)}_t) 
    \;\xrightarrow{B\to\infty}\; 
    \mathbb{E}_{P_{\mathbf{x}_{t},\hat{y}_t}}\!\bigl[D_\theta(\mathbf{x}_{t},\hat{y}_t)\bigr]
    \quad \mathbb{P}\text{-a.s.}
\end{equation}

\textbf{Term 3 (DV marginal log-expectation):}
For the cross-sum, note that for $i\neq j$ the pairs 
$(\mathbf{x}^{(i)}_t, \hat{y}^{(j)}_t)$ are independent draws from 
$P_{\mathbf{x}_{t}}\times P_{\hat{y}_t}$, since $\mathbf{x}^{(i)}$ and $\hat{y}^{(j)}$ 
are functions of distinct i.i.d.\ seeds. The double average is 
therefore a second-order U-statistic:
\begin{equation}
    \frac{1}{B^2}\sum_{i=1}^B\sum_{j=1}^B 
    e^{D_\theta(\mathbf{x}^{(i)}_t,\,\hat{y}^{(j)}_t)}
    \;\xrightarrow{B\to\infty}\;
    \mathbb{E}_{P_{\mathbf{x}_{t}}\times P_{\hat{y}_t}}\!\bigl[
    e^{D_\theta(\mathbf{x}_{t},\hat{y}_t)}\bigr]
    \quad \mathbb{P}\text{-a.s.},
\end{equation}
by the law of large numbers for U-statistics 
\citep[Theorem~5.4]{Serfling1980}. 
Since $\log(\cdot)$ is continuous and the argument is bounded away 
from zero by Assumption~(i), the continuous mapping theorem gives:
\begin{equation}
    \log\!\left(\frac{1}{B^2}\sum_{i,j} 
    e^{D_\theta(\mathbf{x}^{(i)}_t,\hat{y}^{(j)}_t)}\right)
    \;\xrightarrow{B\to\infty}\;
    \log\mathbb{E}_{P_{\mathbf{x}_{t}}\times P_{\hat{y}_t}}\!\bigl[
    e^{D_\theta(\mathbf{x}_{t},\hat{y}_t)}\bigr]
    \quad \mathbb{P}\text{-a.s.}
\end{equation}
Combining the three limits gives 
$\widehat{g}(\phi,\theta)\to g(\phi,\theta)$ $\mathbb{P}$-a.s.
The argument for $L_{D_\theta}$ is identical since it contains the 
same two DV terms.
\end{proof}
Next we state a theorem and a definition which will be used in our proof for conclusion (b) of the theorem.
\begin{theo}[Uniform risk bound for adapted sequences; \citealp{lauer2023uniform}]
\label{fact:lauer}
Let $\mathbf{Z}_n = (\mathbf{x}_i, y_i)_{i=1}^n$ be adapted to a filtration $(\mathcal{F}_i)$, with no further assumption on stationarity or mixing. Let $\mathcal{L}$ be a class of bounded losses, $\ell : \mathcal{V}^2 \to [0, C]$. Then for any $\delta \in (0,1)$, with probability at least $1 - \delta$,
\begin{equation}
\forall \psi \in \Psi, \quad |L_{n}(\psi) - \hat{L}_{n}(\psi)| \leq 2\overline{R}_n(\mathcal{L}) + C\sqrt{\frac{\log(2/\delta)}{2n}},
\end{equation}
where
\begin{align}
L_n(\psi) &:= \frac{1}{n}\sum_{i=1}^n \mathbb{E}[\ell_\psi(\mathbf{x}_i, y_i) \mid \mathcal{F}_{i-1}], \\
\hat{L}_n(\psi) &:= \frac{1}{n}\sum_{i=1}^{n} \ell_\psi(\mathbf{x}_i, y_i),
\end{align}
and $\overline{R}_n(\mathcal{L})$ is any bound on the sequential Rademacher complexity of $\mathcal{L}$ expressed in terms of the marginal distributions of $(\mathbf{x}_i, y_i)$.
\end{theo}
\begin{proof}
    See. \citealp{lauer2023uniform}.
\end{proof}
\begin{definition}[Empirical pseudo-metric and uniform covering number]
For a sequence $\mathbf{x}_n = (\mathbf{x}_i)_{1\le i\le n} \in \mathcal{D}^n$, define
the empirical $L^2$ pseudo-metric on functions $\mathcal{D}\to\mathbb{R}$ by
\begin{equation}
    d_{2,\mathbf{x}_n}(g,g') \;=\;
    \left(\frac{1}{n}\sum_{i=1}^n |g(\mathbf{x}_i)-g'(\mathbf{x}_i)|^2\right)^{1/2}.
\end{equation}
For a function class $\mathcal{F}$ and scale $\epsilon > 0$, the covering
number $\mathcal{N}(\epsilon,\mathcal{F},d_{2,\mathbf{x}_n})$ is the
smallest cardinality of a set $\mathcal{H}\subseteq\mathcal{F}$ such that
every $g\in\mathcal{F}$ satisfies $d_{2,\mathbf{x}_n}(g,h)<\epsilon$ for
some $h\in\mathcal{H}$. 
\end{definition}

\begin{lemma}[Surrogate convergence — Conclusion (b)]
\label{lem:surrogate-v2}
Under Assumptions of the Theorem 4.2,
as $n = |\mathcal{H}| \to \infty$:
\begin{equation}
    \sup_{\mathbf{x}\in\mathcal{D}}\,\bigl|\hat{f}_{\psi^*}(\mathbf{x}) - f(\mathbf{x})\bigr|
    \;\xrightarrow{\mathbb{P}\text{-a.s.}}\; 0.
\end{equation}
\end{lemma}

\begin{proof}
Let $\mathcal{L}_\Psi = \{\ell_\psi(x,y) = (\hat f_\psi(x)-y)^2 : \psi \in \Psi\}$\textcolor{black}{, and, with} $\mathcal{D}$, $\Psi$ compact and $f$,
$\hat f_\psi$, $\epsilon_t$ bounded\footnote{This is mild condition, and satisfied by truncating any
sub-Gaussian noise model at a high-probability threshold;without affecting the results.}, $\mathcal{L}_\Psi$ is bounded by some
$C < \infty$. Also assume that the true objective \(f\) is representable by the fixed surrogate class. Moreover, let $\mu_n := \frac{1}{n}\sum_{i=1}^n \delta_{\mathbf{x}_i}$ denote the
empirical occupation measure of the inputs stored in $\mathcal{H}$ (where
$n = |\mathcal{H}|$); by Assumption~(ii), $\mu_n \to \mu_\infty$ weakly,
$\mathbb{P}$-almost surely, for some limiting measure $\mu_\infty$ with
full support on $\mathcal{D}$.

\textbf{Step 1 (Uniform concentration).}
Since $\Psi$ is compact and the spectral norm of each weight matrix is a
continuous function of its entries, the architecture of Assumption~(iii)
has spectral norms $\|W_i\|_\sigma$ bounded uniformly over $\Psi$ at every
layer $i = 1,\ldots,L$. \textcolor{black} {Since $\Psi$ is compact and, for each layer $i=1,\ldots,L$, the maps $\psi\mapsto\|A_i(\psi)\|_\sigma$ and $\psi\mapsto\|A_i(\psi)^\top\|_{2,1}$ are continuous (Assumption (iii)), both attain finite suprema over $\Psi$:
 $$s_i := \sup_{\psi\in\Psi}\|A_i(\psi)\|_\sigma < \infty, \qquad b_i := \sup_{\psi\in\Psi}\|A_i(\psi)^\top\|_{2,1} < \infty, \qquad i=1,\ldots,L.$$}

This is precisely the hypothesis required by the
spectrally-normalized covering number bound of
\citet{bartlett2017spectrally}: \textcolor{black}{for the induced function class $\mathcal F_\Psi=\{\hat f_\psi:\psi\in\Psi\}$ (ReLU activations, so $\rho_i=1$, since $\rho_{i}$ denotes the Lipschitz constant of the network's $i$-th layer activation; reference matrices $M_i=0$ following \citep{bartlett2017spectrally}), and $n$ points $\mathbf{x}_1,\ldots,\mathbf{x}_n$ with maximum layer width $W$,
\begin{equation}\label{eq39}\log\mathcal N(\tau,\mathcal F_\Psi,\|\cdot\|_2) \;\le\; \frac{\|X\|_2^2\ln(2W^2)}{\tau^2}\Big(\prod_{j=1}^L s_j^2\Big)\Big(\sum_{i=1}^L (b_i/s_i)^{2/3}\Big)^3, \qquad \|X\|_2^2 := \sum_{i=1}^n\|\mathbf{x}_i\|^2, \end{equation}}
\textcolor{black}{where $\|\cdot\|_2$ denotes the unnormalized Euclidean metric on network evaluations at the $n$ points, as in Bartlett et al. (2017). Since $d_{2,x^n}(g,g') = n^{-1/2}\|g-g'\|_2$ (Definition A.3), a $\|\cdot\|_2$-cover at resolution $\tau$ is a $d_{2,x^n}$-cover at resolution $\tau/\sqrt n$. Setting $\tau=\epsilon\sqrt n$ in \Eqref{eq39} gives
}
\begin{equation}\label{meq}
   \log \mathcal N (\epsilon, \mathcal F_\Psi, d_{2,x^n}) \le \frac{A}{n}\sum_{i=1}^n \|\mathbf{x}_i\|^2 \Big/ \epsilon^2, \qquad A := \ln(2W^2)\Big(\prod_{j=1}^L s_j^2\Big)\Big(\sum_{i=1}^L (b_i/s_i)^{2/3}\Big)^3.
\end{equation}
\textcolor{black}{Here $A$ depends only on $L$, the layer widths $W$, and the network's spectral complexity (Bartlett et al., 2017, Eq. 1.2) — combining both the spectral-norm bounds $s_i$ and the $(2,1)$-norm bounds $b_i$ — and is independent of $n$ and of the joint distribution of $(\mathbf{x}_i)_{i\le n}$. Moreover, since $\mathcal D$ is compact (Assumption (iii)), $n^{-1}\sum_{i=1}^n\|\mathbf{x}_i\|^2 \le \sup_{\mathbf{x}\in\mathcal D}\|\mathbf{x}\|^2 =: c < \infty$ uniformly in $n$, so the right-hand side of \eqref{meq} is $O(1/\epsilon^2)$, uniformly in $n$.} 

\textcolor{black}{By the contraction principle applied to the Lipschitz loss $u\mapsto u^2$ on a bounded domain, $\hat R(\mathcal L_\Psi) \le L_\ell \hat R(\mathcal F_\Psi)$ for a constant $L_\ell$. We bound $\hat R(\mathcal F_\Psi)$ via chaining \citep[Thm.7]{lauer2023uniform}: for any integer $N>0$, }
\textcolor{black}{\begin{equation}\label{chaining}
\hat R_{\mathbf{x}^n}(\mathcal F_\Psi) \le \frac{D_\mathcal F}{2^N} + 6D_\mathcal F\sum_{j=1}^N 2^{-j}\sqrt{\frac{\log\mathcal N(D_\mathcal F 2^{-j},\mathcal F_\Psi,d_{2,\mathbf{x}^n})}{n}},\end{equation}
where $D_\mathcal F := \sup_{\psi,\psi'\in\Psi,\,x\in\mathcal D}|\hat f_\psi(\mathbf{x})-\hat f_{\psi'}(\mathbf{x})|$ is the diameter of $\mathcal F_\Psi$ under $d_{2,\mathbf{x}^n}$, finite since $\Psi\times\mathcal D$ is compact and $\hat f_\psi(\mathbf{x})$ is jointly continuous.}

\textcolor{black}{Substituting the covering-number bound (\Eqref{meq}) at scale $\epsilon = D_\mathcal F 2^{-j}$ into the $j$-th term of the sum:}

 \textcolor{black}{\begin{equation}6D_\mathcal F\cdot 2^{-j}\sqrt{\frac1n\cdot\frac{\tilde A\cdot 4^j}{D_\mathcal F^2}} = 6\sqrt{\frac{\tilde A}{n}}, \qquad \tilde A := A\cdot\sup_{x\in\mathcal D}\|\mathbf{x}\|^2, \end{equation}}
 \textcolor{black}{a quantity independent of $j$: the $2^{-j}$ prefactor cancels exactly against the $4^j$ growth inside the square root arising from the $1/\epsilon^2$-scaling of \Eqref{meq}. The sum in  therefore reduces to $N$ identical terms:}
\begin{equation}\hat R_{\mathbf{x}^n}(\mathcal F_\Psi) \le \frac{D_\mathcal F}{2^N} + 6N\sqrt{\frac{\tilde A}{n}}. \end{equation}
\textcolor{black} {Choosing $N=\lceil\log_2 n\rceil$ makes the first term $D_\mathcal F/2^N \le D_\mathcal F/n = O(1/n)$, while the second becomes $6\lceil\log_2 n\rceil\sqrt{\tilde A/n} = \tilde O(\sqrt{\tilde A/n})$; since $1/n = o(\sqrt{\log n/n})$, the first term is asymptotically dominated by the second, giving}

\textcolor{black}{$$\hat R_{\mathbf{x}^n}(\mathcal F_\Psi) = \tilde O\!\left(\sqrt{\tilde A/n}\right).$$}

\textcolor{black}{Taking expectations, $R_n(\mathcal F_\Psi) = \tilde O(\sqrt{\tilde A/n})$, and by the contraction bound above,
 $$R_n(\mathcal L_\Psi) \le L_\ell R_n(\mathcal F_\Psi) = \tilde O(\sqrt{A/n})$$
 (up to log factors), independent of the joint law of $(\mathbf{x}_i)_{i\le n}$; the $n$-independent constants $L_\ell$ and $\sup_{x\in\mathcal D}\|x\|^2$ are absorbed into the implicit constant.}

Applying Theorem~\ref{fact:lauer} with $\delta = \delta_n := 1/n^2$, with
probability $\ge 1 - 1/n^2$,
\begin{equation}
    \forall \psi \in \Psi, \quad
    \big|L_n(\psi) - \hat L_n(\psi)\big| \le \eta_n
    := 2\,\overline{\mathcal{R}}_n(\mathcal{L}_\Psi) + \textcolor{black}{C}\sqrt{\log(2n^2)/2n}
    \;\xrightarrow{n\to\infty}\; 0.
\end{equation}

\textbf{Step 2 (Argmin transfer).}
Let $\hat\psi_T := \arg\min_{\psi\in\Psi} \hat L_n(\psi)$ (Assumption~ii)
and $\psi_n^\circ := \arg\min_{\psi\in\Psi} L_n(\psi)$. On the event of
Step~1,
\begin{equation}
    L_n(\hat\psi_T) \le \hat L_n(\hat\psi_T) + \eta_n
    \le \hat L_n(\psi_n^\circ) + \eta_n
    \le L_n(\psi_n^\circ) + 2\eta_n.
\end{equation}

\textbf{Step 3 (Borel–Cantelli).}
Since $\sum_n \delta_n = \sum_n 1/n^2 < \infty$, by Borel–Cantelli the event
of Step~1 fails only finitely often, $\mathbb{P}$-a.s. Hence
Hence $L_n(\hat{\psi}_n) - L_n(\psi_n^o) \to 0$, $\mathbb{P}$-a.s.

\textbf{Step 4 (Coverage and approximation adequacy).}
Under the MDS assumption (i) and the boundedness of $f$ on the compact domain $\mathcal{D}$, the conditional risk simplifies as
\begin{equation}
\mathbb{E}[\ell_\psi(\mathbf{x}_i, y_i) \mid \mathcal{F}_{i-1}]
= \mathbb{E}\big[(\hat{f}_\psi(\mathbf{x}_i) - y_i)^2 \mid \mathcal{F}_{i-1}\big]
= (\hat{f}_\psi(\mathbf{x}_i) - f(\mathbf{x}_i))^2 + \sigma^2,
\end{equation}
since $y_i = f(\mathbf{x}_i) + \epsilon_i$ with $\mathbb{E}[\epsilon_i \mid \mathcal{F}_{i-1}] = 0$ and $\mathbb{E}[\epsilon_i^2 \mid \mathcal{F}_{i-1}] = \sigma^2$. Hence, with $R_n(\psi) := \frac{1}{n}\sum_{i=1}^n (\hat{f}_\psi(\mathbf{x}_i) - f(\mathbf{x}_i))^2$, we have $L_n(\psi) = R_n(\psi) + \sigma^2$.

By Assumption (ii), the buffer's occupation measure $\mu_n = \frac{1}{n}\sum_{i=1}^n \delta_{\mathbf{x}_i}$ converges weakly to $\mu_\infty$. Since the class $\{(\hat{f}_\psi - f)^2 : \psi \in \Psi\}$ is uniformly bounded and equicontinuous on the compact set $\mathcal{D}$ (Assumption (iii) and continuity of $f$), the weak convergence upgrades to uniform convergence of the integrals:
\begin{equation}
\sup_{\psi \in \Psi} \Big| R_n(\psi) - R_\infty(\psi) \Big| \xrightarrow{n \to \infty} 0, \quad R_\infty(\psi) := \int_{\mathcal{D}} (\hat{f}_\psi(\mathbf{x}) - f(\mathbf{x}))^2 \, d\mu_\infty(\mathbf{x}),
\end{equation}
$\mathbb{P}$-almost surely. Consequently, the empirical minimizer $\psi_n^o := \arg\min_{\psi \in \Psi} L_n(\psi)$ satisfies
\begin{equation}
R_\infty(\psi_n^o) \leq \inf_{\psi \in \Psi} R_\infty(\psi) + o(1),
\end{equation}
so that
\begin{equation}
L_n(\psi_n^o) \xrightarrow{n \to \infty} \inf_{\psi \in \Psi} \mathbb{E}_{\mu_\infty}[(\hat{f}_\psi(\mathbf{x}) - f(\mathbf{x}))^2] + \sigma^2.
\end{equation}
By the Universal Approximation Theorem (Hornik et al., 1989), for any $\eta > 0$ there exists a network of sufficient capacity with $\inf_{\psi \in \Psi} \|\hat{f}_\psi - f\|_\infty \leq \eta$; since a sup-norm bound controls the $L^2(\mu)$-risk under any probability measure $\mu$ on $\mathcal{D}$, this limit is within $O(\eta^2)$ of the noise floor $\sigma^2 = \mathbb{E}[\epsilon^2]$, irrespective of the support of $\mu_\infty$.

\textbf{Step 5 (Risk to sup-norm, via Arzel\`a--Ascoli).}
Combining Steps 3-4, $\hat{f}_{\psi_n} \to f$ in $L^2(\mu_\infty)$ up to $O(\eta)$ along the whole sequence. Since $\eta > 0$ is arbitrary, we may take the network class to be large enough that the approximation error is below any prescribed threshold; hence $\hat{f}_{\psi_n} \to f$ in $L^2(\mu_\infty)$ along the whole sequence. We upgrade this to uniform convergence on $\mathrm{supp}(\mu_\infty)$ by contradiction. Suppose $\sup_{\mathbf{x} \in \mathrm{supp}(\mu_\infty)} |\hat{f}_{\psi_n}(\mathbf{x}) - f(\mathbf{x})| \not\to 0$; then some subsequence $(n_k)$ satisfies $\sup_{\mathbf{x} \in \mathrm{supp}(\mu_\infty)} |\hat{f}_{\psi_{n_k}}(\mathbf{x}) - f(\mathbf{x})| \geq \eta_0 > 0$ for all $k$. The family $\{\hat{f}_\psi\}_{\psi \in \Psi}$ is uniformly bounded and uniformly Lipschitz (same spectral-norm bound as Step 1), hence equicontinuous on $\mathcal{D}$. By Arzela-Ascoli, $\{\hat{f}_{\psi_{n_k}}\}_k$ has a subsequence converging uniformly on $\mathcal{D}$ to some $h \in C(\mathcal{D})$, hence in $L^2(\mu_\infty)$. Since Step 4 gives $\hat{f}_{\psi_n} \to f$ in $L^2(\mu_\infty)$ along the whole sequence, uniqueness of $L^2(\mu_\infty)$-limits gives $h = f$ $\mu_\infty$-a.e. Since $h - f$ is continuous, this directly gives $h = f$ on $\mathrm{supp}(\mu_\infty)$: if $h(\mathbf{x}_0) \neq f(\mathbf{x}_0)$ for some $\mathbf{x}_0 \in \mathrm{supp}(\mu_\infty)$, continuity gives a neighborhood $U \ni \mathbf{x}_0$ on which $h \neq f$, and $\mathbf{x}_0 \in \mathrm{supp}(\mu_\infty)$ forces $\mu_\infty(U) > 0$, contradicting $h = f$ $\mu_\infty$-a.e. This yields uniform convergence on $\mathrm{supp}(\mu_\infty)$ along the subsequence, contradicting $\sup_{\mathbf{x} \in \mathrm{supp}(\mu_\infty)} |\hat{f}_{\psi_{n_k}}(\mathbf{x}) - f(\mathbf{x})| \geq \eta_0$. Hence $\sup_{\mathbf{x} \in \mathrm{supp}(\mu_\infty)} |\hat{f}_{\psi_n}(\mathbf{x}) - f(\mathbf{x})| \to 0$, $\mathbb{P}$-a.s.

\end{proof}

Now we turn to the proposition. We should mention that conclusions (a) and (b) are established in two independent asymptotic regimes (batch size
$B$, training horizon $T$
). \textcolor{black}{Proposition \ref{prop:existence}} below is a separate existence result in a third regime (network capacity $\delta$), and depends only on (a).

\begin{lemma}[Existence of a concentrated policy]
\label{lem:concentrated-policy}
For $\mathbf{x} \in \mathbb{R}^d$ and $r > 0$, let
$\tilde{B}(\mathbf{x}, r) := \{\mathbf{x}' \in \mathbb{R}^d : \|\mathbf{x}' - \mathbf{x}\| \leq r\}$
denote the closed Euclidean ball. Let $\mathcal{S} \subseteq \mathbb{R}^{d_s}$ be a compact seed domain satisfying
$\mathbb{P}_{\mathcal{N}(0,I)}(\mathcal{S}) \geq 1 - \varepsilon_{\mathcal{S}}$
for some $\varepsilon_{\mathcal{S}} > 0$. For any $\delta > 0$, there exists a feedforward network $E_\phi$ such that
\begin{equation}
\sup_{\mathbf{s} \in \mathcal{S}} \|E_\phi(\mathbf{s}) - \mathbf{x}^*\| \leq \delta.
\label{eq:concentrated-policy}
\end{equation}
Let $P_\delta := (E_\phi)_\# P_{\mathcal{S}}$ denote the induced policy --- the law of
$\mathbf{x} = E_\phi(\mathbf{s})$ with $\mathbf{s} \sim P_{\mathcal{S}}$,
where $P_{\mathcal{S}}$ is the restriction of $\mathcal{N}(0,I)$ to $\mathcal{S}$ which is consistent with the actor-critic sampling scheme of Section~\ref{sec4.4.2}. Consequently:
\begin{enumerate}
\item[(a)] for every $\mathbf{s} \in \mathcal{S}$,
$f(E_\phi(\mathbf{s})) \geq f(\mathbf{x}^*) - \omega_f(\delta)$,
where $\omega_f$ is the modulus of continuity of $f$ on $\mathcal{D}$; and
\item[(b)] $\operatorname{supp}(P_\delta) \subseteq \tilde{B}(\mathbf{x}^*, \delta) \cap \mathcal{D}$.
\end{enumerate}
\end{lemma}

\begin{proof}
Let $\mathcal{S} \subseteq \mathbb{R}^{d_s}$ be a compact seed domain chosen such that
$\mathbb{P}_{\mathcal{N}(0,I)}(\mathcal{S}) \geq 1 - \varepsilon_{\mathcal{S}}$;
such a set exists by tightness of the Gaussian measure on $\mathbb{R}^{d_s}$.
Define $\pi^* : \mathcal{S} \to \mathcal{D}$ as the constant map $\pi^*(\mathbf{s}) = \mathbf{x}^*$,
which is trivially continuous. Since $\mathcal{S}$ and $\mathcal{D}$ are compact and $\pi^*$ is continuous,
the universal approximation theorem \citep{Hornik1989MultilayerFN} gives, for any $\delta > 0$,
a feedforward network $E_\phi$ with
$\sup_{\mathbf{s} \in \mathcal{S}} \|E_\phi(\mathbf{s}) - \pi^*(\mathbf{s})\| \leq \delta$.
Claim~(a) follows from uniform continuity of $f$ on the compact domain $\mathcal{D}$.
Claim~(b) follows since $E_\phi(\mathcal{S}) \subseteq \tilde{B}(\mathbf{x}^*, \delta) \cap \mathcal{D}$
(the sup bound, together with $E_\phi$ mapping into $\mathcal{D}$ by construction),
and this set is closed, hence contains the support of the pushforward of any measure on $\mathcal{S}$ through $E_\phi$.
\end{proof}

\begin{corollary}[Existence of an $\epsilon$-optimal solution]
\label{cor:eps-optimal}
Under the assumptions of Lemma~\ref{lem:concentrated-policy},
for any $\epsilon > 0$ there exist networks $E_\phi$, $\hat{f}_\psi$, $D_\theta$
--- realizable within the hypothesis classes $\Phi$, $\Theta$, $\Psi$ ---
such that for every $\mathbf{s} \in \mathcal{S}$,
the action $\mathbf{x} = E_\phi(\mathbf{s})$ satisfies
\begin{equation}
f(\mathbf{x}) \geq f(\mathbf{x}^*) - \epsilon.
\label{eq:eps-optimal}
\end{equation}
\end{corollary}

\begin{proof}
Choose $\delta > 0$ such that $\omega_f(\delta) \leq \epsilon$;
this is possible since $\omega_f(\delta) \to 0$ as $\delta \to 0$
by uniform continuity of $f$ on the compact domain $\mathcal{D}$.
Let $E_\phi$ be the network from Lemma~\ref{lem:concentrated-policy}.
Then for every $\mathbf{s} \in \mathcal{S}$,
\begin{equation}
f(E_\phi(\mathbf{s})) \geq f(\mathbf{x}^*) - \omega_f(\delta) \geq f(\mathbf{x}^*) - \epsilon,
\end{equation}
which is~\eqref{eq:eps-optimal}.
The surrogate $\hat{f}_\psi$ and helper $D_\theta$ can be chosen arbitrarily
from their respective hypothesis classes (e.g., with sufficient capacity
to approximate $f$ and the mutual information, respectively,
by the universal approximation theorem).
\end{proof}

\begin{remark}[Scope of the proof]
Last lemma is an existence result: it establishes 
that a network $E_{\hat\phi}$ achieving $\varepsilon$-optimality 
exists within the function class. It does not claim that the 
alternating gradient procedure in Algorithm~\ref{alg:vbo} 
necessarily finds this network within a finite number of steps. 
Establishing finite-step convergence rates for the joint 
bi-level optimization in~\Eqref{mnet} under adaptive sampling is a non-trivial open 
problem that we leave for future work. We want to emphasize that this open problem also exist in other information theoretic acquisition function literature such as \citep{pmlr-v267-ishikura25a}.
\end{remark}}

\begin{figure*}
\centering
    \includegraphics[width=.65\textwidth]{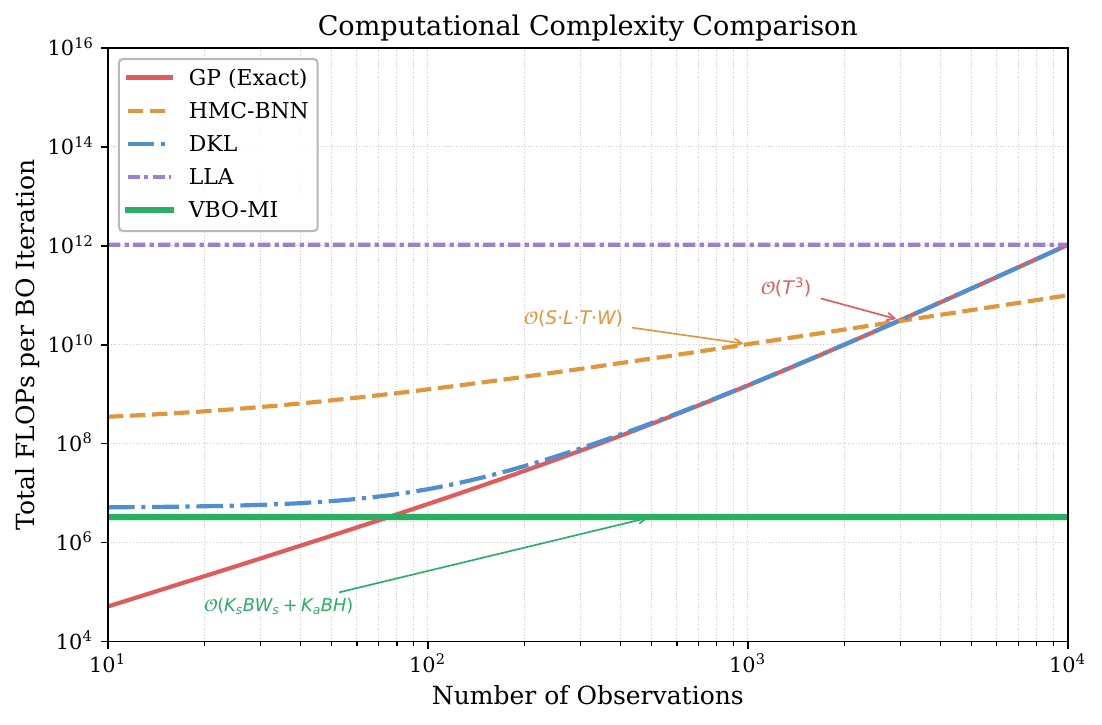}
    \caption{Theoretical computational complexity comparison measured in Floating-Point Operations (FLOPs) per iteration}
    \label{complexity-com}
\end{figure*}
\section{GP-Based Bayesian Optimization}
\label{GBO}
In this section we discuss some of the detailed results regarding the GP-based Bayesian optimization  discussed in Section 3. we summarize some of the main assumptions commonly used in the Gaussian literature~\citep{srinivas2010gaussian}~\citep{pmlr-v108-kirschner20a}~\citep{pmlr-v162-tay22a} as below:

\noindent\textbf{Assumption 1. \underline{Prior assumption on \( f \):}} The function \( f: \mathcal{D} \to \mathbb{R} \) is assumed to be a sample drawn from a Gaussian Process (GP), denoted as
\begin{equation}
    f \sim \mathrm{GP}(\mu(\mathbf{x}), k(\mathbf{x}, \mathbf{x}')),
\end{equation}
where the mean function is \(\mu(\mathbf{x}) = \mathbb{E}[f(\mathbf{x})]\) and the covariance function is defined by
\begin{equation}
    k(\mathbf{x}, \mathbf{x}') = \mathbb{E}\big[(f(\mathbf{x}) - \mathbb{E}[f(\mathbf{x})])(f(\mathbf{x}') - \mathbb{E}[f(\mathbf{x}')])\big].
\end{equation}
It is typically assumed that, without loss of generality, \(\mu(\mathbf{x}) = 0\) and the variance of \(f\) prior is bounded, i.e., \(k(\mathbf{x}, \mathbf{x}) \leq 1\) for all \(\mathbf{x} \in \mathcal{D}\). In other words, for any finite collection of inputs \(\{\mathbf{x}_t\}\), the corresponding (prior) function values \(\{f(\mathbf{x}_t)\}\) follow a joint Gaussian distribution determined by the covariance function.

\vspace{0.5em}
\noindent\textbf{Assumption 2. \underline{Posterior assumption on \( f \):}} After observing noisy samples, the posterior distribution of \( f \) conditioned on past inputs and observations remains a GP. Let $
A_T = [\mathbf{x}_1, \dots, \mathbf{x}_T]$ denote the set of previously sampled points, where each $\mathbf{x}_{t}\in\mathbb{R}^{d}$, and
$\mathbf{Y}_{T}=[y_{1},\dots, y_{T}]^{T}$ be the vector of noisy observations, where $y_{t} = f(\mathbf{x}_{t}) + \epsilon_{t},\,\, \epsilon_{t} \sim \mathcal{N}(0, \sigma^{2})$ with \( y_{t} \in \mathbb{R} \) scalar observations. It follows that the posterior distribution of \( f \) is given by~\citep{10.7551/mitpress/3206.001.0001, srinivas2010gaussian}:
\begin{equation}
    f \mid A_T, \mathbf{Y}_T \sim \mathrm{GP}(\mu_T(\mathbf{x}), k_T(\mathbf{x}, \mathbf{x}')),
\end{equation}
where
\begin{align}
    \mu_T(\mathbf{x}) &= \mathbf{k}_T(\mathbf{x})^T (\mathbf{K}_T + \sigma^2 \mathbf{I})^{-1} \mathbf{Y}_T, \\
    k_T(\mathbf{x}, \mathbf{x}') &= k(\mathbf{x}, \mathbf{x}') - \mathbf{k}_T(\mathbf{x})^T (\mathbf{K}_T + \sigma^2 \mathbf{I})^{-1} \mathbf{k}_T(\mathbf{x}').
\end{align}
Here, $\mathbf{k}_T(\mathbf{x}) = [k(\mathbf{x}_1, \mathbf{x}), \dots, k(\mathbf{x}_T, \mathbf{x})]^T$ is the covariance vector between the new point \(\mathbf{x}\) and the previously sampled points \(A_T\), and $
\mathbf{K}_T = [k(\mathbf{x}_i, \mathbf{x}_j)]_{i,j=1}^T$ is the kernel matrix. The posterior variance at \(\mathbf{x}\) is given by \(\sigma_T^2(\mathbf{x}) = k_T(\mathbf{x}, \mathbf{x})\).
\section{Experimental Setups}
\label{app:setup}
 
To ensure a fair comparison with previously reported results, we
adopt the same benchmark problems and hyperparameter configurations
as described in~\citet{li2024a}, applied to the experimental scenarios
described below.
Reported results are averaged over multiple random seeds to account
for stochastic variability.
 \subsection{Implementation Details}
\subsubsection{VBO-MI Implementation Details}
\label{app:vbomi-details}
 
\paragraph{Network architectures.}
The \emph{action-net} $E_\phi$ consists of three fully connected
layers of width $128$ with \texttt{ReLU} activations and a \texttt{Tanh}
output layer that bounds actions to corresponding input domain based on the problem.
The \emph{helper network} $D_\theta$ is a pairwise LSTM-MLP critic:
two single-layer LSTMs (each
with hidden size $64$) whose final hidden states are concatenated and
passed through a three-layer MLP with
\texttt{ReLU} activations.
The \emph{surrogate network} $\hat{f}_\psi$ is a two-layer MLP
 with \texttt{ReLU} activations,
mapping proposed actions to predicted reward values.
 
\paragraph{Optimisation hyperparameters.}
All three networks are optimised with Adam.
The learning rate is $2 \times 10^{-3}$ for $E_\phi$ and $D_\theta$,
and $1 \times 10^{-3}$ for $\hat{f}_\psi$.
The batch size is $B = 30$ or $B=60$ for real-world and synthetic tasks.
The warm-up phase runs for $W = 30$ iterations, followed by a
200-step surrogate pre-training phase on the warm-up observations.
In the main BO loop, the update counts are
$K_a = 5$ (helper), $K_s = 10$ (surrogate), and $K_b = 10$
(action-net) per iteration.
{\color{black}
\subsubsection{Baseline Implementation Details}
\label{app:baseline-details}

We compare against the following baseline surrogate models, all evaluated
within the same BO loop described in
Section~\ref{sec:eval}: Hamiltonian Monte Carlo (HMC)
\citep{pmlr-v139-izmailov21a,li2024a}, Stochastic Gradient Hamiltonian
Monte Carlo (SGHMC) \citep{10.5555/3044805.3045080}, the Linearized
Laplace Approximation (LLA) \citep{NEURIPS2022_12143893}, Deep
Kernel Learning (DKL) \citep{pmlr-v51-wilson16,li2024a}, Infinite-width
Bayesian Neural Networks (I-BNN) \citep{lee2018deep,li2024a}, standard
Gaussian Processes (GP) \citep{li2024a}, and the
Focal Sparse Gaussian Process (Sparse) \citep{NEURIPS2024_da30215e}.

For HMC, SGHMC, LLA, DKL, and I-BNN, we follow the surrogate
architectures and inference procedures reported in \citet{li2024a};
the Focal Sparse GP baseline follows the configuration reported in
\citet{NEURIPS2024_da30215e}, which is not covered by \citet{li2024a}.
Across all baselines\footnote{We use analytical gradients for optimizing the EI as in Focal Sparse Gaussian process baseline.}, we use Monte-Carlo Expected
Improvement (qEI) \citep{botorch} as the acquisition function,
and initialize each trial with a Sobol-sampled design, consistent with
the protocol of \citet{li2024a}. Moreover, we use 10 initial points for synthetic task, and 50 initial points for Lunar lander, and Rover trajectory, and 20 for pest control tasks. Table~\ref{tab:baseline-hyperparams}
summarizes the architecture, inference procedure, and hyperparameters
used for each baseline.
\begin{table}[H]
\centering
\small
\caption{\textcolor{black}{Baseline architecture and inference hyperparameters.}}
\label{tab:baseline-hyperparams}
\begin{tabular}{|l|p{9cm}|}
\hline
\textcolor{black}{Baseline} & \textcolor{black}{Configuration} \\
\hline
\textcolor{black}{GP }    & \textcolor{black}{Kernel: Matérn-5/2; length-scale prior: Gamma(3,6); output-scale prior: Gamma(2.0, 0.15).} \\
\hline
\textcolor{black} {DKL}    & \textcolor{black}{Feature extractor: 3-layer MLP, Tanh, $w{=}128$; base kernel: Matérn-5/2. }\\
\hline
\textcolor{black}{I-BNN}  & \textcolor{black}{Kernel:} \textcolor{black}{GP with NNGP kernel\cite{lee2018deep} (depth 3, ReLU); weight variance: 10.0; bias variance: 1.6.} \\
\hline
\textcolor{black}{HMC }   &\textcolor{black}{ Architecture: 3-layer MLP, Tanh, $w{=}128$;. }\\
\hline
\textcolor{black}{SGHMC } & \textcolor{black}{ Architecture: 3-layer MLP, Tanh, $w{=}128$; step size: adaptive; likelihood variance: heteroskedastic (network output). }\\
\hline
\textcolor{black}{LLA}    & \textcolor{black}{Architecture: 3-layer MLP, Tanh, $w{=}128$;. }\\
\hline
\textcolor{black}{Sparse} &\textcolor{black}{ Kernel: Matérn-5/2, inducing variables $m=50,200$ for synthetic and real-world tasks. }\\
\hline
\end{tabular}
\end{table}}

\begin{table}[t]
\centering
\begin{threeparttable}
\caption{\textcolor{black}{VBO-MI per-task experimental configuration. Warm-up queries and
total queries are both oracle evaluations of $f$; training mini-batch size
equals the oracle batch size $B$ throughout.}}
\label{tab:per-task-config}
\begin{tabular}{lccccccc}
\toprule
\textcolor{black}{Task} & \textcolor{black}{$d$} & \textcolor{black}{$B$} & \textcolor{black}{$W$} & \textcolor{black}{Warm-up queries} & \textcolor{black}{Main iters.\ $T$} & \textcolor{black}{Total queries} \\
\midrule
\textcolor{black}{Ackley} & \textcolor{black}{10} & \textcolor{black}{60} & \textcolor{black}{30} & \textcolor{black}{1{,}800} & \textcolor{black}{70} & \textcolor{black}{6{,}000} \\
\textcolor{black}{Levy} & \textcolor{black}{20} & \textcolor{black}{60} & \textcolor{black}{30} & \textcolor{black}{1{,}800} & \textcolor{black}{70} & \textcolor{black}{6{,}000} \\
\textcolor{black}{Griewank} & \textcolor{black}{20} & \textcolor{black}{60} & \textcolor{black}{30} & \textcolor{black}{1{,}800} & \textcolor{black}{70} & \textcolor{black}{6{,}000} \\
\textcolor{black}{Lunar Lander} & \textcolor{black}{12} & \textcolor{black}{30} & \textcolor{black}{30} & \textcolor{black}{900} & \textcolor{black}{70} & \textcolor{black}{3{,}000} \\
\textcolor{black}{Pest Control} & \textcolor{black}{25} & \textcolor{black}{30} & \textcolor{black}{30} & \textcolor{black}{900} & \textcolor{black}{36} & \textcolor{black}{2{,}000}\tnote{\dag} \\
\textcolor{black}{Rover Trajectory} & \textcolor{black}{60} & \textcolor{black}{30} & \textcolor{black}{30} & \textcolor{black}{900} & \textcolor{black}{70} & \textcolor{black}{3{,}000} \\
\bottomrule
\end{tabular}
\begin{tablenotes}
\small
\item[\dag] \textcolor{black}{The true total evaluation count for Pest Control is 1{,}980.
For visual consistency with the other benchmarks, the plotted curve repeats
the last observed value up to 2{,}000 evaluations.}
\end{tablenotes}
\end{threeparttable}
\end{table}
\textcolor{black} {\begin{remark}
     we do not track an explicit online diversity metric during training, instead, we avoid action-net collapse through a hyperparameter selection: by identifying a working range of hyperparameters, we are able to train the framework to convergence without degeneration. We report the corresponding sensitivity analysis over batch size and the batch-size and $\beta$  in Appendix E (Figs. 9–10), which characterizes performance within this working regime
 \end{remark}}
\subsection{Synthetic Benchmarks}
\label{app:synthetic}
 
We evaluate on three scalable multimodal functions.
All three have their global minimum at zero and are framed as
maximisation problems by optimising $-f$.
 
\noindent\textbf{Ackley} ($d = 10$, domain $[-32.768, 32.768]^{10}$)
\textemdash\  A deceptive function with a near-flat outer region and
a sharp global basin at the origin.
Standard parameters: $a = 20$, $b = 0.2$, $c = 2\pi$.
 
\noindent\textbf{Levy} ($d = 20$, domain $[-10, 10]^{20}$)
\textemdash\  A scalable multimodal function whose difficulty grows
with dimension due to its sinusoidal structure; global minimum is
$f(\mathbf{1}) = 0$.
 
\noindent\textbf{Griewank} ($d = 20$, domain $[-600, 600]^{20}$)
\textemdash\  A non-separable function whose product term introduces
interactions between dimensions, defined as
$f(\mathbf{x}) = \sum_i x_i^2/4000 - \prod_i \cos(x_i/\sqrt{i}) + 1$,
with global minimum $f(\mathbf{0}) = 0$.
 
\subsection{Real-World Benchmarks}
\label{app:realworld}
 
\noindent\textbf{Lunar Lander} ($d = 12$) \textemdash\ The objective
is to learn an optimal control policy for the OpenAI Gym
\texttt{LunarLander-v2} environment.
The reward is computed over full episodes and the goal is to maximise
the expected return averaged over 50 randomly generated environments.
The controller is parameterised as a feed-forward neural network with
two hidden layers of size $256$ using \texttt{tanh} activations.
A learning rate of $3 \times 10^{-4}$ is used with the Adam
optimiser.
 
\noindent\textbf{Pest Control} ($d = 25$, categorical) \textemdash\
This benchmark optimises a pest-management policy over multiple
seasons, modelled as a combinatorial problem with 25 categorical
variables each taking 5 values~\citet{NEURIPS2019_6c990b7a}.
We apply one-hot vector
encoding~\citep{GARRIDOMERCHAN202020} to handle categorical inputs.
The policy is parameterised by a two-layer neural network with $128$
neurons per layer and \texttt{tanh} activations, optimized with Adam
at learning rate $2 \times 10^{-3}$.
 
\noindent\textbf{Rover Trajectory Planning} ($d = 60$) \textemdash\
The goal is to find an optimal trajectory defined by 30 waypoints
(2 coordinates each) in a 2D environment with
obstacles~\citep{NEURIPS2019_6c990b7a}.
The reward penalises obstacle collisions and path deviation; higher
is better with no known analytical optimum.
This is the highest-dimensional real-world task in our evaluation
and directly tests the scalability of each method.
 
\noindent For all experiments, BO is performed
with an initial design of $n_{\mathrm{init}}$ points drawn via a
Sobol sequence, followed by sequential acquisition steps until the
total evaluation budget is reached.
The evaluation budget and $n_{\mathrm{init}}$ match those of the
baseline study of~\citet{li2024a} for direct comparability.

\section{Computational Complexity and FLOP Analysis}
\label{sec:complexity}
 
A significant bottleneck in BO is the
computational overhead associated with surrogate model updates,
the inner-loop optimization of the acquisition function, and ---
most critically in expensive black-box settings --- the number of
real function evaluations required per iteration. \textcolor{black}{In this section we evaluate complexity of our proposed method against baselines in terms of Flop count computational complexity.}
We compare against HMC-BNN, Deep Kernel Learning (DKL), and
Linearized Laplace Approximation (LLA).
 
\paragraph{Oracle query complexity.}
In standard BNN-based BO, the surrogate is re-fitted and the
acquisition function is re-optimised at every iteration, but the
number of real queries to $f$ per iteration remains $B$ (one batch
of candidate points).  VBO-MI matches this budget: the surrogate
$\hat{f}_\psi$, the helper $D_\theta$, and the action-net $E_\phi$
are all updated using either the surrogate's own predictions or
mini-batches from the replay buffer $\mathcal{H}$, so no additional
oracle calls are incurred inside the $K_a$, $K_s$, or $K_b$ loops.
The \emph{only} real query to $f$ per iteration is the evaluation
batch of $B$ points collected at the end of each loop, which
simultaneously calibrates the surrogate and tracks the
best-so-far value.  The per-iteration oracle cost is therefore
$\mathcal{O}(B)$, matching the standard BO budget.
\textcolor{black} {\begin{remark}
     In Section\ref{sec:ablation} we have provided the comparison of surrogate-based and no-surrogate VBO-MI comparisons in terms of oracle query complexity.
 \end{remark}}
\paragraph{Neural network FLOP complexity.}
Table~\ref{tab:flops} summarises the per-iteration FLOP breakdown
for each method.
 
\textbf{Helper update ($K_a$ steps).}
The helper $D_\theta$ receives surrogate predictions
$\hat{y}_t = \hat{f}_\psi(\mathbf{x}_t)$ in place of real
observations, at a combined cost of
$\mathcal{O}(K_a \cdot B \cdot (W_s + H))$, where $W_s$ is the
surrogate parameter count and $H$ is the LSTM hidden dimension.
 
\textbf{Surrogate update ($K_s$ steps).}
The surrogate is updated for $K_s$ gradient steps on mini-batches
of size $B$ drawn from $\mathcal{H}$, at cost
$\mathcal{O}(K_s \cdot B \cdot W_s)$.
Since $W_s$ is small (a three-layer MLP) and $K_s$ is modest,
this adds a low constant-factor overhead.
 
\textbf{Action-net update ($K_b$ steps).}
Each action-net step includes a frozen surrogate forward pass plus
the action-net backward pass, at cost
$\mathcal{O}(K_b \cdot (B \cdot W_s + W_e))$, where $W_e$ is the
action-net parameter count.
 
The key distinction from GP- and BNN-based methods is in the
surrogate update complexity.
GP-based methods require an $\mathcal{O}(T^3)$ Cholesky
decomposition at each iteration; HMC-BNN requires
$\mathcal{O}(S \cdot L \cdot T \cdot W)$ leapfrog steps over the
full dataset; and DKL combines neural network training with a GP
Cholesky.
By contrast, VBO-MI's internal update costs are independent of $T$:
the replay buffer $\mathcal{H}$ is capped at a fixed size, and all
network updates process fixed-size mini-batches.
The total per-iteration FLOP cost is therefore
$\mathcal{O}(K_s \cdot B \cdot W_s + K_a \cdot B \cdot H
             + K_b \cdot B \cdot W_s)$,
which is \emph{constant in $T$} --- a qualitatively different
scaling regime from all GP- and BNN-based baselines.
Furthermore, by eliminating the inner-loop acquisition optimizer
(replaced by the amortised action-net), the acquisition phase costs
only $\mathcal{O}(K_b \cdot W_e)$ rather than the
$\mathcal{O}(N_\text{starts} \cdot N_\text{steps} \cdot S \cdot W)$
required by BNN-based methods, yielding a near $10^2\times$
reduction in acquisition FLOPs.
 
Fig.~\ref{complexity-com} plots the scaling behaviour (\textcolor{black}{comparison of VBO-MI with and without surrogate}) and
Table~\ref{tab:flops} summarises the asymptotic complexity \textcolor{black}{against the baselines}.
Table~\ref{tab:complexity_params} lists the numerical values used
for the theoretical estimates (\textcolor{black}{ comparison of proposed method against baselines.)}.

{\color{black}
\textbf{Time Complexity of the Proposed Method.}
Table~\ref{tab:wallclock} reports the wall-clock time for our method and all baselines on both benchmarks, measured on an NVIDIA Tesla T4 GPU with an Intel(R) Xeon(R) CPU @ 2.00GHz. As shown in the table, our method's runtime is substantially lower than that of the BNN-based baselines while achieving comparable or better optimization performance.
}

\begin{table}[ht]
\centering
\caption{Comparison of FLOP and oracle-call complexity per BO
  iteration.  $T$: observations, $W$: BNN weights, $W_s$: surrogate
  weights, $H$: LSTM hidden dim, $S$: HMC samples, $L$: leapfrog
  steps, $K_{a,b,s}$: update steps.  We assume $W > H$ and
  $S,L > K_{a,b}$.}
\label{tab:flops}
\vskip 0.15in
\begin{small}
\begin{sc}
\begin{tabular}{lcccc}
\toprule
Method & Surrogate Update & Acquisition Opt. & Real $f$/iter & Scaling \\
\midrule
GP (Exact)
  & $\mathcal{O}(T^3)$
  & $N \cdot \mathcal{O}(T^2)$
  & $B$
  & Cubic \\
HMC-BNN
  & $\mathcal{O}(S \cdot L \cdot T \cdot W)$
  & $N \cdot \mathcal{O}(S \cdot W)$
  & $B$
  & High Linear \\
DKL
  & $\mathcal{O}(T \cdot W + T^3)$
  & $N \cdot \mathcal{O}(T^2 + W)$
  & $B$
  & Cubic \\
LLA
  & $\mathcal{O}(T \cdot W + W^3)$
  & $N \cdot \mathcal{O}(W^2)$
  & $B$
  & Linear \\
\textbf{VBO-MI (ours)}
  & $\mathcal{O}(K_s{\cdot}B{\cdot}W_s + K_a{\cdot}B{\cdot}H)$
  & $\mathcal{O}(K_b{\cdot}(B{\cdot}W_s + W_e))$
  & $\mathbf{B}$
  & \textbf{Constant} \\
\bottomrule
\end{tabular}
\end{sc}
\end{small}
\end{table}

\begin{table}[ht]
\centering
\caption{Parameters used for the FLOP complexity analysis in
  Section~\ref{sec:complexity}.}
\label{tab:complexity_params}
\vskip 0.15in
\begin{small}
\begin{sc}
\begin{tabular}{lcl}
\toprule
Parameter & Value & Description \\
\midrule
$W$          & $10{,}000$ & BNN parameter count (weights + biases) \\
$W_s$        & $\approx\!500$ & Surrogate network parameter count (3-layer MLP) \\
$W_e$        & $\approx\!400$ & Action-net parameter count (3-layer MLP) \\
$D$          & $10$     & Input dimensionality of the objective function \\
$H$          & $64$     & Hidden state dimension of the Helper LSTM \\
$S$          & $50$     & Number of HMC posterior samples \\
$L$          & $20$     & Leapfrog steps per HMC sample \\
$B$          & $256$    & Batch size (oracle calls per eval block) \\
$K_a$        & $5$      & Gradient steps for the Helper network \\
$K_b$        & $10$     & Gradient steps for the Action-net \\
$K_s$        & $10$     & Gradient steps for the Surrogate network \\
$N_{\text{starts}}$ & $10$ & Restarts for acquisition function optimization \\
$N_{\text{steps}}$  & $50$ & Iterations per restart \\
\midrule

\end{tabular}
\end{sc}
\end{small}
\end{table}
{\begin{table}[t]
\color{black}
\centering
\caption{ \textcolor{black}{ Average wall-clock time (per second) for each iteration on for each method on the synthetic (Ackley), and Real-world (Rover Trajectory), benchmarks, \textcolor{black}{with total of 100 iteration.} }}
\label{tab:wallclock}
\begin{tabular}{|l|c|c|}
\hline
Method & Ackley ($d=10$) & Rover Trajectory ($d=60$) \\
\hline
GP              & $32.68\pm 1.22  $        & $109.49\pm 2.72$           \\
\hline
I-BNN           & $20.69\pm 0.92 $ & $0.819\pm 0.08$          \\
\hline
Laplace         & $34.11\pm 1.18$          & $18.85 \pm 0.95$        \\
\hline
HMC             & $202.2\pm 6.21$ & $197.40\pm 5.67$      \\
\hline
SGHMC           & $37.74\pm 2.16 $          & $43.24\pm 2.94 $        \\
\hline
Sparse GP (Focal) & $157.80\pm 4.31$     &    $48.98\pm 1.55$      \\
\hline
DKL             & $1107.63\pm 6.32$  & $185.14\pm 2.13$          \\
\hline
VBO-MI   & $30.54\pm 1.33$      & $33.23\pm 1.26$       \\
\hline
\end{tabular}
\end{table}}

\section{Additional Numerical Evaluations}
\label{app:additional}

In this section, we provide extended numerical results across a broader
set of benchmarks and hyperparameter conditions.

\paragraph{Additional Synthetic Benchmark Results}

\textcolor{black}{To further evaluate the performance of VBO-MI across synthetic objective
functions with different dimensionalities, we consider two additional
benchmarks: the 6-dimensional Hartmann function and the 20-dimensional
Rastrigin function (negative of the main function for maximization). The results are shown in
Fig.~\ref{fig:hartmann} and Fig.~\ref{fig:rastrigin}}.

\textcolor{black}{As shown in Fig.~\ref{fig:hartmann}, VBO-MI achieves competitive/superior
performance on the 6-dimensional Hartmann function. On the 20-dimensional
Rastrigin function (Fig.~\ref{fig:rastrigin}), VBO-MI achieves a larger
performance advantage over the baseline methods. These additional results
provide further empirical evidence that VBO-MI remains effective across
synthetic functions with different structures and dimensionalities.}
\begin{figure*}[t]
\centering
\begin{subfigure}{0.48\textwidth}
    \centering
    \includegraphics[width=\linewidth]{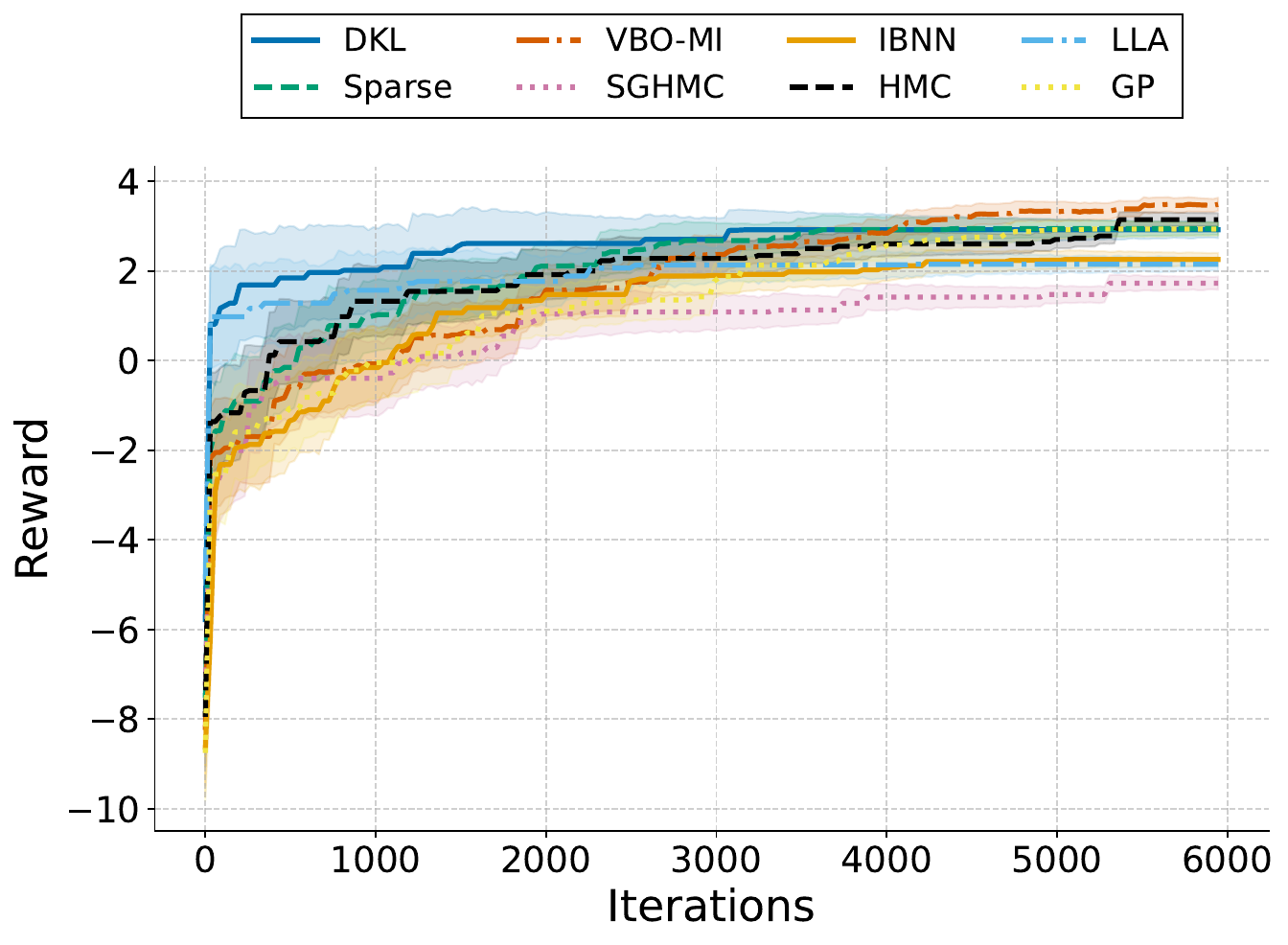}
    \caption{Hartmann ($d=6$)}
    \label{fig:hartmann}
\end{subfigure}
\hfill
\begin{subfigure}{0.48\textwidth}
    \centering
    \includegraphics[width=\linewidth]{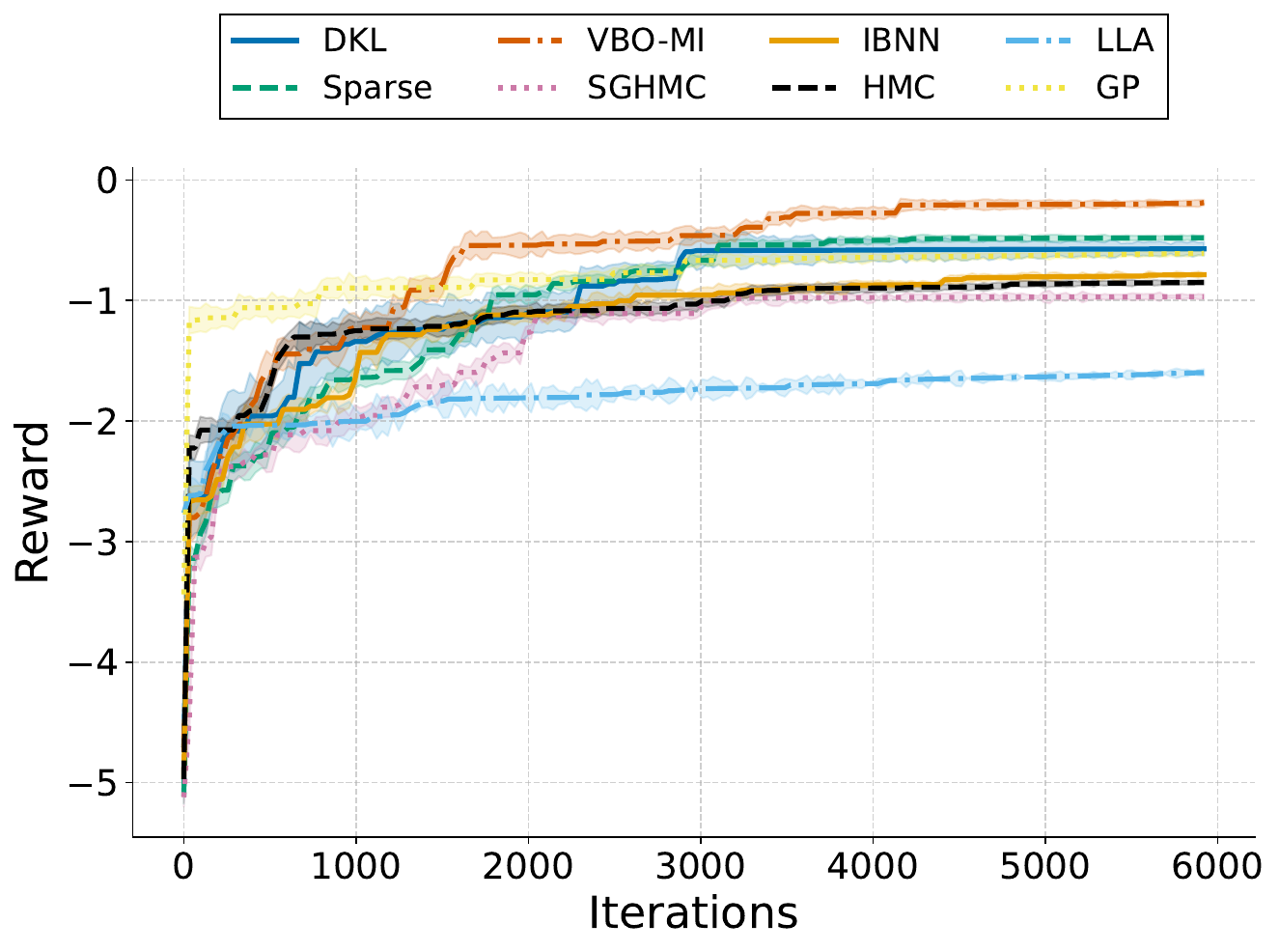}
    \caption{Rastrigin ($d=60$)}
    \label{fig:rastrigin}
\end{subfigure}
\caption{\textcolor{black}{Comparison of VBO-MI with two additional synthetic functions}}
\label{fig:rastrigin0}
\end{figure*}

\paragraph{Batch size sensitivity.}
Fig.~\ref{fig:benchmarks-batch} reports the effect of batch size $B$ on
final reward on PDE optimization, Interferometer, Lunar Lander, and Pest
Control.
VBO-MI generally improves with larger batch sizes up to a threshold,
beyond which gains diminish — a pattern consistent across all baselines.

\paragraph{$\beta$ sensitivity.}
Fig.~\ref{fig:benchmarks-beta} reports the sensitivity of each method
to the exploration parameter $\beta$ on the same four tasks.
VBO-MI is notably less sensitive to the choice of $\beta$ compared to
GP-based baselines, maintaining competitive final reward values across
a wide range of $\beta \in [10, 200]$.

\paragraph{Extended acquisition term ablation.}
Fig.~\ref{fig:benchmarks-ex-comp} extends the acquisition term ablation
of Section~\ref{sec:ablation} to all four lower-dimensional real-world
tasks (PDE, Interferometer, Lunar Lander, Pest Control).
As in the main ablation on Ackley and Rover, replacing either term with
its GP counterpart degrades performance, and replacing the exploitation
term causes greater damage, particularly on higher-dimensional tasks.
This confirms that the finding is not specific to the benchmarks shown
in Section~\ref{sec:ablation} and generalizes across task types.
\paragraph{Noise Effect}
\label{app:noise}
 
Fig.~\ref{fig:noise-robustness} studies how observation noise affects the
surrogate-assisted variant of VBO-MI relative to the original (no-surrogate)
variant, evaluated on Ackley ($d=10$) under three noise levels:
$\sigma \in \{0.4, 0.8, 1.8\}$.
 
Across all noise levels, the surrogate-assisted method (solid lines) achieves
a higher final reward than the original variant (dashed lines).
The performance gap between the two variants is consistent at low and moderate
noise ($\sigma=0.4$ and $\sigma=0.8$), confirming that the surrogate remains a
reliable approximation of $f$ when noise is moderate.
At high noise ($\sigma=1.8$) both variants converge more slowly, as expected,
but the surrogate-assisted method still maintains a clear advantage.
This indicates that the surrogate $\hat{f}_\psi$ is robust to noisy
observations: since it is trained by MSE minimisation over the replay buffer,
it learns a smoothed estimate of $f$ that is less sensitive to individual noisy
evaluations than the direct oracle queries used by the original variant.
 
\begin{figure}[ht]
    \centering
    \includegraphics[width=0.62\columnwidth]{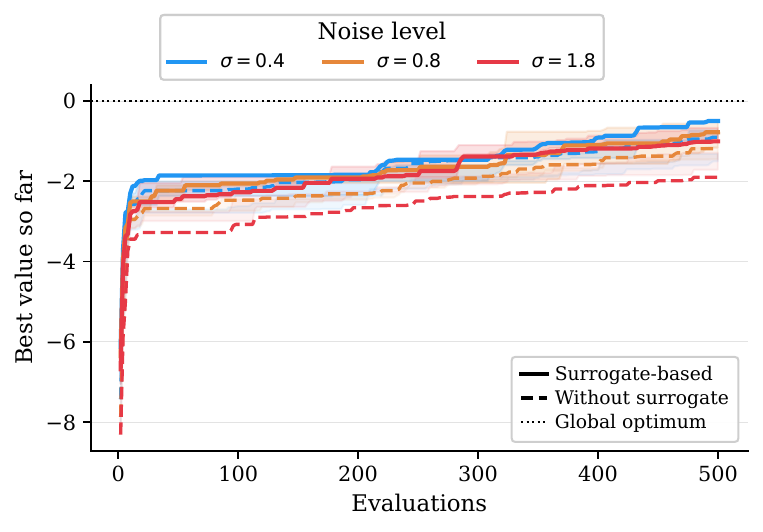}
    \caption{Effect of observation noise on VBO-MI.
             Solid lines: surrogate-assisted variant.
             Dashed lines: original variant (direct oracle queries in $K_a$,
             $K_b$ loops).}
    \label{fig:noise-robustness}
\end{figure}
 \paragraph{Warm-up steps}

Here, we investigate the impact of incorporating a surrogate
model into VBO-MI and analyze the sensitivity of the algorithm to the
number of warm-up steps. Specifically, we compare the performance of
the proposed surrogate-based framework against a variant without
surrogate modeling for different warm-up settings
($20$, $30$, and $60$ warm-up steps).

Fig.~\ref{fig:surrogate_warmup} illustrates the optimization performance
over the course of function evaluations. Overall, the surrogate-based
version consistently achieves superior optimization performance compared
to the non-surrogate variant, demonstrating the effectiveness of the
surrogate model in guiding the search process.

Moreover, increasing the number of warm-up steps from $20$ to $30$
improves the optimization performance, indicating that a moderate number
of initial exploratory samples helps stabilize the Bayesian inference
process and provides a better initialization for subsequent optimization.
However, further increasing the warm-up period to $60$ steps does not
yield noticeable additional improvement, as the performance becomes
similar to the $30$-step setting. This observation suggests that
approximately $30$ warm-up steps provide a suitable trade-off between
initial exploration and optimization efficiency for the considered setup.

\begin{figure}[t]
    \centering
    \includegraphics[width=0.65\linewidth]{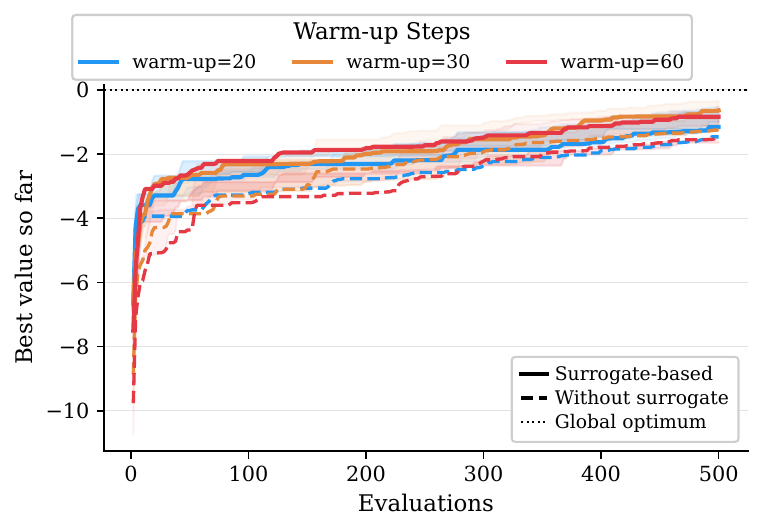}
    \caption{
    Comparison of VBO-MI with and without surrogate modeling for different
    numbers of warm-up steps. The surrogate-based version consistently
    outperforms the non-surrogate counterpart. Increasing the warm-up
    steps improves performance up to approximately $40$ steps, while
    further increases provide marginal gains.
    }
    \label{fig:surrogate_warmup}
\end{figure}

\begin{figure}[ht]
    \centering
    \textbf{Batch size effect on results}\par\medskip
    \begin{subfigure}[t]{0.75\textwidth}
        \centering
        \includegraphics[width=\textwidth]{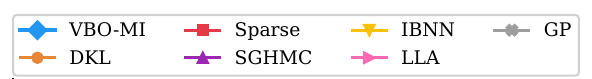}
    \end{subfigure}
    \begin{subfigure}[t]{0.33\textwidth}
        \centering
        \includegraphics[width=\textwidth]{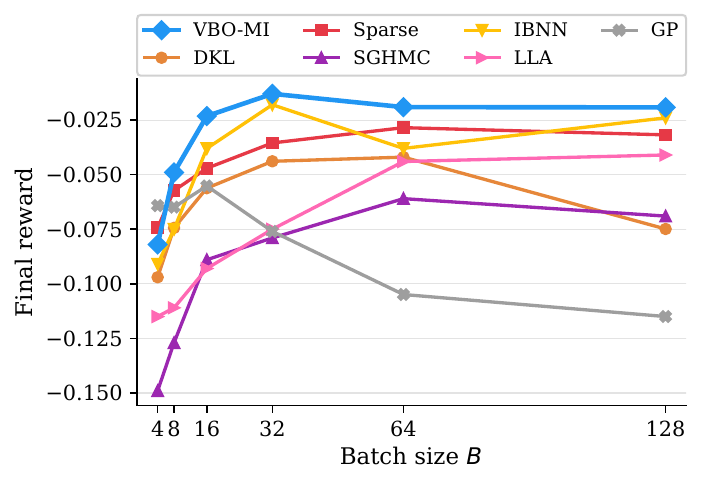}
        \caption{PDE optimization ($d=4$)}
        \label{figpde1-batch}
    \end{subfigure}
    \hfill
    \begin{subfigure}[t]{0.33\textwidth}
        \centering
        \includegraphics[width=\textwidth]{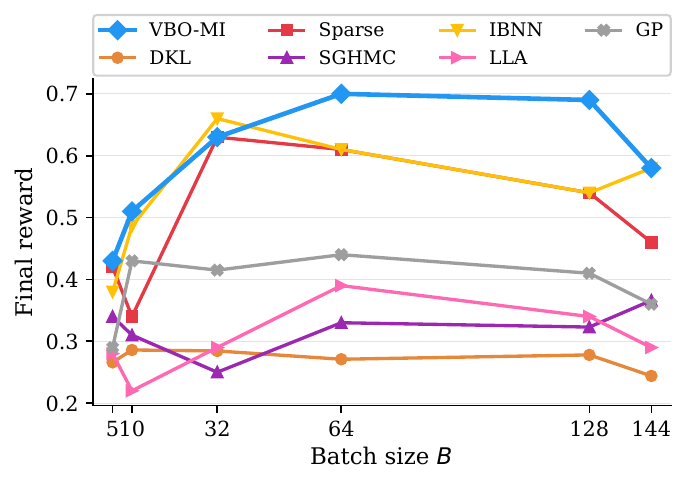}
        \caption{Interferometer ($d=4$)}
        \label{fig-2b-batch}
    \end{subfigure}
    \vspace{0.4cm}
    \begin{subfigure}[t]{0.33\textwidth}
        \centering
        \includegraphics[width=\textwidth]{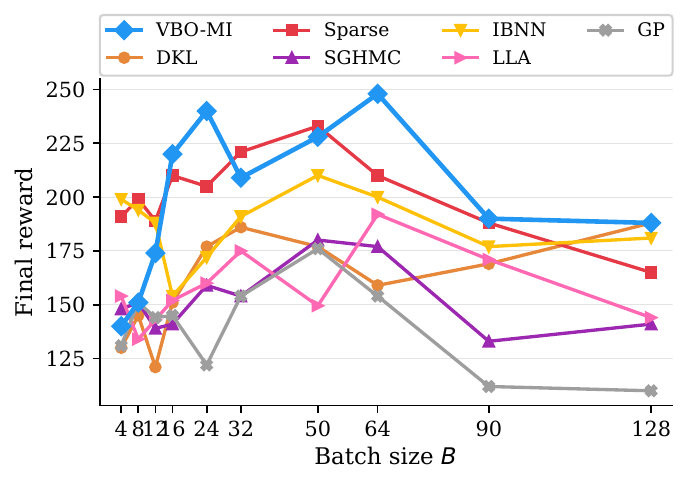}
        \caption{Lunar Lander ($d=12$)}
        \label{fig-2c-batch}
    \end{subfigure}
    \hfill
    \begin{subfigure}[t]{0.33\textwidth}
        \centering
        \includegraphics[width=\textwidth]{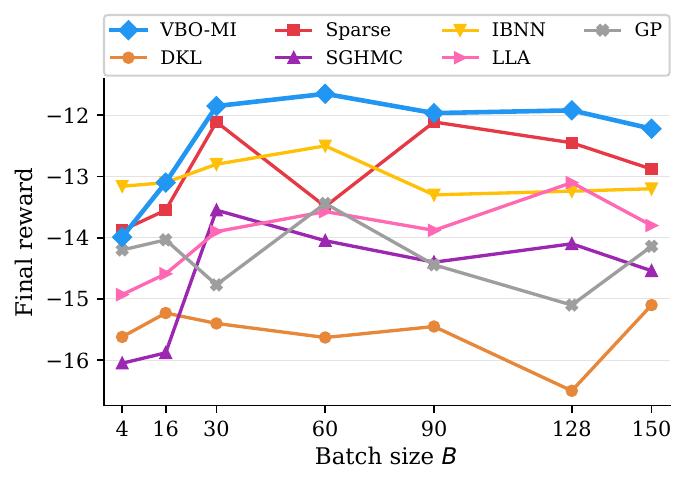}
        \caption{Pest Control ($d=25$)}
        \label{fig-2d-batch}
    \end{subfigure}
    \caption{Batch size sensitivity of VBO-MI and baselines on four
             real-world tasks. Final reward improves with batch size
             up to a threshold, after which gains diminish.}
    \label{fig:benchmarks-batch}
\end{figure}

\begin{figure}[ht]
    \centering
    \textbf{Beta ($\beta$) effect on results}\par\medskip
    \vspace{0.3cm}
    \begin{subfigure}[t]{0.75\textwidth}
        \centering
        \includegraphics[width=\textwidth]{figs/betalunar_3.pdf}
    \end{subfigure}
    \begin{subfigure}[t]{0.33\textwidth}
        \centering
        \includegraphics[width=\textwidth]{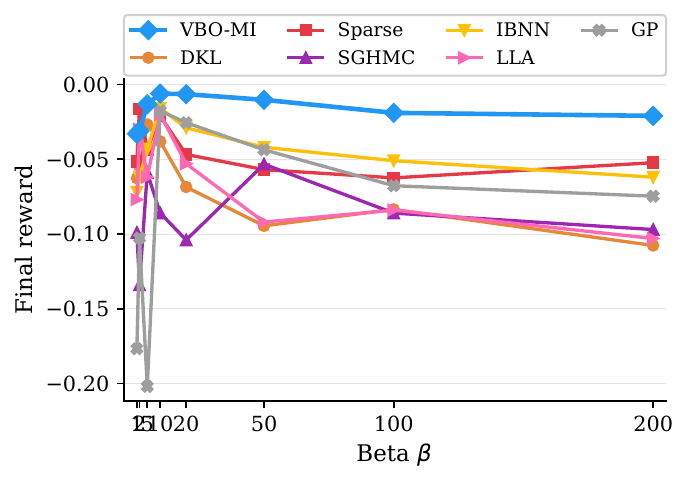}
        \caption{PDE optimization ($d=4$)}
        \label{figpde1-beta}
    \end{subfigure}
    \hfill
    \begin{subfigure}[t]{0.34\textwidth}
        \centering
        \includegraphics[width=\textwidth]{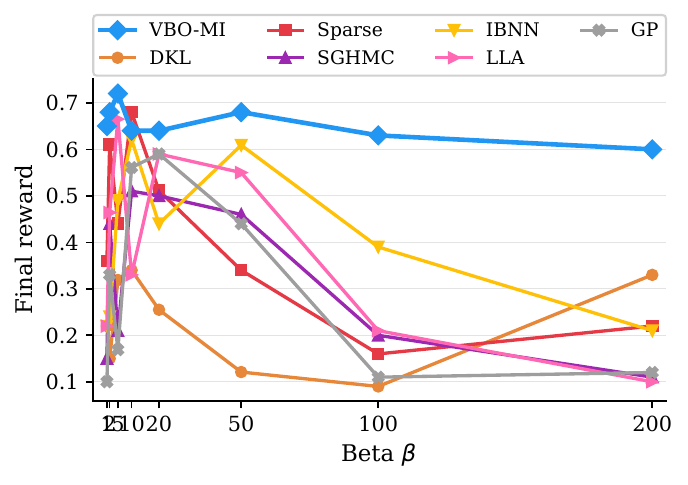}
        \caption{Interferometer ($d=4$)}
        \label{fig-2b-beta}
    \end{subfigure}
    \vspace{0.4cm}
    \begin{subfigure}[t]{0.36\textwidth}
        \centering
        \includegraphics[width=\textwidth]{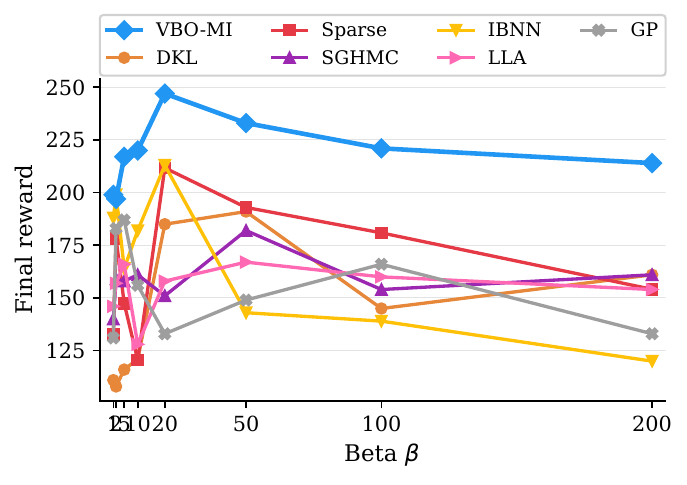}
        \caption{Lunar Lander ($d=12$)}
        \label{fig-2c-beta}
    \end{subfigure}
    \hfill
    \begin{subfigure}[t]{0.33\textwidth}
        \centering
        \includegraphics[width=\textwidth]{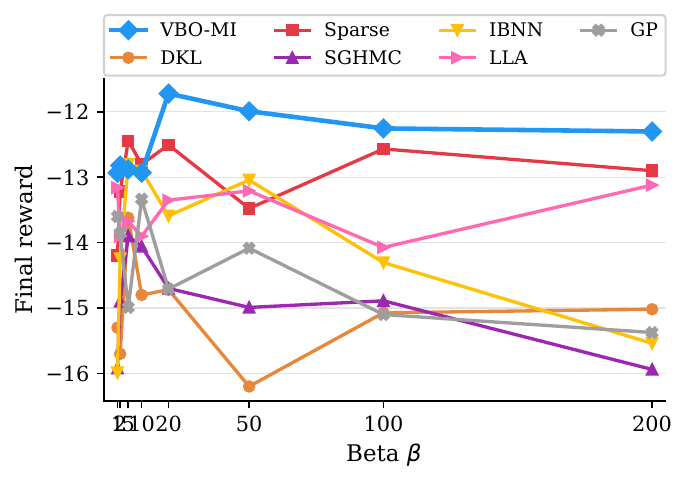}
        \caption{Pest Control ($d=25$)}
        \label{fig-2d-beta}
    \end{subfigure}
    \caption{Sensitivity to exploration parameter $\beta$ across four
             real-world tasks.
             VBO-MI is less sensitive to $\beta$ than GP-based
             baselines, maintaining high reward across $\beta\in[10,200]$.}
    \label{fig:benchmarks-beta}
\end{figure}

\begin{figure*}[ht]
\centering
\begin{subfigure}{0.75\linewidth}
    \centering
    \includegraphics[width=\linewidth]{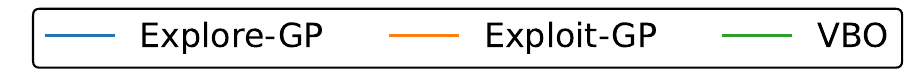}
\end{subfigure}
\vspace{0.2cm}
\begin{subfigure}{0.48\linewidth}
    \centering
    \includegraphics[width=.8\linewidth]{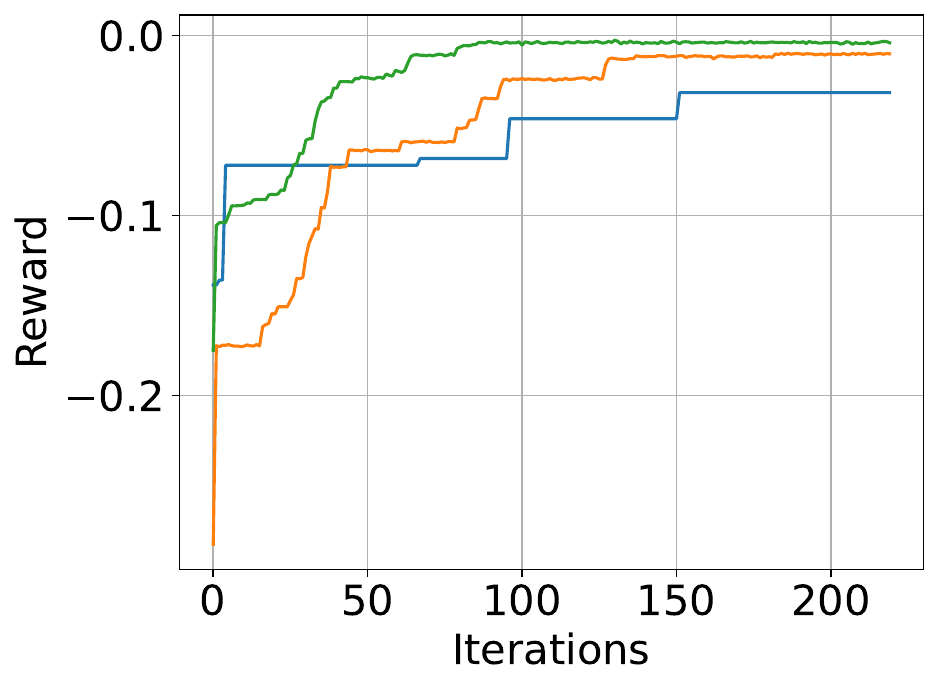}
    \caption{PDE optimization ($d=4$)}
    \label{figpde1-}
\end{subfigure}
\hfill
\begin{subfigure}{0.48\linewidth}
    \centering
    \includegraphics[width=.8\linewidth]{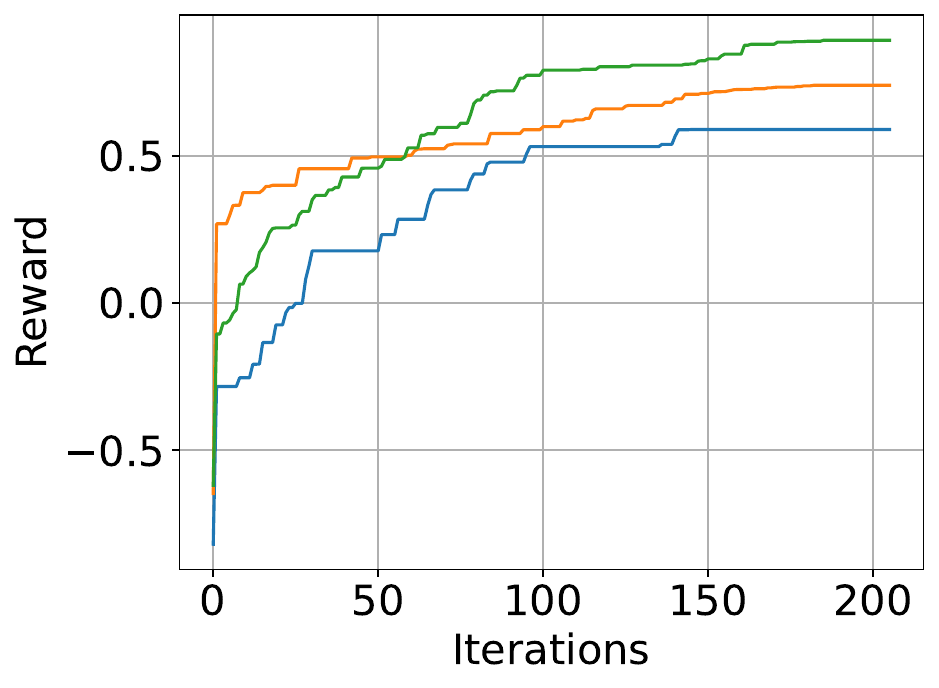}
    \caption{Interferometer ($d=4$)}
    \label{fig-2b--}
\end{subfigure}
\vspace{0.2cm}
\begin{subfigure}{0.48\linewidth}
    \centering
    \includegraphics[width=.8\linewidth]{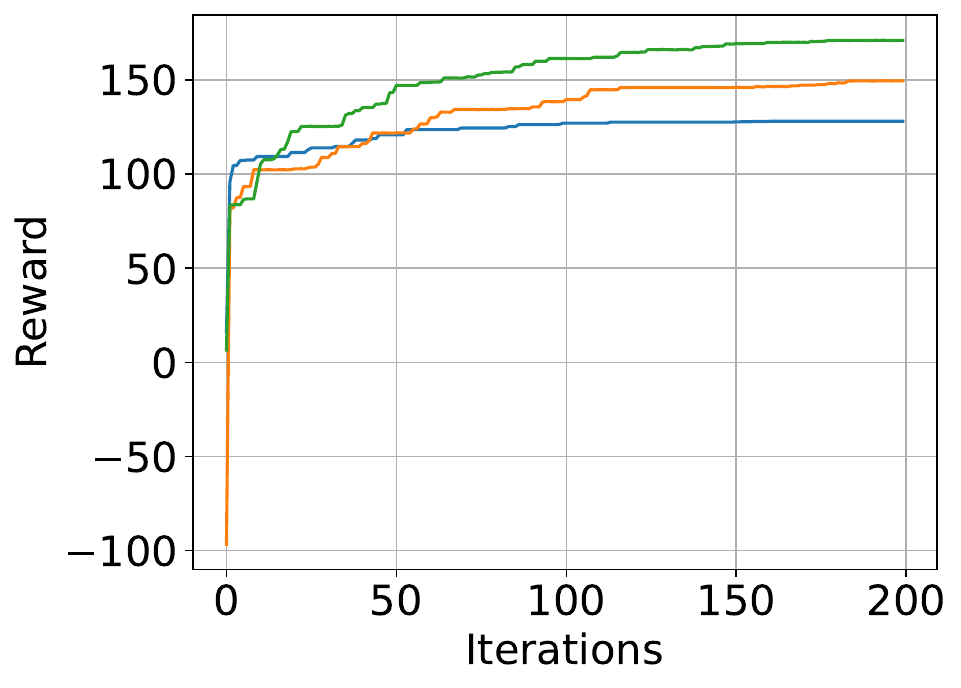}
    \caption{Lunar Lander ($d=12$)}
    \label{fig-2c--}
\end{subfigure}
\hfill
\begin{subfigure}{0.48\linewidth}
    \centering
    \includegraphics[width=.8\linewidth]{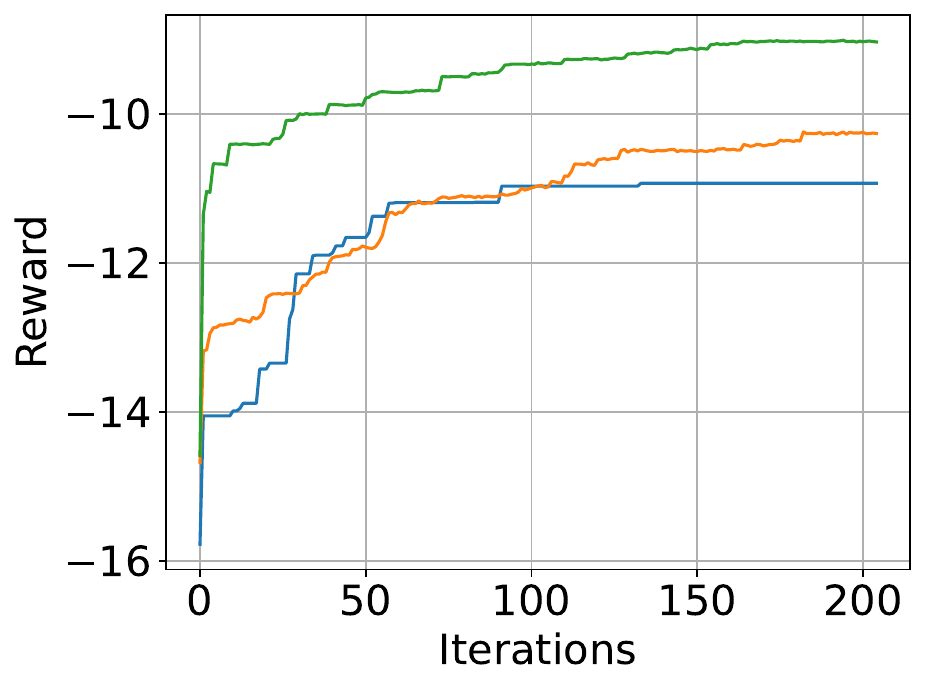}
    \caption{Pest Control ($d=25$)}
    \label{fig-2d--}
\end{subfigure}
\caption{Extended acquisition term ablation on four additional
         real-world tasks (cf.\ Section~\ref{sec:ablation} for results
         on Ackley and Rover).
         \emph{VBO-MI w/ GP exploration}: the exploration term
         $\sqrt{\beta}\,I_\theta$ is replaced by
         $\sqrt{\beta}\,\sigma^2_{\mathrm{GP}}$.
         \emph{VBO-MI w/ GP exploitation}: the exploitation term
         $\hat{f}_\psi$ is replaced by $\mu_{\mathrm{GP}}$.
         Full VBO-MI outperforms both hybrids on all four tasks.}
\label{fig:benchmarks-ex-comp}
\end{figure*}
\end{document}